\documentclass[sigconf]{acmart}

\usepackage{algorithm}
\usepackage{algpseudocode}
\usepackage{xcolor}
\usepackage{subcaption}

\usepackage{indentfirst}
\usepackage{dsfont}
\usepackage{bbm}
\usepackage{hyperref}
\usepackage{url} 
\usepackage{graphicx}
\usepackage{multirow}
\usepackage{pifont}
\usepackage{amsthm}
\usepackage{float}
\newtheorem{definition}{Definition}
\newtheorem{theorem}{Theorem}
\newtheorem{lemma}{Lemma}
\newtheorem{assumption}{Assumption}

\usepackage{booktabs} 
\usepackage{multirow}
\AtBeginDocument{%
  }

\setcopyright{acmcopyright}
\copyrightyear{2018}
\acmYear{2018}
\acmDOI{XXXXXXX.XXXXXXX}

\acmConference[Conference acronym 'XX]{Make sure to enter the correct
  conference title from your rights confirmation emai}{June 03--05,
  2018}{Woodstock, NY}
\acmPrice{15.00}
\acmISBN{978-1-4503-XXXX-X/18/06}

\begin{document}

\title{Practical Differentially Private and  Byzantine-resilient Federated Learning}

\author{Zihang Xiang}
\affiliation{%
 \institution{KAUST}
 \streetaddress{}
 \city{}
 \state{}
 \country{}
 }
\email{zihang.xiang@kaust.edu.sa}

\author{Tianhao Wang}
\affiliation{%
 \institution{ University of Virginia}
 \streetaddress{}
 \city{}
 \state{}
 \country{}
 }
\email{tianhao@virginia.edu}

\author{Wanyu Lin}
\affiliation{%
 \institution{The Hong Kong Polytechnic University}
 \streetaddress{}
 \city{}
 \state{}
 \country{}
 }
\email{wanylin@comp.polyu.edu.hk}

\author{Di Wang }
\affiliation{%
 \institution{KAUST}
 \streetaddress{}
 \city{}
 \state{}
 \country{}
 }
\email{di.wang@kaust.edu.sa}

\title{Practical Differentially Private and  Byzantine-resilient Federated Learning}

\begin{abstract}
Privacy and Byzantine resilience are two indispensable requirements for a federated learning (FL) system. Although there have been extensive studies on privacy and Byzantine security in their own track, solutions that consider both remain sparse. This is due to difficulties in reconciling privacy-preserving and Byzantine-resilient algorithms.

In this work, we propose a solution to such a two-fold issue. We use our version of differentially private stochastic gradient descent (DP-SGD) algorithm to preserve privacy and then apply our Byzantine-resilient algorithms. We note that while existing works follow this general approach, an in-depth analysis on the interplay between DP and Byzantine resilience has been ignored, leading to unsatisfactory performance. Specifically, for the random noise introduced by DP, previous works strive to reduce its impact on the Byzantine aggregation. In contrast, we leverage the random noise to construct an aggregation that effectively rejects many existing Byzantine attacks.

We provide both theoretical proof and empirical experiments to show our protocol is effective:  retaining high accuracy while preserving the DP guarantee and Byzantine resilience. Compared with the previous work, our protocol 1) achieves significantly higher accuracy even in a high privacy regime; 2) works well even when up to $90\%$ distributive workers are Byzantine.
\end{abstract}

\keywords{Federated learning, Byzantine Security, Differential privacy, Distributive computation}

\maketitle

\section{Introduction}
Federated Learning (FL), a learning framework for preserving the privacy of distributed data~\cite{konevcny2016federated},  has thrived during the past few years. To comply with the privacy regulations such as General Data Protection Regulation (GDPR) \cite{gdpr}, variants of FL frameworks have been widely studied, and recently adopted in industry, such as Apple's ``FE$\&$T'' \cite{apple_fed}, Google's Gboard \cite{gboard}, and Alibaba's FederatedScope \cite{xie2022federatedscope}. In an FL system, there are several local workers, each holding a dataset for local training, and a server aggregating gradient vectors from workers for global model updates. 

However, current FL frameworks that seemingly can protect privacy (because the original data never leaves the local workers) are in fact vulnerable to various privacy attacks, such as membership inference attacks \cite{MIA_shokri2017membership} (tries to infer whether some data samples are used in training) and model inversion attacks \cite{zhu2019deepleak} (``reverse-engineer'' sensitive data samples through gradients). These vulnerabilities drive the community to design methods that can further preserve the privacy of data held by workers. Among the privacy-enhancing techniques \cite{DBLP:conf/sigmod/FuXCT022, DBLP:conf/sigmod/FuSYJXT021,DBLP:conf/sigmod/WangBNM22,DBLP:conf/sigmod/RenSYYZX22}, differential privacy (DP) \cite{dwork2006calibrating} is a rigorous mathematical scheme that allows for rich statistical and machine learning analysis and is becoming the {\em de facto} notion for data privacy. 
Many methods have been proposed to tackle the problems of integrating DP into machine learning/deep learning from different perspectives \cite{abadi2016deep,iyengar2019towards,chaudhuri2011differentially,wang2017differentially,wang2019differentially,wang2020empirical,smith2017interaction,papernot2018scalable,papernot2016semi}. More recently, DP has been adopted in the FL setting \cite{geyer2017differentially,agarwal2021skellam,zhang2021understanding,truex2020ldp,wei2020federated}. 

Besides privacy risks, FL systems are also vulnerable to adversarial manipulations from Byzantine workers, which could be fake workers injected by an attacker or genuine workers compromised by an attacker. 
Specifically, in a Byzantine attack, the adversary intends to sabotage the collective efforts by sending false information, such as contrived Byzantine gradients \cite{baruch2019little, byz_xie2020fall}. To mitigate this issue,  recent work proposes Byzantine-resilient machine learning approaches, such as diagnosing and rejecting gradients with abnormal features~\cite{blanchard2017machine, pillutla2019robust,byz_chen2017distributed, caofltrust, regatti2020bygars}. 

Tremendous progress on privacy and Byzantine resilience have been seen in their own track. However, all of them are not applicable to the more practical scenario where a privacy attacker is also Byzantine (a double-role attacker). Being aware of that, some recent work started to focus on such an issue yet provided unsatisfactory answers. Some of them fail to ensure both DP and Byzantine resilience simultaneously \cite{dp_br_add_up}, while some other work tries to explore optimal parameter setups but still end up with a much-limited solution \cite{guerraoui2021combining}. We also notice that some work \cite{zhu2022bridging} tries to combine existing variants in both tracks to side-step the seeming incompatibility of DP and Byzantine security, however, their resistance is retained only when the privacy level is low and the portion of Byzantine clients is small.

\vspace{0.1cm}
\noindent\textbf{Contributions:} We observe that previous solutions fail to give a satisfactory answer for a common reason:  neither the DP algorithm nor the Byzantine defending method is designed against both risks simultaneously. Our contribution is how we start from a co-design to form a DP and Byzantine-resilient solution, proving the synergy of combined DP and Byzantine resilience.

\textbf{1) Co-design}: Since random noise introduced by DP impairs the effectiveness of existing Byzantine-resilient aggregation rules, previous works tend to limit the impact of randomness by increasing the data batch size~\cite{dp_br_add_up, guerraoui2021combining}. In contrast, we leverage random noise to aid Byzantine aggregation: we use small batch size and accordingly construct our first-stage aggregation which effectively rejects many existing attacks.

Moreover, previous works continue to use the standard DP-SGD~\cite{abadi2016deep} to bound gradient sensitivity by clipping, which involves manually tuning the clipping parameter. In contrast, we ensure bounded sensitivity by normalizing, and it
enables our second-stage aggregation, which provides a final sound filtering.

\textbf{2) Cherry on top:} We are the first to find out that bounding gradient sensitivity by normalizing is more suitable for DP learning although normalizing itself is not new\footnote{Some concurrent work \cite{yang2022normalized, bu2022automatic,de2022unlocking,das2021privacy} on DP learning use an operation similar to normalizing with different considerations. }. Specifically, we analyze its theoretical implication and also leverage it to construct a learning protocol that saves quadratic efforts\footnote{Instead of tuning learning rate $\eta$ and clip threshold $C$ for different $\epsilon$, our approach only needs to turn $\eta$ for any instance of $\epsilon$.} in hyper-parameter-tuning for DP learning, where a smaller amount of queries on gradient computation is more favored.

In the final evaluation, we conduct experiments to first show our contribution to DP solution and Byzantine aggregation in their own track. Then, for the core aim to preserve both privacy and Byzantine security, our experimental results show that in addition to having the DP guarantee, our protocol also remains robust against strong attackers when there are up to {\bf $90\%$} distributive workers are Byzantine. We have released our code in supplementary material
\footnote{

\href{
https://github.com/zihangxiang/-Practical-Differentially-Private-and-Byzantine-resilient-Federated.git}{https://github.com/zihangxiang/-Practical-Differentially-Private-and-Byzantine-resilient-Federated.git}.
}.

\section{Background and Preliminaries}\label{sec:background}

\subsection{Federated Learning} 
In a typical setting of machine learning, we have a training dataset $\{x_1, x_2, \cdots, x_{m}\}$ where each $x_i$ contains a feature vector and a label, and we also have a loss function $f$. We aim to find the best model parameters $w$ from a parameter space $\Theta$ which minimizes the following function through stochastic optimization:
\begin{equation}\label{eq:1}
   \min_{w\in \Theta} F(w)=\mathbb{E}_{x \sim \mathcal{P}}\left[f(x; w)\right],
\end{equation}
In FL, suppose there are $n$ workers and the $i$-th worker has local and private data $D_i$, then training in FL happens in a distributed manner. Specifically, in the $t$-th iteration, we have:

1) \textit{Model broadcasting:} The server broadcasts the current model parameters $w^{t-1}$ to all workers. 

2) \textit{Local gradient computation:} After receiving the model sent by the server, each worker will use his/her private data and the model $w^{t-1}$ to compute his/her gradient vector $g^t_i= \nabla f(D_i;w^{t-1}):=\frac{1}{|D_i|}\sum_{x\in D_i}\nabla f(x; w^{t-1})$. Note that workers can also compute their gradients and update their model $N$  times locally, and report the difference between the model they get locally and the last model they receive from the server. In our framework, we take $N=1$ and this is due to the constraints of DP-SGD protocol which will be discussed later. Extending to the cases where $N>1$ will be left for future study.

3) \textit{Gradient aggregation and model update:} The server will perform an aggregation step (denoted by function $\mathbf{Aggregation}$) on the gradient vectors reported by workers and use the result $g^t=\mathbf{Aggregation}(g^t_1,g^2_2,\cdots,g^t_n)$ to update the model by $w^{t} = w^{t-1}- \eta g^t$,  where $\eta$ is the learning rate. Note that there are variants of aggregation strategies, e.g., $g^t= \sum_{i}\frac{|D_i|}{\sum_j |D_j|}g^t_i$~\cite{fedavg}.

\subsection{Differential Privacy for Deep Learning}\label{sec:ppml}

\begin{definition}[Differential Privacy \cite{dwork2006calibrating}]\label{def:dp}
	Given a data universe $\mathcal{X}$, we say that two datasets $D,\,D'\subseteq \mathcal{X}$ are neighbors if they differ by only one data sample, which is denoted as $D \sim D'$. A randomized algorithm $\mathcal{A}$ is $(\epsilon,\delta)$-differentially private if for all neighboring datasets $D,\, D'$ and for all events $S$ in the output space of $\mathcal{A}$, we have $\operatorname{Pr}(\mathcal{A}(D)\in S)\leq e^{\epsilon} \operatorname{Pr}(\mathcal{A}(D')\in S)+\delta.$ 
\end{definition}

An $(\epsilon,\delta)$-DP mechanism typically adds calibrated noise to the output of a query. In this paper we mainly use the Gaussian mechanism to guarantee $(\epsilon, \delta)$-DP:
	\begin{definition}[Gaussian Mechanism]
	Given any function $q : \mathcal{X}^n\rightarrow \mathbb{R}^d$, the Gaussian mechanism is defined as  $q(D)+\xi$ where $\xi\sim \mathcal{N}(0,\frac{2\Delta^2_2(q)\log(1.25/\delta)}{\epsilon^2}I_d)$,  where where $\Delta_2(q)$ is the $\ell_2$-sensitivity of the function $q$,
{\em i.e.,}
		$\Delta_2(q)=\sup_{D\sim D'}||q(D)-q(D')||_2$. Gaussian mechanism satisfies $(\epsilon, \delta)$-DP when $\epsilon\leq 1$. 
	\end{definition}

Another notable property of DP is that DP is closed under postprocessing, {\em i.e.,} if we post-process the output of an  $(\epsilon, \delta)$-DP algorithm, then the whole procedure will still be $(\epsilon, \delta)$-DP.

\textbf{DP in deep learning:}
Differentially Private SGD (DP-SGD) is a widely used method in  machine learning to ensure DP \cite{abadi2016deep, bassily2014private, song2013stochastic}. It modifies the SGD-based methods by adding Gaussian noise to perturb the (stochastic) gradient in each iteration of the training, {\em i.e.,} in the centralized setting, during the $t$-th iteration DP-SGD will compute a noisy gradient as follows: 
\begin{equation}\label{equ:dpsgd}    
g^t = \frac{1}{|B|}(\sum_{x_i \in B}\limits \hat{g}^t_{i} +\mathcal{N}\left(0, \sigma^2 C^2 I\right)), 
\end{equation}
where $B$ is a subsampled data batch used to compute the gradient, $\sigma$ is the noise multiplier, $\hat{g}^t_{i}$ is the gradient vector computed by feeding one data sample to $w^{t-1}$ which is the current model before the $t$-th iteration,  
and $g^t$ is the (noisy) gradient used to update the model.
The main reason here we use $\hat{g}^t_{i}$ instead of the original gradient vector is that we wish to make the term $\sum\limits \hat{g}^t_{i}$ have bounded $\ell_2$-sensitivity so that we can use the Gaussian mechanism to ensure DP. The most commonly used approach to get a $\hat{g}^t_{i}$ is clipping the gradient: $\hat{g}^t_{i}=\nabla f(x_i; w^{t-1})\min \{1, \frac{C}{\|\nabla f(x_i; w^{t-1})\|_2}\}$ {\em i.e.}, each gradient vector is clipped by $C$ (scale those whose $\ell_2$-norm is greater than $C$ to be $C$ exactly and leave the rest untouched). Since the $\ell_2$-sensitivity of $\sum\limits \hat{g}^t_{i}$ is bounded by $C$, after the clipping, we can add Gaussian noise to ensure DP. 

To prove the $(\epsilon, \delta)$-DP property of DP-SGD, there has been a line of research \cite{asoodeh2021three, mironov2019r, gopi2021numerical, zheng2020sharp, wang2019subsampled, zhu2019poission}. We use \textit{TensorFlow Privacy} \cite{tfp} to search for noise multiplier given $\epsilon$ and $\delta$. Other than the DP-SGD framework, we note that there exists others work \cite{choquette2021capc} tackling the privacy issue in FL by extending the PATE framework \cite{papernot2016semi}. However, their work does not apply to the case where Byzantine workers exist.

\subsection{Byzantine Attacks in FL} 
The FL protocol is vulnerable to Byzantine attacks, as each worker can report a malicious gradient vector to deteriorate the model performance \cite{baruch2019little,byz_xie2020fall} and even bias the model in a specific way \cite{chen2017targeted, DBLP:conf/ndss/LiuMALZW018}.  

Several Byzantine-robust approaches are proposed to tackle different attacks \cite{blanchard2017machine, pillutla2019robust, guerraoui2018hidden, yin2018Byzantine, byz_chen2017distributed, caofltrust, regatti2020bygars}.  And there are also new advanced Byzantine attacks that try to bypass such defenses \cite{ data_p_bagdasaryan2020backdoor, data_p_chen2017targeted, baruch2019little, utbyz_jagielski2018manipulating, utbyz_rubinstein2009antidote, byz_chen2017distributed, byz_xie2020fall}. 

\textbf{A Taxonomy:} To understand the features of those Byzantine attacks, we summarise three dimensions to capture their properties.

\textit{1) Objective:} There are generally 2 types of objectives: \textit{a}) Denial-of-service attack (also called \textit{untargeted} or convergence prevention)~\cite{baruch2019little} that tries to destroy the training process and makes the model unusable. \textit{b}) Backdoor attack \cite{DBLP:conf/ndss/LiuMALZW018} that tries to poison the training data to make the model predict intended results on inputs with specific triggers while still behaving normally on other inputs. 

\textit{2) Capability:} Existing attacks assume different attackers' power. 
Some work assumes that the Byzantine attacker is \textit{omniscient}, {\em i.e.}, the attacker knows what the honest workers send to the server~\cite{fang2020local, byz_xie2020fall}, 
while others assume the attacker does not know the honest workers' data \cite{baruch2019little}.
In both cases, the attacker knows the aggregation rules.
\textit{3) Specificity of targeted defenses:} \textit{a}) Some attacks are defense-specific, {\em i.e.}, and they are tailored for specific defense methods \cite{byz_xie2020fall, baruch2019little}.  However, it is unclear whether such attacks still remain effective against other or new defense protocols. \textit{b}) Some other attacks are universal: these attacks are either defense-agnostic \cite{zhu2022bridging, caofltrust} or have a meta-method \cite{fang2020local} that can be instantiated to attack almost all existing defense strategies upon knowing the defense rules.

\textbf{Instantiating Existing Attacks:} We briefly introduce some existing Byzantine attacks which have been considered in previous work.  They can be categorized by the above 3 dimensions: 1) \textit{Gaussian attack}~\cite{zhu2022bridging, regatti2020bygars} uploads pure Gaussian noise trying to hurt utility. 2) \textit{Label-flipping attack}~\cite{fang2020local, caofltrust} first poisons the local dataset by flipping the original label $I$ to $H-1-I$ ($H$ is the total number of classes, $I=0,1,\cdots, H-1$ is the label) and then follows the FL protocol. Note that we can also adopt other ways to perform the label flipping (such as randomly flipping to a different label). In fact, the way to flip the label does not matter as long as it tries to reduce the overall accuracy.  3) \textit{Optimized Local Model Poisoning attack}~\cite{fang2020local}, the state-of-the-art Byzantine attack method that can be accordingly instantiated for a specific Byzantine defense method in an adversarial way given the Byzantine defense protocol first. 

For ease of discussion, we use ``Byzantine attacker'' to refer to a master attacker which can inject several fake workers into the system and control all of them. Hence, in this sense, attackers who can send malicious uploads to the server can possibly collude.


\section{Problems and Existing Solutions \label{section:Problem and Existing Solutions}}
\subsection{Problem Setting \label{section:Problem Formulation}}

\noindent\textbf{Attacker:} FL is indeed exposed to threats from two kinds of attackers: \textit{privacy attacker} and \textit{Byzantine attacker}. Note that we will not discuss specific \textit{privacy attack}s and only focus on \textit{Byzantine attack}s for the following reasons.

\textbf{Protecting privacy with DP:} 
From an information-theoretical view, DP guarantees privacy in the worst case by limiting the maximum amount of information that any privacy attacker can extract even with side information and unlimited computational resource \cite{dwork2006calibrating}. It has also been shown that by tuning $\epsilon$ small enough \cite{rahman2018membership}, adopting DP effectively rejects strong privacy attackers \cite{MIA_shokri2017membership}. Following the privacy settings as in the previous work \cite{guerraoui2021combining,dp_br_add_up,zhu2022bridging}, we focus on the item-level privacy for each worker's dataset; 
in this case, the gradient needs to be privatized (by adding random noise) before being uploaded to the server.

\textbf{Focusing on Byzantine-resilience:} We are not interested in tuning $\epsilon$ to test the algorithm's strength to defend against privacy attacks. Since our algorithms are guaranteed to be $(\epsilon, \delta)$-DP theoretically,
\textit{we only need to focus on Byzantine attacks}.

In other words, we treat DP as one of the basic properties that an FL system should possess. In fact, leveraging our tailored DP protocol to defend against Byzantine attacks is one of our novelties.

\textbf{Byzantine attacker:}
We specify the Byzantine attacker according to the taxonomy mentioned above:

\begin{itemize}
    \item \textit{Objective:} Our Byzantine attacker is trying to perform Denial-of-Service (DoS) attacks.
    \item  \textit{Capability:} Our Byzantine attacker is omniscient; it knows all the data held by honest workers, information sent by  honest workers, and the aggregation rules. We assume such capability in our framework to show that our protocol still works even when facing such a strong adversary.
    \item \textit{Specificity of targeted defense:} We consider a stronger version of universal attacks. Our protocol will be made public and the attacker is allowed to instantiate his attack on our protocol.
\end{itemize}

In other words, we are interested in defending the stronger \textit{untargeted} attacks 
and we leave backdoor attacks as future work.

\vspace{0.2cm}
\noindent\textbf{Defender:} Privacy is guaranteed through DP, and each user applies DP to protect its local data.  To defend against Byzantine attacks, the server needs to design new aggregation rules 
such that false gradients from Byzantine workers are excluded. Here we assume 
the server possesses {\it a small amount} of labeled data samples which are kept secret from attackers. Let $\mathbb{X}$ be the the data space and $\mathbb{Y}$ be the label space in a classification task. We do not require the server to have direct access to the local-hold data, instead, we assume the data space $\mathbb{X}_{aux}$ from which the auxiliary data is sampled is the same as that of local-hold data.
Notably, this additional assumption is reasonable in real applications as getting such a tiny amount of data is relatively cheap, and there is work on DP learning \cite{hamm2016learning,papernot2016semi,papernot2018scalable,bassily2018model,kurakin2022toward,golatkar2022mixed,zhu2020private,wang2019estimating} and Byzantine-resilient learning \cite{caofltrust,regatti2020bygars} making this assumption. In our experiments, we simulate obtaining such data by randomly drawing $2C$ sampling from a validation set where $C$ is equal to the number of classes of that dataset (e.g., for the MNIST dataset, $C=10$, thus 20 auxiliary data samples will suffice). It is also helpful to consider whether the server-own data and local-hold data follow the same label distribution $\mathbb{Y}$ (with a slight abuse of notation), accordingly, we also conduct experiments for \textit{i.i.d.} case (distributions on $\mathbb{Y}$ are the same for both) and \textit{non-i.i.d.} case (distributions on $\mathbb{Y}$ are different).

We also assume the server knows the truth that at least $\gamma n$ workers are honest among all $n$ workers.
It is notable that in the paper we do not need to place any restriction on $\gamma$, while previous work \cite{yin2018Byzantine, fang2020local} needs to assume that $\gamma >0.5$.

In conclusion, each local worker adopts DP to protect privacy, hence, for one worker, the privacy attacker can be anyone (including the server) except itself. The Byzantine attacker (disguised as some local workers) contrives its upload and tries to destroy the training to make the model has low utility. As a Byzantine defender, the server is honest-but-curious, {\em i.e.}, it wants to have a model with good utility, thus it follows the protocol. However, it may try to infer sensitive information from the uploads sent by local workers.

\subsection{Existing Solutions}

\begin{table}[!t] 
\footnotesize
\centering
    \begin{subtable}[h]{\columnwidth}
    \centering
        \begin{tabular}{r|cc}
            \toprule
            Methods & Privacy & $>50\%$-Resilience\\
            \midrule
            Krum \cite{blanchard2017machine} & \ding{55} & \ding{55}\\
            Coordinate-wise Median \cite{yin2018Byzantine} & \ding{55}& \ding{55}\\
            Trimmed Mean \cite{yin2018Byzantine} & \ding{55} &\ding{55}\\
            Bulyan \cite{guerraoui2018hidden} & \ding{55} &\ding{55}\\
            Zhu et al. \cite{zhu2022Byzantine} & \ding{55} &\ding{55}\\
            FLTrust \cite{caofltrust} & \ding{55}& \ding{51}\\
            Rachid et al. \cite{guerraoui2021combining}& \ding{51}& \ding{55}\\
            Xu et al. \cite{ma2022differentially} & \ding{51} &\ding{55}\\
            Heng et al. \cite{zhu2022bridging} & \ding{51} &\ding{55}\\
            \textbf{Our work}& \ding{51} &\ding{51}\\
            \hline
        \end{tabular}
    
     \end{subtable}
     
     \caption{
      Comparison with previous work. For privacy, \ding{51} means the method is guaranteed to be DP while the \ding{55} means the converse; for $>50\%$-Resilience, \ding{51} means the method remains resilient when the number of Byzantine attackers exceeds half of the total while \ding{55} means the method is no longer effective under such majority Byzantine attack.}
    \label{tab:strengths_summary}
\vspace{-0.3cm}
\end{table}

In Table \ref{tab:strengths_summary}, we summarize the previous methods on privacy and Byzantine security issues in FL. As we can see from the table, our method can defend against more than $50\%$ Byzantine workers while also achieving DP. In the following, we will provide more details and discuss the limitations of previous methods. 

1) The following lines of work only focus on defending against Byzantine attacks: We first recall some existing solutions to Byzantine resilience. There is a line of work focusing on designing robust aggregation rules for corrupted gradients  \cite{blanchard2017machine, pillutla2019robust, guerraoui2018hidden, yin2018Byzantine, byz_chen2017distributed, caofltrust, regatti2020bygars, guerraoui2018hidden} including Krum \cite{blanchard2017machine}, RFA \cite{pillutla2019robust}, coordinate-wise median \cite{yin2018Byzantine}, and Trimmed Mean \cite{yin2018Byzantine}. We summarise the detail of these four methods in \href{https://github.com/zihangxiang/-Practical-Differentially-Private-and-Byzantine-resilient-Federated.git}{supp. material}. for interested readers. In general, the first two methods involve computing pair-wise distance between vectors while the latter two concentrate on robust aggregation on each coordinate of vectors. Due to their intrinsic limitations, all of these methods are only applicable when the majority of workers ($>50\%$) are honest. Recently, Zhu et al.~\cite{zhu2022Byzantine} propose improvements to existing Byzantine-resilient methods to provide a certain level of resilience. 

Recently, there is work showing that it is possible to leverage the knowledge of clean gradient computed from non-private auxiliary data \cite{caofltrust,regatti2020bygars} to help the aggregation. Common behavior in such methods is that the server weights each uploaded gradient according to their similarity compared to the gradient computed by server-own auxiliary data.

All of the above methods have no DP guarantee because they are designed to defend against Byzantine attacks ignoring privacy attacks. Hence, local datasets' privacy is at risk.

2) The following lines of work try to apply Byzantine defending methods on top of DP output: Some recent work investigates the problem of maintaining both DP and Byzantine resilience in federated learning \cite{dp_br_add_up, guerraoui2021combining}. Specifically, they study methods of directly combing DP-SGD with some existing aggregation methods such as Krum, {\em i.e.,} by applying the aggregation on the noisy gradient. 

\textit{Difficulties in reconciling DP and Byzantine-resilient protocols:} Previous work shows that for these methods, to become Byzantine-resilient when DP noise is injected, the fraction of Byzantine workers must decrease with $\sqrt{d}$, where $d$ is the size of the model if the batch size is not large enough \cite{dp_br_add_up, guerraoui2021combining}. This means such methods achieve good performance only when the number of Byzantine workers is small. Their experiments also verify that this type of method is unsatisfactory by showing that the testing accuracy deteriorates significantly even for a small model learned on a simple dataset~\cite{dp_br_add_up}. Another line of work is based on robust stochastic model aggregation on the local workers' gradients. In these methods, the gradients of each worker are compressed into signs (1 for non-negative and -1 for negative) with DP \cite{ma2022differentially, zhu2022bridging}, however, all of them remain effective only under $<50$\% Byzantine attack.

\section{Our Approach}\label{sec:Our Approach}
\noindent\textbf{Observations and lessons learned :} Existing solutions apply off-the-shelf Byzantine methods on top of noisy gradient to explore optimal parameter setups. They fail to reach satisfactory performance because neither the Byzantine defending protocol nor the DP protocol is designed for the scenario where privacy and Byzantine resilience are both needed. 

\subsection{Method Overview}
We first re-design the DP protocol, there are two notable properties we enforce in our DP protocol: 1) small training batch size for each worker; 2) use normalizing instead of clipping to bound per-example gradient norm. \textit{We are not considering privacy and Byzantine resilience in a separate manner}. We design our first-stage aggregation based on the first property and design a second-stage aggregation based on the second. As will be seen later, the first property together with the first-stage aggregation trivially yet effectively rejects some existing attacks. The second property together with the second-stage aggregation effectively rejects more advanced attacks which bypasses our first-stage aggregation. As a cherry on top, in our DP protocol, such two properties themselves enable efficient hyper-parameter tuning.

\subsection{Modifying DP Protocol}\label{DPlearning}

Our DP protocol is summarized in Algorithm \ref{alg:scen1}. The two notable properties compared with vanilla DP-SGD~\cite{abadi2016deep} are: 
1) different from existing works that adopt big batch size ($10^2-10^6$) \cite{de2022unlocking,anil2021large}, we adopt small batch size $b_c$ (typically $8$ or $16$). Note that small batch size is essential for our first-stage aggregation (see Section \ref{sec:first-aggregate} for details); 2) the second, is to replace the clipping operation in vanilla DP-SGD by normalization,
vanilla DP-SGD method clips the gradients by multiplying the gradient vector $g$ by $ Factor = \min \{1, \frac{C}{\|g\|_2}\}$ ($C$ is called the clipping threshold). We modify the multiplication factor to $Factor = \frac{1}{\|g\|_2}$, which normalizes the gradients to be of unit length. Also note that inspired by \cite{normalizedsgd}, to have a better convergence behavior, the gradients are processed with momentum. 

We will see in the following why normalizing enables efficient hyper-parameter tuning. In later sections about defending against Byzantine attacks, we will also see that using small batch size is essential in our first-stage aggregation, and normalizing also plays an important role in the second-stage aggregation.

\begin{algorithm}[!t]
\caption{Private and Secure Learning
}\label{alg:scen1}
\begin{algorithmic}[1]
\small
\renewcommand{\algorithmicrequire}{\textbf{Input:}}
\renewcommand{\algorithmicensure}{\textbf{Output:}}
\Require initial model $w^0$, number of iteration  $T$, learning rate $\eta$, datasets held by $n$ workers $\{D_i| i=1,2,\cdots,n\}$, gradient momentum $\beta$, noise multiplier $\sigma$, batch size $b_c$, loss function $f(;)$

\State {Each worker $i$ initializes a size-$b_c$ momentum list $\phi^0_i=[0,\cdots,0]$} \label{alg:scen1:dp_begin}
\For{$t=1,2,\cdots,T$}
    \State Server broadcasts model $w ^{t-1}$ to all workers
    \For{$i=1, 2, \cdots, n$} \textbf{in  parallel}
         \State Sample a size-$b_c$ mini-batch $d_i$
         \For{$j=1,2,\cdots,b_c$} \textbf{in  parallel}
            \State $g_j \gets \nabla f(x_j\in d_i;w ^{t-1})$
            \State $\phi^t_i[j] \gets (1-\beta) g_j +\beta \phi^{t-1}_{i}[j] $ 
            \label{alg:scen1:momentum}
        \EndFor
        \State $g^t_i \gets \frac{1}{b_c}\left(\sum\limits_{j \in [b_c]}\frac{\phi^t_i[j]}{ \left\| \phi^t_i[j]\right\|_2}+ \mathcal{N}(0, \sigma^2 I)\right)$ \label{alg:scen1:gaussian_mech}
        \State Upload $g^t_i$ to server, then $\phi^t_i[j] \gets g^t_i$
    \EndFor\label{alg:scen1:dp_end}
    \State $G^t_s \gets \mathbf{FilterGradient}( \{g^t_i|i=1,2,\cdots,n\},w ^{t-1})$\label{alg:scen1:filter}
    \State $w ^t \gets w ^{t-1} - \eta \frac{1}{n}\sum\limits_{g\in G^t_s} g$ \label{alg:scen1:update}
\EndFor
\Ensure learned target model $w^T$
\end{algorithmic}
\end{algorithm}

\textbf{Normalization helps hyper-parameter tuning:} We now introduce Theorem \ref{thm:general_convergence} which supports our hyper-parameter tuning strategy. Based on our DP protocol, consider a simpler case where there is no Byzantine attacker and we only have one honest worker. 
The model update (without momentum) in the $t$-th iteration has the following form:
\begin{equation}\label{equ:model_update_one_worker}
    w^{t} = w^{t-1} - \frac{\eta }{|B^t|}\left(\sum\limits_{g^t \in B^t} \frac{g^t}{\left\|g^t\right\|}+ z\right),
\end{equation}
where $B^t$ is the current local batch of per-example gradient (we fix the batch size to be $|B^t|=b_c$), $z \sim \mathcal{N}(0, \sigma^2 I)$ is the DP noise and $g^t=\nabla f(x; w^{t-1})$.  Since $g^t$ is derived from only a batch of samples, there is a sampling error.  Denoting the sampling error by $\xi^t$, we can rewrite $\nabla f(x; w^{t-1})$ as $\nabla F(w^{t-1}) + \xi^t $, where $\nabla F(w^{t-1})$ is the gradient of the stochastic function.
With some assumptions\footnote{In deep neural networks we always have $F(w) > 0$, and for the L-Lipschitz continuous and bounded variance assumptions, they have been commonly used in the previous work for convergence analysis \cite{normalizedsgd, zhang2021understanding,yang2022normalized, bu2022automatic}.} that: 1) The  stochastic function $F$ is bounded below with $F(w) > 0$; 2) $F$ has $L$-Lipschitz continuous gradient (defined in Assumption \ref{ass:1}); 3) The random vector $g^t=\nabla F(w^{t-1}) + \xi^t $  has bounded variance, {\em i.e.,}  $\mathbb{E}\left\|  \xi^t \right\|^2 \leq \nu^2$ with some $\nu$. We have the following result: 

\begin{theorem}\label{thm:general_convergence}
(Convergence Behavior) Given a learning rate $\eta$, and the model is updated according to Equation \ref{equ:model_update_one_worker} (gradient is normalized),  we have
\begin{equation*}\label{equ:decent_bound_final_summed_up_main_body}
 \frac{1}{T}\sum_{t=1}^{T}\mathbb{E}\left\|\nabla F(w^t)\right\| \leq \underbrace{\frac{3F(w^{0})}{T\eta} + \frac{3L\eta}{2}\left(1 + \frac{\sigma^2d}{b_c^2}\right)}_{M} + 8\nu.
\end{equation*}
\end{theorem}
\begin{proof}
See \href{ https://github.com/zihangxiang/-Practical-Differentially-Private-and-Byzantine-resilient-Federated.git}{supp. material}.
\end{proof}

By the above result, we can see that it is sufficient to minimize the term $M$, whose expression provides valuable 
guidance on choosing $T$ and $\eta$ that need to be set before training. If the magnitude of the noise and the batch size $b_c$ satisfy $\frac{\sigma^2d}{b_c^2}\gg 1$, then by setting the learning rate as
\begin{equation}\label{equ:setting_eta}
    \eta = \frac{1}{\sigma}\sqrt{\frac{2F(w^0)b_c^2}{TLd}},
\end{equation}
we have  $M\approx\sqrt{\frac{36F(w^0)Ld\sigma^2}{2b_c^2T}}$. Note that since we always have $\sigma=\Omega(\frac{q\sqrt{T\log(1/\delta)}}{\epsilon})$  where $q=\frac{b_c}{|D|}$ is the sampling rate ($|D|$ is the size of data) \cite{abadi2016deep}. Thus, we have $M=\Omega\left(\frac{1}{\epsilon|D|}\sqrt{F(w^0)Ld\log(1/\delta)}\right).$
This implies that: 1) The lower bound of $M$ is getting worse  when $\epsilon$ becomes smaller; 2) We get this optimal bound by relating $T$ and $\eta$ 
via Equation (\ref{equ:setting_eta}). If we fix one, we can potentially get the other one analytically instead of going through inefficient hyper-parameter tuning. In practice, we fix $T$ first and decide the learning rate $\eta$. With $T$ fixed, Equation (\ref{equ:setting_eta}) suggests that the optimal learning rate should be set inversely proportional to the DP noise multiplier $\sigma$, and this leads to our efficient hyper-parameter tuning strategy which outperforms existing methods. Note that the previous analysis was built on the assumption that $\frac{\sigma^2d}{b_c^2}\gg 1$. Thus, to satisfy the assumption, 
we can either use a bigger model (increase $d$) or adopt a smaller batch size. Hence, using a small batch size is preferred for our method and differs from existing work as mentioned before.

Hence, our DP approach saves quadratic efforts and is truly beneficial for DP learning. This is beyond only considering the running-time complexity.

\subsection{First-stage Aggregation}\label{sec:first-aggregate}

\textbf{Design strategy:} Inspecting existing work on Byzantine resilience, the uploads ($d$-dimension vectors) by Byzantine attackers are arbitrary in $\mathbb{R}^d$. Hence, a single faulty inclusion on a malicious upload can totally destroy model updates. As a strategy, enforcing some constraints on the subspace where any upload should lie will be beneficial to defend against attacks.
From a high-level perspective, our refactored DP protocol together with first-stage aggregation does the job of ``\textit{constrain}''; our second-stage aggregation does the job of ``\textit{complementary aggregation}". 

Specifically, our choice of small batch size leads to the phenomenon that DP noise dominates for each upload, which leads to some expected statistical properties. Another phenomenon is that although for each worker DP noise is dominating,  we can still achieve good utility overall. The reason is that the server takes the average of all aggregated uploads; such an operation reduces the DP noise variance and averages gradients to its non-zero expectation. Therefore, we can still get good utility as long as the number of honest workers is sufficiently large. As will be shown in the experiments section, 10-20 honest workers will suffice and such  a number is smaller than that in previous work of FL systems \cite{caofltrust, fang2020local}. 

\textbf{Forming first-stage aggregation:} 
As we can see from Algorithm \ref{alg:scen1}, an honest worker will upload $g = \tilde{g} + z$ to the server where $\tilde{g}$ is the sum of some normalized terms, and $z$ is the DP noise.
If DP noise is dominating ($\left\|z\right\|\gg\left\|\tilde{g}\right\|$), we can approximately treat vector $g$ as each coordinate of $g$ is sampled from $\mathcal{N}(0, \sigma^2)$. We then have the following conclusion. 

\textbf{Norm test:}
Note that $\frac{\| g \|^2}{\sigma^2}$ follows the chi-squared distribution with degree $d$. 
As $d$ is very large, by Central Limit Theorem, we can safely approximate the distribution of $\left\| g \right\|^2$ as Gaussian distribution: $\mathcal{N}(\sigma^2d, 2\sigma^4d)$. Hence, $\left\| g \right\|^2$ falls in the interval $\left[\sigma^2d-3\sigma^2\sqrt{2d},\sigma^2d+3\sigma^2\sqrt{2d}\right]$ almost surely.\footnote{Such interval is narrow, because $\frac{\sigma^2\sqrt{2d}}{\sigma^2d}\ll 1$ when $d$ is very large. And by the 68-95-99.7 rule \cite{enwiki:1097113055}, we set such an interval to span three s.t.d. around the center so that a benign gradient falls into this interval with 99.7\% probability approximately.} 
    

Similarly, we can also conclude other statistical results for higher-order moments leveraging the property of Gaussian distribution. However, the real situation only allows us to use a limited number of these statistics. To further enhance the efficiency and the soundness of such checking, we will leverage non-parametric test methods which test the hypothesis that given samples follow a reference distribution. We leverage {Kolmogorov–Smirnov} test (KS test) \cite{kolmogorov1933sulla} as described below.

\textbf{KS test:} Treat each coordinate of $g$ as a sample and the null hypothesis is that these samples are sampled from the same distribution $\mathcal{N}(0, \sigma^2)$. Suppose that we are currently testing on upload $g$ with $d$ coordinates ($d$-dimension vector) and we denote the $i$-th coordinate of $g$ as $g[i]$. KS test will 1) compute the empirical Cumulative Distribution Function as: $$C_d(x)= \frac{1}{d}\sum_{i=1}^d\mathbbm{1}_{g[i]<x},$$
where $\mathbbm{1}_{g[i]<x}$ is the indicator function that takes value on 1 if $g[i]<x$ and 0 otherwise; 2) compute the KS statistics $D_{KS}=\sup_x \left| C_d(x)-\Phi_{\sigma}(x) \right|$ where $\Phi_{\sigma}(x)$ is the CDF of $\mathcal{N}(0, \sigma^2)$; 3) compute the P-value by $D_{KS}$ from Kolmogorov D-statistic table \cite{marsaglia2003evaluating} and there are many off-the-shelf libraries that can compute it. If the \textit{P-value} is smaller than 0.05,\footnote{We use the widely adopted significance level.} we reject the null hypothesis (the server then treats $g$ as one malicious upload that is not sampled from $\mathcal{N}(0, \sigma^2)$ and rejects it).


\begin{algorithm}[!ht]
\caption{$\mathbf{FirstAGG}(g)$}\label{alg:primary_agg_rule}
\begin{algorithmic}[1]
\small
\renewcommand{\algorithmicrequire}{\textbf{Input:}}
\renewcommand{\algorithmicensure}{\textbf{Output:}}
\Require $g$, the upload to be tested
    \If { $\left\|g\right\|< \sqrt{ \sigma^2d-3\sigma^2\sqrt{2d}}$ or $\left\|g\right\| > \sqrt{ \sigma^2d+3\sigma^2\sqrt{2d}}$}
            {
            $g \gets 0$
            }
        \EndIf
    \If { $KS(g)< 0.05$}
            {
            $g \gets 0$
            }
        \EndIf
\Ensure $g$
\end{algorithmic}
\end{algorithm}

\textbf{KS test confines Byzantine subspace:} After we set the significance level for the P-value, we are essentially placing an upper bound on $D_{KS}$ (if $D_{KS}$ is too large, the corresponding P-value will be small enough to make $g$ to be rejected). And from the definition of $D_{KS}$, we know that such an upper bound applies to $\left| C_d(x)-\Phi_{\sigma}(x) \right|$ for any $x\in \mathbb{R}$. 
This can also be interpreted as that the curve of $C_d(x)$ will fall into a band bounded by an upper envelop $E_u(x)=min(1, \Phi_{\sigma}(x)+D_{KS})$ and a lower envelop $E_l(x)=max(0, \Phi_{\sigma}(x)-D_{KS})$. If we set the significance level strictly enough, $D_{KS}$ will be small enough such that the band will be narrow, requiring that $C_d(x)$ almost aligns with $\Phi_{\sigma}(x)$. 

Formally, since $C_d(x)$ is a step function (with $d$ steps) that is monotonously increasing with $d$ steps ($\frac{i}{d} \rightarrow \frac{i+1}{d}$ for $i=0,1,\cdots,d-1$), we have the following theorem.

\begin{theorem}\label{thm:resilience}
(Byzantine Resilience) If we sort all coordinates of an upload vector into a sequence by increasing order, to pass the \textbf{KS} test, the $k$-th (we count from 1 to $d$) element must fall into the interval:
\begin{equation}\label{equ:upload_subspace}
    \left[E_u^{-1}\left(\frac{k}{d}\right),E_l^{-1}\left(\frac{k-1}{d}\right)\right]
\end{equation}
\end{theorem}
\begin{proof}
Let the $k$-th element be $x_k$, according to the definition of $C_d(x_)$, then we must have $C_d(x_k) = \frac{k-1}{d}$ and $C_d(x_k + s) = \frac{k}{d}$ where s is some small real number. To pass the {KS} test, $\frac{k-1}{d}$ must be above $E_l(x_k)$ and $\frac{k}{d}$ must be below $E_u(x_k)$, consequently, $x_k$ falls into the above interval.
\end{proof}

In essence, our first-stage aggregation enforces that the attacker's upload must lie in the subspace of $\mathbbm{R}^d$ as described by Equation \ref{equ:upload_subspace}. This is different from existing work on Byzantine security, where the attacker's upload can be arbitrary in $\mathbbm{R}^d$ \cite{baruch2019little,byz_xie2020fall,caofltrust,fang2020local,regatti2020bygars}.

\textbf{Ensuring $\left\|z\right\|\gg\left\|\tilde{g}\right\|$:}
Our first-stage aggregation only keeps those uploads that follow our DP protocol (approximately with the form $g = \tilde{g} + z$) and this validity builds on the assumption that $\left\|z\right\|\gg\left\|\tilde{g}\right\|$. The good news is that we can always control $\frac{\left\|z\right\|}{\left\|\tilde{g}\right\|}$. Recall that $\left\|z\right\| \approx \sigma\sqrt{d}$ and $\tilde{g}$ is just the sum of $b_c$ norm-bounded vectors, hence, before the training, we can compute $\frac{\left\|z\right\|}{\left\|\tilde{g}\right\|}$. To increase $\frac{\left\|z\right\|}{\left\|\tilde{g}\right\|}$, we can either 1) use a bigger model (increase $d$); or 2) adopt a smaller batch size for local workers. Thus, as highlighted in Section \ref{DPlearning}, using a small batch size is one of our technical details that differs from vanilla DP-SGD.

\subsection{Second-stage Aggregation} 
In this part, we present our second-stage aggregation which does the job of ``\textit{complementary aggregation}''. In total, our first-stage and second-stage aggregation constitute our final protocol shown in Algorithm \ref{alg:filtergradient}.

\textbf{As a complement:} According to the resilience analysis for our first-stage aggregation, any acceptable upload is confined to lie in a special subspace described by Equation \ref{equ:upload_subspace}. To deceive our first-stage aggregation, the Byzantine attacker can also enforce its upload has the same form as $g = g'+z$ with $\left\|z\right\|\gg\left\|g'\right\|$ where $z$ is the DP noise and $g'$ is malicious component. Now the question is:

\textit{Is there an effective way for the server to differentiate benign uploads from Byzantine uploads based on the different nature of $\tilde{g}$ and $g'$?}


The answer is yes if the server can get some estimate on the true gradient. To have such a capability, we assume that the server has access to some auxiliary data that can be used to compute the gradient during the training. Our empirical finding shows that \textit{two samples per class} are enough for our second-stage aggregation to be effective. The intuition is that benign $\tilde{g}$ should update the model towards roughly the same direction as the true gradient $\nabla F$ while the malicious one does not. Quantitatively speaking, with high confidence, $\mathbb{E}\left\langle \nabla F, \tilde{g} \right\rangle$ > $\mathbb{E}\left\langle \nabla F, g' \right\rangle$.
And the server can use the gradient of non-private data to approximate $\nabla F$. 


\textbf{Theoretical motivation:}\label{sec:inner_motivate} 
For simplicity, based on our DP protocol, at a certain iteration, considering one honest worker's upload is\footnote{We assume the batch size is 1. The iteration number is omitted for ease of notation as it is clear from the context.}:
\begin{equation}\label{equ:simplified_honest_uploading}
    g =\frac{\nabla f(x; w)}{\|\nabla f(x; w)\|} +z = \frac{\nabla F(w) + \xi}{\left\|\nabla F(w) + \xi\right\|} + z
\end{equation}
where $z \sim \mathcal{N}(0, \sigma^2 I)$ is the DP noise and $\nabla f(x; w)$ can be written as $\nabla F(w) + \xi $ where $\xi$ is random noise due to data sampling. We compute the expectation of the inner product between $g$ and the true gradient $\nabla F(w)$, which has the following inequality (the proof is given by A.5.1 in \href{ https://github.com/zihangxiang/-Practical-Differentially-Private-and-Byzantine-resilient-Federated.git}{supp. material}.).

\begin{equation}\label{equ:expecation_of_inner_product_lb}
    \begin{aligned}
         \mathbb{E}\left\langle \nabla F(w), g \right\rangle = & \mathbb{E}\left\langle \nabla F(w), \frac{\nabla F(w) + \xi}{\left\|\nabla F(w) + \xi\right\|} \right\rangle \\
         \geq & \frac{\mathbb{E}\left \|\nabla F(w)  \right\|}{3} - \frac{8\mathbb{E}\left\| \xi \right\|}{3},\\
    \end{aligned}
\end{equation}
where the expectation is taken over the randomness of the data sampling and DP noise. Note that Equation (\ref{equ:simplified_honest_uploading}) is the special case where the honest worker only uses one data sample to compute the gradient and in the general case where many data samples are used, Equation (\ref{equ:expecation_of_inner_product_lb}) still holds (expectation is linear with respect to sum operation). Again, we do not have such bound if using clipping.

We then consider  $\mathbb{E}\left\langle \nabla F,g' \right\rangle$ for 
the attacks we consider: 1) for Gaussian attack, $\mathbb{E}\left\langle \nabla F(w), g \right\rangle = 0$; 2) for Label-flipping attack, we hypothesize that $\mathbb{E}\left\langle \nabla F(w), g \right\rangle \leq 0$ as such Byzantine gradient is to destroy our learning; 3) for Optimized Local Model Poisoning attack, we have $\mathbb{E}\left\langle \nabla F(w), g \right\rangle < 0$ as this is the goal of such attack. In total, for the three attacks, we have $\mathbb{E}\left\langle \nabla F(w), g \right\rangle \leq  0$.

As $\mathbb{E}\left\| \xi \right\|^2$ is bounded, we can be confident that at least at the early phase of training, $\mathbb{E}\left \|\nabla F(w)  \right\|$ is large enough to satisfy $\frac{\mathbb{E}\left \|\nabla F(w)  \right\|}{3} - \frac{8\mathbb{E}\left\| \xi \right\|}{3} > 0$. Our empirical result give positive evidence on the 
correctness of $\frac{\mathbb{E}\left \|\nabla F(w)  \right\|}{3} - \frac{8\mathbb{E}\left\| \xi \right\|}{3} > 0$. Thus, we can use this to filter Byzantine uploads and this is the foundation for our second-stage aggregation.


\begin{algorithm}[!ht]
\caption{$\mathbf{FilterGradient}( \{g^t_i|i=1,2,\cdots,n\},w_{k-1})$}\label{alg:filtergradient}
\begin{algorithmic}[1]
\small
\renewcommand{\algorithmicrequire}{\textbf{Input:}}
\renewcommand{\algorithmicensure}{\textbf{Output:}}
\Require gradients from each worker $i$ at the $t$-th iteration $g^t_i$, model $w^{t-1}$, server-hold dataset $D_p$, 
the loss function $f(;)$,  
server-maintained score list $\mathbb{S}$, server's belief of honest worker ratio $\gamma$
    
    \For{$i=1,2,\cdots,n$} \textbf{in  parallel}
        \State $g^t_i \gets \mathbf{FirstAGG}(g^t_i)$ 
    \EndFor
    
    \State Server computes $g^t_{s} \gets \nabla f(D_p;w^{t-1})$\label{alg:filtergrad_server_estimate}\label{alg:filtergradient_s_begin}
    \State Server initialize $\mathbb{S}_{tmp}=[0,0,\dots,0]$
    \For{$i=1,2,\cdots,n$} \textbf{in  parallel}\label{alg:filter_grad_cos_begin}
        
        \State $\mathbb{S}_{tmp}\left[ i\right] = \left\langle g^t_i, g^t_{s} \right\rangle $ \Comment{Motivated by our analysis in Section \ref{sec:inner_motivate}}
    \EndFor\label{alg:filter_grad_cos_end}
    \State $\hat{\mu} \gets$ average of top $\lceil\gamma n\rceil$ scores in $\mathbb{S}_{tmp}$  \label{alg:filter_grad_top_average}

    
    \For{$i=1,2,\cdots,n$} \textbf{in  parallel}
        \State $\mathbb{S}_{tmp}\left[ i\right] \gets 0$ if $\mathbb{S}_{tmp} \left[ i\right] <  \hat{\mu}$\label{alg:filter_grad_supress}
        \State $\mathbb{S}\left[ i\right] \gets \mathbb{S}\left[ i\right] + \mathbb{S}_{tmp}\left[ i\right]$\label{alg:filter_grad_acc_end}
    \EndFor
    
    \State Select those upload inside $\{g^t_i|i=1,2,\cdots,n\}$ which correspond to top $\lceil\gamma n\rceil$ scores inside $\mathbb{S}$ to form set $G^t_s$\label{alg:filtergradient_s_end}
    
    
    
    

    
\Ensure $G^t_s$
\end{algorithmic}
\end{algorithm}

\subsection{Final Byzantine-resilient Protocol}
\textbf{Combining all stages:} According to the above statements, we design our second-stage aggregation which is shown in line \ref{alg:filtergradient_s_begin}-\ref{alg:filtergradient_s_end} in Algorithm \ref{alg:filtergradient}. In line \ref{alg:filtergrad_server_estimate}, the server gets an estimation on $\nabla F$ by computing the gradient using some non-private data; Since the server has a prior belief that at least $\lceil\gamma n\rceil$ workers are honest, in line \ref{alg:filter_grad_top_average}, server gets the average on top $\lceil\gamma n\rceil$ inner product scores among all scores computed by current upload of each worker (line \ref{alg:filter_grad_cos_begin}-\ref{alg:filter_grad_cos_end}). This  average is used as the threshold to suppress scores lower than it to zero in line \ref{alg:filter_grad_supress}. By processing all scores by using the threshold, we can suppress all scores corresponding to Byzantine uploads and preserve the  benign ones; the processed scores are accumulated in line \ref{alg:filter_grad_acc_end} to be used to differentiate benign uploads from Byzantine ones which is described in line \ref{alg:filtergradient_s_end}. Then the selected vectors are returned for the model update. 

\textbf{Novelties:} Our first-stage aggregation is the first aggregation rule leveraging the aforementioned DP properties. Leveraging auxiliary data to aid Byzantine aggregation is not new \cite{caofltrust,regatti2020bygars}, nonetheless, our second-stage aggregation differs from all existing approaches in 1) theoretical support: we have a solid theoretical explanation while previous work stands on heuristics; 2) differentiation metric: we use inner product while previous work use cosine similarity (cosine similarity never leads to the lower bound in Equation \ref{equ:expecation_of_inner_product_lb}); 3) the way to integrate any upload into model update: by a unifying language, in our protocol, the weight assigned to any upload is binary (1 or 0), while existing work use real-valued weights according to computed similarities, we find that when DP is enforced (noise is added to the upload), assigning real-valued weights to any upload results in further biasing gradient which leads to rubbish model update. 

\subsection{Byzantine Attacks to Our Protocol}\label{sec:local_attack}
Recall that our attacker is the stronger version. In consistency with such consideration and to test the limit of our protocol's resilience, we stand in the perspective of an attacker and form possible attacks based on our already-released protocol.

\textbf{Attacker's response:} First, the attacker has to pass our first-stage aggregation to possibly have a malicious impact. Hence, he only has 2 possible guidelines in general:

\hypertarget{G1}{{\it Guideline 1}}: The attacker must first generate a $d$-element ordered sequence according to Theorem \ref{thm:resilience}. Then, the attacker will form any permuted version of such sequence to be malicious. If the attacker is content with any permutation, this would be \textit{Gaussian attack} \cite{zhu2022bridging, regatti2020bygars} as mentioned before. If the attacker aims to find any particular order, he will fail because it incurs $\mathcal{O}(d!)$ computation complexity.

{\it Guideline 2}: Just like the honest workers' upload, the attacker can make his malicious upload has the form (or can be decomposed to such form)as $g = g'+z$ with $\left\|z\right\|\gg\left\|g'\right\|$ to pass the KS test.

For completeness, we will not only test on \textit{Gaussian attack} but also test on other attacks which comply with Guideline 2. Following previous work, we will include \textit{Label-flipping attack} \cite{fang2020local, caofltrust} and \textit{Optimized Local Model Poisoning attack} \cite{fang2020local}. The former has been described in previous sections and we will explain how to form the latter attack in the following.

The attacker forms Optimized Local Model Poisoning attack by the following meta procedure: 

\begin{itemize}
    \item Infer the aggregated result $g_r$ by applying the aggregation rule on all benign uploads; 
    \item Based on the result and the aggregation rule, he forms his Byzantine upload which passes the aggregation and makes the final aggregated result have the inverse direction compared to $g_r$.
\end{itemize}

Accordingly, the goal of the adversary is to pass the first-stage aggregation to be possibly malicious further. We formally summarize such strategy as the following optimization problem:

\begin{equation}\label{equ:attacker's goal}
\begin{aligned}
    \min_{\{g_{M_i}\}}\quad & S_c\left( \sum g_{M_i} + \sum g_{B_j}, \sum g_{B_j}\right)\\
    {s.t.} \quad & \left\|\mathbf{FirstAGG}(g_{M_i})\right\|>0, \\
\end{aligned}
\end{equation}
where the constraint means that the Byzantine upload can pass our first-stage aggregation, $\{g_{M_i}\}$ are all Byzantine uploads by the Byzantine attacker, and $\{g_{B_j}\}$ are all benign uploads by honest workers. The function $S_c(A,B)=\frac{AB}{\left\|A\right\|\left\|B\right\|}$ calculates the cosine similarity between two vectors. Suppose all benign uploads are $g_{B_1},g_{B_2},\cdots,g_{B_m}$ by $B_m$ honest workers and all Byzantine uploads are $g_{M_1},g_{M_2},\cdots,g_{M_n}$ by $M_n$ 
Byzantine workers. According to Equation (\ref{equ:attacker's goal}), the attacker aims to reach the following goal:
    
\begin{equation}
    \sum g_{M_i} = -(1+\lambda) \sum g_{B_j}
\end{equation}      
where $\lambda > 0$ is a positive number. This leads to the term $\sum g_{M_i} + \sum g_{B_j} = -\lambda \sum g_{B_j}$ results in the inverse direction compared to $\sum g_{B_j}$. By setting:
\begin{equation}
     g_{M_1}=g_{M_2}=\cdots=-\frac{(1+\lambda)}{M_n}\sum g_{B_j}
\end{equation}
this goal is reached. And setting $\lambda = \frac{M_n}{\sqrt{B_m}} - 1$ will let Byzantine uploads pass our first-stage aggregation (one can check that all malicious and benign upload behaves the same when applying our first-stage aggregation on them). 
Note that to be able to perform such an attack, we need $M_n > \sqrt{B_m}$ (because $\lambda > 0$), that is, such a strong attack only exists when the number of Byzantine workers is sufficiently large.

Note that we do not simulate the Optimized Local Model Poisoning attack compromising our second-stage aggregation because the attacker must know the serve-hold auxiliary dataset, and this means that the attacker must fully control the server which is unrealistic. Also note that, for the goal in Equation \eqref{equ:attacker's goal}, we set it to be the inverse to the sum of all benign uploads. This is because such a goal leads to an efficient solution that can be tolerated by the attacker. In fact, the attacker can choose its goal freely as long as the constraint in  Equation \eqref{equ:attacker's goal} is 
satisfied. However, other goals may not lead to an efficient solution. For instance, if the attacker chooses it to be orthogonal, the attacker is faced with the hard problem as discussed in \hyperlink{G1}{Guideline 1}.

\textbf{Discussion on the adaptive attack:} There exists another attack that copies benign uploads by honest workers for some iterations and suddenly turns to be malicious after that. We call this attack as \textit{adaptive attack}. The way the attacker is malicious can be any instantiation of the previous three attacks we mentioned before. We will also include this attack in our experiment for completeness. Note that although Optimized Local Model Poisoning attack seems to be more advanced than Gaussian attack and Label-flipping attack, it is unclear which attack is most successful on our protocol before the experiment.


\textbf{Discussion on excluded attacks:} Optimized Local Model Poisoning attack performs well on attacking various existing Byzantine defense methods  \cite{fang2020local}. Another similar recent work \cite{shejwalkar2021manipulating} adopts the attacking intuition (the meta procedure mentioned above) of the Optimized Local Model Poisoning attack in other cases where the attacker's power is more limited (the attacker is weaker). Hence, here we only adopt the Optimized Local Model Poisoning attack in our experiments.

To the best of our knowledge, many other attacks can be trivially defended by our protocol,  such as the attacks that have been considered in the existing work: ``A little'' attack \cite{baruch2019little} and ``Inner'' attack \cite{byz_xie2020fall}. ``A little'' attack involves estimating the coordinate-wise mean and the s.t.d. of benign uploads to form its attack. However, our learning protocol enforces that the DP noise is dominating, hence knowing benign uploads gains the attacker no useful information when forming ``A little'' attack. Most importantly, naively applying such an attack will end up being rejected by first-stage aggregation. This shows the power of our protocol.

\vspace{-0.1cm}
\subsection{Discussions}\label{sec:byz_discussion}

\noindent\textbf{First-stage aggregation provides critical robustness:} 
Only using our second-stage aggregation to aggregate all worker's uploads is not enough, because, due to randomness, it is not guaranteed that Byzantine upload will never be selected for model update, and selected Byzantine upload could destroy our model in just one iteration as it is arbitrary. In contrast, there exists no such concern when we apply our first-stage aggregation, according to previous resilience analysis for our first-stage aggregation, it enforces any upload (including malicious ones) $g$ which passes the filtering to have the form $g=\hat{g}+z$ with $\left\|z\right\|\gg\left\| \hat{g}\right\|$ where $\hat{g}$ is strictly norm-bounded and $z$ is the DP noise. For all malicious uploads, strictly norm-bounded $\hat{g}$ means their detrimental impact is bounded.

\vspace{0.2cm}

\noindent\textbf{DP-Byzantine-robustness interaction:} we do not consider DP and Byzantine-robustness in isolation. Instead, our whole protocol is formed by leveraging each other's properties.

As mentioned previously in our design strategy, we use our first-stage rule to ``\textit{constrain}'' the way that any upload should behave by re-designing our DP protocol so that any Byzantine upload violating it will be immediately rejected. Hence, other than only protecting privacy, this refactoring on DP also provides the first-stage Byzantine-robustness. To deal with those Byzantine uploads that pass our first-stage aggregation, we further design our second-stage rule to do the ``\textit{complementary aggregation}'' by leveraging the properties of our refactored DP protocol. In total, our privacy protocol and the robust aggregation rule are aware of each other, leading to a solution that is both privacy-preserving and Byzantine-resilient. 


\section{Theoretical Guarantees}\label{sec:thm}
We provide theoretical guarantees on privacy, utility, and Byzantine robustness of our protocol in this section. For convenience, we assume that the dataset of each worker has the same size which is denoted as $|D|$, and the size of non-private data held by the server is $|D_P|$. We also denote  $w^*=\arg\min_{w \in \Theta} F(w)$. 

\noindent\textbf{Privacy guarantee:} We have the following privacy guarantee. 

\begin{theorem}\label{thm:privacy_guarantee}
(Privacy Guarantee) There exist constants $c_1$ and $c_2$ such that given the sampling rate $q=\frac{b_c}{|D|}$ and the number of iteration steps $T$. For each worker, Algorithm \ref{alg:scen1} is $(\epsilon,\delta)$-DP for any $\delta>0$ and $\epsilon < c_1q^2T$ if 
    $\sigma \geq c_2\frac{q\sqrt{T\ln \frac{1}{\delta}}}{\epsilon}.$
\end{theorem}
\begin{proof}
See A.5.2 in \href{ https://github.com/zihangxiang/-Practical-Differentially-Private-and-Byzantine-resilient-Federated.git}{supp. material}.
\end{proof}

\noindent\textbf{Utility and Byzantine robustness:} Theorem \ref{thm:Byzantine_roubustness} shows the utility  and robustness of Byzantine resilience of our algorithm. Before formally introducing Theorem\ref{thm:Byzantine_roubustness}, we present  some assumptions, which are commonly used in the previous work on optimization and Byzantine-robust learning \cite{byz_chen2017distributed,caofltrust}.  

\begin{assumption}\label{ass:1}
The expected loss function $F(w)$ is $\mu$-strongly convex and differentiable over the space $\Theta$ with L-Lipschitz continuous gradient. Formally, we have the following for any $w, \widehat{w} \in \Theta$:
$$
\begin{array}{c}
F(\widehat{w}) \geq F(w)+\langle\nabla F(w), \widehat{w}-w\rangle+\frac{\mu}{2}\|\widehat{w}-w\|^{2} \\
\|\nabla F(w)-\nabla F(\widehat{w})\| \leq L\|w-\widehat{w}\|. 
\end{array}
$$
Moreover, the empirical loss function $f(D, w):=\frac{1}{|D|}\sum_{x\in D}f(x; w)$ is $L_{1}$-Lipschitz continuous with high probability. Formally, for any $\zeta \in(0,1)$, there exists an $L_{1}$ such that:
$$
\operatorname{Pr}\left(\sup _{w, \widehat{w} \in \Theta: w \neq \widehat{w}} \frac{\|\nabla f(D, w)-\nabla f(D, \widehat{w})\|}{\|w-\widehat{w}\|} \leq L_{1}\right) \geq 1-\frac{\zeta}{3}
$$
\end{assumption}
\begin{assumption}\label{ass:2}
The gradient of the empirical loss function $\nabla f\left(D, w^{*}\right)$ at the optimal global model $w^{*}$ is bounded. Moreover, the gradient difference $h(D, w)=\nabla f(D, w)-$ $\nabla f\left(D, w^{*}\right)$ for any $w \in \Theta$ is bounded. Specifically, there exist positive constants $\sigma_{1}$ and $\gamma_{1}$ such that for any unit vector $\boldsymbol{v},\left\langle\nabla f\left(D, w^{*}\right), \boldsymbol{v}\right\rangle$ is sub-exponential with $\sigma_{1}$ and $\gamma_{1}$; and there exist positive constants $\sigma_{2}$ and $\gamma_{2}$ such that for any $w \in \Theta$ with $w \neq w^{*}$ and any unit vector $\boldsymbol{v}$, $\langle h(D, w)-\mathbb{E}[h(D, w)], \boldsymbol{v}\rangle /\left\|w-w^{*}\right\|$ is sub-exponential with $\sigma_{2}$ and $\gamma_{2}$. Formally, for all $|\tau| \leq 1 / \gamma_{1}, |\tau| \leq 1 / \gamma_{2}$, we have:
$$
\begin{array}{c}
\sup _{\boldsymbol{v} \in \boldsymbol{B}} \mathbb{E}\left[\exp \left(\tau\left\langle\nabla f\left(D, w^{*}\right), \boldsymbol{v}\right\rangle\right)\right] \leq e^{\sigma_{1}^{2} \tau^{2} / 2} \\
\sup _{w \in \Theta, \boldsymbol{v} \in \boldsymbol{B}} \mathbb{E}\left[\exp \left(\frac{\tau\langle h(D, w)-\mathbb{E}[h(D, w)], \boldsymbol{v}\rangle}{\left\|w-w^{*}\right\|}\right)\right] \leq e^{\sigma_{2}^{2} \tau^{2} / 2}
\end{array}
$$
where $\boldsymbol{B}$ is the unit sphere $\boldsymbol{B}=\{\boldsymbol{v}:\|\boldsymbol{v}\|=1\}$
\end{assumption}

Note that the strongly convex and Lipschitz continuous conditions in Assumption \ref{ass:1} are widely adopted in the convergence analysis of optimization algorithms, and these conditions indicate the largest eigenvalue of  the Hessian matrix of the loss  function is between $\mu$ and $L$. Assumption \ref{ass:2} indicates that the gradient $\nabla f(D, w^*)$ is quite close to its expectation $\mathbb{E}[\nabla f(D, w^*)]=0$, 
and the difference $h(D, w)$ concentrates to its expectation with high probability.




\begin{theorem}\label{thm:Byzantine_roubustness}
For an arbitrary number of Byzantine
workers, the difference between the global model learned by
Algorithm \ref{alg:scen1} and the optimal global model $w^*$ under no attacks is bounded. Specifically, 
if the parameter space $\Theta \subseteq B(0,r\sqrt{d})$, {\em i.e.}, it is contained in a ball with radius $r\sqrt{d}$ and $\nabla F(w^*)=0$.  Set $\sigma$ as in Theorem \ref{thm:privacy_guarantee},  $T= \mathcal{O} \left(\frac{1}{\rho} \ln{ \left(\sqrt{n}|D| \sqrt{\left|D_{0}\right|}\right)}\right)$ and  $\eta_{t-1}\leq \left\|g_s^{t-1}\right\|_2\eta_0$ with fixed $\eta_0\leq \frac{\mu}{2L^2}$ in the $t$-th iteration in Algorithm \ref{alg:scen1}, then 
if $n, |D_p|$ and $|D|$ are sufficiently large and $\eta_0$ is sufficiently small such that 
\begin{equation}
    \begin{aligned}
      \sqrt{n} & \geq \tilde{\Omega}\left(\frac{\sqrt{d\ln{\frac{1}{\delta}}}}{\epsilon|D|}\cdot \max\left\{\sqrt{\ln{\frac{1}{\xi}}},\frac{\ln{\frac{1}{\xi}}}{ r\rho\sqrt{|D_p|}}\right\}\right)
    \end{aligned}
\end{equation}
and $\frac{\eta_0 \sigma_1}{\sqrt{|D_p|}} \leq \mathcal{O}\left( \frac{r\sqrt{d}}{\rho}\right)$ with $0<\rho<1$. 
Then, with probability at least $1-\xi$ with $\xi\in(0,1)$, we have:
\begin{equation}\label{equ:robustness_distance}
    \left\|w_T-w^*\right\| \leq \Tilde{\mathcal{O}}\left( \frac{1}{\rho^2} \frac{d\ln{\frac{1}\zeta}\sqrt{\ln{\frac{1}{\delta}}}\sigma_1}{ |D|\sqrt{nb_s}\epsilon} + \frac{1}{\rho} \frac{\sigma_1\sqrt{d\ln{\frac{1}{\zeta}}}}{\sqrt{|D_p|}}\right), 
\end{equation}
where the Big-$\tilde{\mathcal{O}}$ and Big-$\tilde{\Omega}$ notations omit other logarithmic terms. Here $\rho = 1- \sqrt{1-\frac{u^{2}}{4 L^{2}}}-32 \eta_{0} \Delta_{2}-3 \eta_{0} L$ with $\Delta_2=\sigma_2\sqrt{\frac{2}{|D_p|}}\sqrt{K_1+K_2}$ with $K_1=d\log \frac{\max\{L, L_1\}}{\sigma_2}$ and $K_2=\frac{d}{2}\log\frac{|D_p|}{d}+\log \frac{6\sigma^2 r\sqrt{|D_p|}}{\gamma_2\sigma_1\zeta}$. 
\end{theorem}

\begin{proof}
In A.5.3 in \href{ https://github.com/zihangxiang/-Practical-Differentially-Private-and-Byzantine-resilient-Federated.git}{supp. material}.
\end{proof}

 Theorem \ref{thm:Byzantine_roubustness} is on the robustness. Briefly speaking, as can be seen from Equation \eqref{equ:robustness_distance}, if $|D|, n$ and $|D_p|$ are large enough (nonetheless, our experiment shows that only a small number of non-private data will suffice), with some iteration number and stepsize, even there is an arbitrary number of Byzantine workers,
the final model we get will be close to the optimal  model (measured by the $l_2$ distance) with high probability.  

\section{Experimental Results}\label{sec:eval}

\subsection{Datasets and System Settings}
\textbf{Datasets and models:} We conduct experiments on MNIST \cite{mnist}, Colorectal \cite{colorectal}, Fashion \cite{fashion}, and USPS \cite{usps}. Details of all these benchmark datasets with various properties are summarised in \href{ https://github.com/zihangxiang/-Practical-Differentially-Private-and-Byzantine-resilient-Federated.git}{supp. material}.. Details of neural network setup are also in \href{ https://github.com/zihangxiang/-Practical-Differentially-Private-and-Byzantine-resilient-Federated.git}{supp. material}. Each experiment is repeated with different random seeds \{1, 2, 3\} and we report the min. max. and mean. All of our experiments are conducted under the same base learning rate $\eta_b=0.2$ (which will be 
explained later), batch size $b_c=16$ and momentum $\beta=0.1$. We set the number of epochs $T=\lceil10|D|/b_c\rceil$ for Colorectal and USPS, $T=\lceil8|D|/b_c\rceil$ for MNIST and Fashion.




\textbf{Data sample distribution:} We consider both \textit{i.i.d.} and \textit{non-i.i.d.} settings. To be specific, \textit{i.i.d.} is the case where each worker's local dataset follows the same distribution as the whole data population while \textit{non-i.i.d.} is the case where each worker's local dataset's distribution is arbitrary \cite{fedavg}. We simulate both settings following previous work \cite{caofltrust,fang2020local, guerraoui2021combining, zhu2022bridging}, and details are presented in \href{ https://github.com/zihangxiang/-Practical-Differentially-Private-and-Byzantine-resilient-Federated.git}{supp. material}.

For generating server-own auxiliary data, we only randomly sample 2 data samples per class from the validation dataset. As mentioned earlier, obtaining such a tiny amount of data is easy. Note that generating such auxiliary data is totally agnostic to the distribution of the whole data population while such data still enables our protocol's effectiveness (as will be confirmed by our experiments). Once the auxiliary data is generated, it is vacuous to compare the distribution of such auxiliary data to distributions of any other datasets, because the size of our auxiliary data is micro.


    
\textbf{Byzantine setup}: We fix the number of honest workers (20 for MNIST and Fashion, 10 for Colorectal and USPS), and vary the number of Byzantine workers ($0\%, 20\%, 40\%, 60\%,90\%$ of total). 

\textbf{Privacy settings}: We do experiments on  different privacy settings $\epsilon = \{2^{-3},2^{-2},2^{-1},2^{0},2^{1}\}$ while fixed $\delta = {1/|D_i|^{1.1}}$, where $|D_i|$ is the size of the local dataset possessed by worker $i$.


\textbf{Reference Accuracy}: 
The \textit{Reference Accuracy} is the testing accuracy of FL under the scenario where no Byzantine threat exists and FL only adopts DP (not adopting any Byzantine defense method). Compare any private and Byzantine-resilient protocol's performance to the \textit{Reference Accuracy}, many useful conclusions can be drawn: 


\begin{itemize}
    \item \textbf{Side-effect:} Apply a protocol under the scenario where there are no Byzantine threats, by comparing it with \textit{Reference Accuracy}, we know how much ``side-effect'' caused by that protocol. The ideal case is that we expect the ``medicine'' causes no additional harm to the ``patient'' with no ``illness''.

    \item \textbf{Efficacy:} For the scenario where there is a certain number of attackers, by comparing with \textit{Reference Accuracy}, we know how effectively a protocol defends the attack. The ideal case is that the ``medicine'' eradicates the ``illness'' (under such case, the performance should be the same as \textit{Reference Accuracy}).
\end{itemize}


\subsection{Claims and Experimental Evidence}

All of the attacks we consider have been tested. Based on the observation that our protocol remains resilient across all attacks and due to space limitation, we arbitrarily only present results for \textit{Label-flipping attack} under {\it i.i.d.} in the main body. All additional results for other attacks we consider under both {\it i.i.d.} and {\it non-i.i.d} settings is in \href{ https://github.com/zihangxiang/-Practical-Differentially-Private-and-Byzantine-resilient-Federated.git}{supp. material}.

\textbf{A quick overview:} We provide 7 claims with their corresponding evidence. By comparing with previous work, claim 1-2 show our contribution to DP learning and Byzantine resilience in their own track. Most importantly, recall our core aim is to ensure privacy and Byzantine resilience simultaneously, we use claims 3-7 to show its effectiveness.


\begin{figure*}[!ht] 
    \centering
    
    \includegraphics[width=0.90\linewidth]{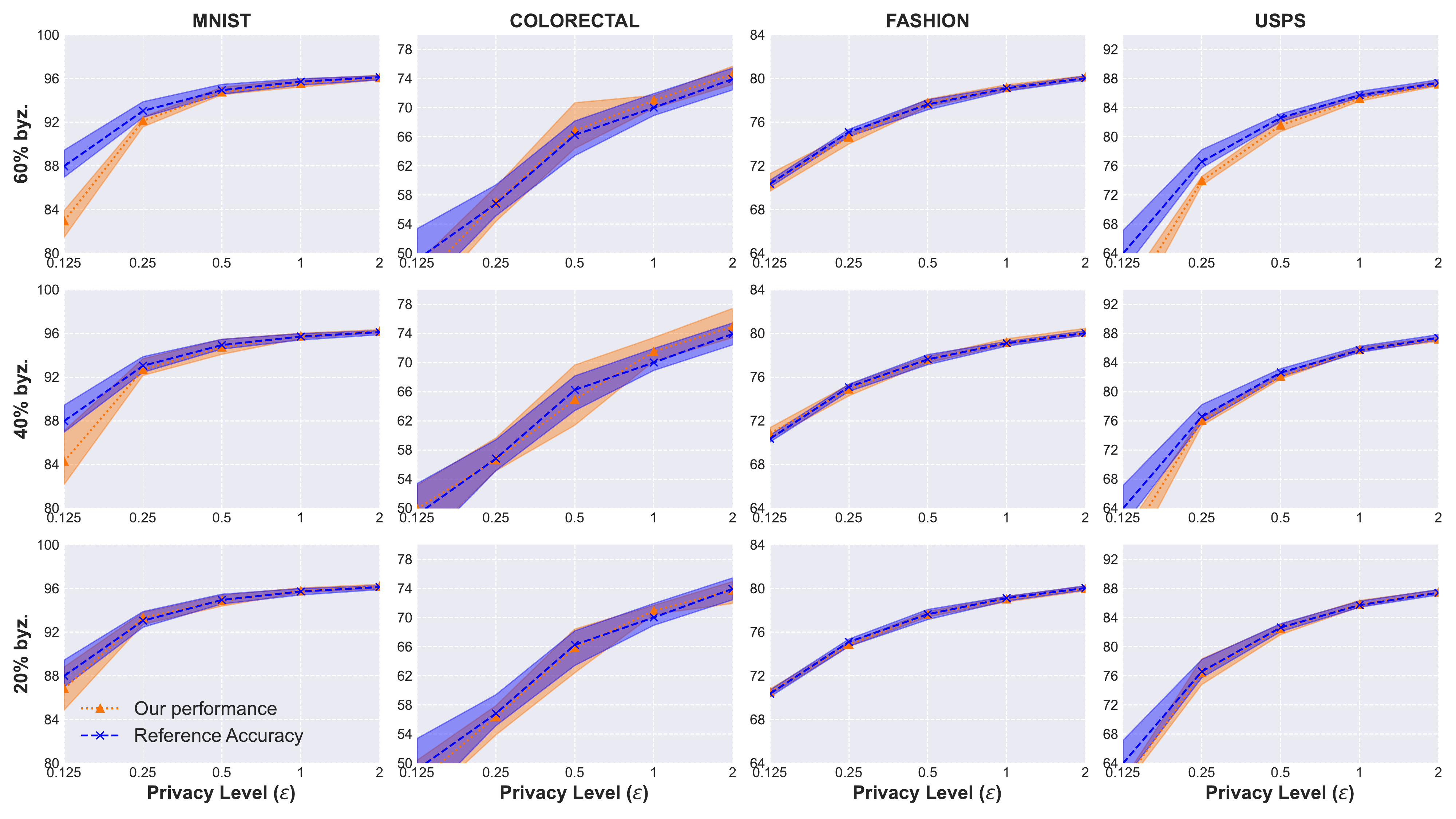}
    
    \caption{Byzantine-resilient performance (testing accuracy) under Label-flipping attack. The experiment is conducted under 3 different attacking levels ($20\%,40\%,60\%$ of the total workers are Byzantine).}
    
    \label{fig:dp_byz_iid_label}
\end{figure*}

\textbf{CLAIM 1:} Normalizing is better than vanilla DP-SGD (which uses clipping) at 1) gaining Byzantine resilience; 2) efficiently tuning hyper-parameter for DP deep learning. 

\textbf{Evidence:} A thought experiment will suffice. Recall that clipping is essentially normalizing gradient vectors with $\ell_2$-norm greater than $C$ to be $C$ exactly and leaving those vectors with $\ell_2$-norm smaller than $C$ untouched. If we are guaranteed that all gradient vectors' $\ell_2$-norm is greater than $C$, then normalizing and clipping only differ in the learning rate scale and are essentially equivalent to each other, {\em i.e.}, clipping with $C=2, \eta=0.1$ is the same as normalizing (to be unit $\ell_2$-norm) with $\eta=0.2$. This means that if we have that guarantee, clipping will also enjoy the lower bound in Equation \ref{equ:expecation_of_inner_product_lb} for gaining Byzantine resilience and will also enjoy our analysis for hyper-parameter tuning in Theorem \ref{thm:general_convergence}. 

However,~\textit{it is unfeasible to get a prior bound} on the gradient vector's norm for arbitrary deep-learning neural networks. Meanwhile, it is unclear whether clipping could lead to similar theoretical results which serve our purpose. Adopting normalizing circumvents such issues.




\textbf{CLAIM 2:} Our protocol outperforms existing solutions.

\textbf{Evidence:} We compare our protocol to previous work with the same aim (ensure privacy and Byzantine security simultaneously). We will show that our tailored Byzantine aggregation with DP outperforms previous solutions whose methodology is to naively apply off-the-shelf Byzantine aggregation with DP. And our result shows the contribution of our Byzantine aggregation rule.

For a fair comparison, we provide the results for the scenario where our privacy level is similar and our attacker is the same compared with existing solutions.


\begin{table}[!ht] 
\footnotesize
        \centering
            \begin{tabular}{c||c|c|c}
            \toprule
            Method & Byz./ Privacy & ``A little'' attack \cite{baruch2019little} & ``Inner'' attack \cite{byz_xie2020fall} \\
            \midrule
            \multirow{2}{*}{\cite{guerraoui2021combining}} & 40\%, $\epsilon= 3.46$ & $.61 $ & $.75 $  \\
            & 20\%, $\epsilon= 7.58$ & $.78 $ & $.79$ \\
            \midrule
            \multirow{2}{*}{Ours} & 60\%, $\epsilon= 2.00$ & $.79\pm.010 $ & $.80\pm.010 $  \\
            & 40\%, $\epsilon= 2.00$ & $.80\pm.005 $ & $.80\pm.005$ \\
            \bottomrule
           \end{tabular}
       
    \caption{Testing accuracy comparison with existing work \cite{guerraoui2021combining} on Fashion. 
    }
    \label{tab:combine_result}
\vspace{-0.5cm}
\end{table}

\begin{table}[!ht] 
\footnotesize
        \centering
            \begin{tabular}{c||c|c}
            \toprule
            Method & Byz./ Privacy & Gaussian attack  \\
            \midrule
            \multirow{2}{*}{\cite{zhu2022bridging}} & 10\%, $\epsilon= .21$ & $.20 $  \\
            & 10\%, $\epsilon= .40$ & $.43$  \\
            \midrule
            \multirow{2}{*}{Ours} & 60\%, $\epsilon= .125$ & $.86\pm.010 $  \\
            & 40\%, $\epsilon= .125$ & $.86\pm.010 $   \\
            \bottomrule
           \end{tabular}
      
    \caption{Testing accuracy comparison with existing work \cite{zhu2022bridging} on MNIST.
    }
    \label{tab:bridging_result}
\vspace{-0.5cm}
\end{table}

\textit{Comparison with \cite{guerraoui2021combining}:}
We compare our results with \cite{guerraoui2021combining} in  Table \ref{tab:combine_result}. We can see from Table \ref{tab:combine_result} that  \cite{guerraoui2021combining} only reaches $61\%$ accuracy under $40\%$ ``a little'' Byzantine attack \cite{baruch2019little} in the privacy setting ($\epsilon=3.46,  \delta=1.2\times10^{-4}$). We also notice that under the same privacy setting but a different Byzantine attack, \cite{guerraoui2021combining} achieves $75\%$ testing accuracy, and \cite{guerraoui2021combining} makes comments that ``a little'' attack is stronger against their defense.


Applying our Byzantine defense method under the same attacks, we get around $80\%$ testing accuracy when there are $60\%,40\%$ Byzantine workers in the privacy setting ($\epsilon=2, \delta=1.4\times10^{-4}$). Thus, we can gain much more utility compared to \cite{guerraoui2021combining} even when the majority worker are Byzantine and with even better privacy guarantee (we are ensuring $(\epsilon=2, \delta=1.4\times10^{-4})$-DP instead of $(\epsilon=3.46, \delta=1.2\times10^{-4})$-DP). The utility we gain is also better than \cite{guerraoui2021combining} under its weakest attack with a much weaker privacy guarantee: $(\epsilon=7.58, \delta=1.2\times10^{-4})$-DP.


\textit{Comparison with \cite{zhu2022bridging}:} 
As can be seen from Table \ref{tab:bridging_result}, the method in \cite{zhu2022bridging} reaches $43\%$ testing accuracy on MNIST when there are only $10\%$ Byzantine workers under the privacy setting ($\epsilon = 0.4, \delta = 0$). As a comparison, our learning protocol provides $86\%$ testing accuracy when there are $60\%$ Byzantine workers under privacy setting ($\epsilon=0.125, \delta=1.4\times10^{-4}$). We gain much more utility even when the majority of workers are Byzantine, which is impossible for \cite{zhu2022bridging} to accomplish due to their intrinsic limitation of aggregation methods.

\begin{table}[!ht] 
\footnotesize
    \begin{subtable}[h]{\columnwidth}
        \centering
        \begin{tabular}{c||c|c|c|c|c|c|c|c|c|c|c|c}
        \toprule
        \multirow{2}{*}{} & \multicolumn{3}{|c|}{MNIST} & \multicolumn{3}{|c|}{COLOR.} & \multicolumn{3}{|c|}{FASHION} &  \multicolumn{3}{|c}{USPS} \\
        \cline{2-13}
           & \multicolumn{3}{|c|}{$\epsilon=$} & \multicolumn{3}{|c|}{$\epsilon=$} & \multicolumn{3}{|c|}{$\epsilon=$} &  \multicolumn{3}{|c}{$\epsilon=$}\\
 
       & $\frac{1}{8}$&$\frac{1}{2}$&$2$  & $\frac{1}{8}$&$\frac{1}{2}$&$2$ & $\frac{1}{8}$&$\frac{1}{2}$&$2$ & $\frac{1}{8}$&$\frac{1}{2}$&$2$  \\
        \midrule
        \textit{RA} & .88&.95&.96 & .49&.66&.74 & .69&.77&.80 & .64&.82&.87 \\
        \midrule
        \textit{zero} & .85&.94&.96 & .44&.67&.74 & .69&.77&.80 & .58&.81&.87 \\
        
        \bottomrule
        \end{tabular}
    \end{subtable}
    
     \caption{Experimental result on the test for the ``side-effect'' our protocol brings. \textit{RA} stands for \textit{Reference Accuracy} and \textit{zero} stands for the scenario where all original 60\% Byzantine workers turn to behave honestly (hence we have \textit{zero} attackers) while our protocol is still applied. The performance (testing accuracy) results from taking the average of three runs with different seeds. }
     \label{tab:no_over_medication}
\vspace{-0.5cm}
\end{table}

\textbf{CLAIM 3:} Our protocol brings no ``side-effect'' even when there is no Byzantine attack.

\textbf{Evidence:} we design the following experiment to test whether our protocol brings any ``side-effect''. Let 60\% of workers be Byzantine, however, those Byzantine workers do not perform any attack. Instead, they behave just like all honest workers. The server still follows its prior belief that only 40\% of workers are trustworthy. Our results are shown in Table \ref{tab:no_over_medication}.

We can see that other than at the extreme privacy level ($\epsilon=1/8=.125$), our protocol's performance is almost identical to the \textit{Reference Accuracy}, hence incurring no ``side-effect''. We indeed observe a noticeable accuracy drop when at $\epsilon=1/8=.125$, this is because in such extreme case, noise becomes so overwhelming that the training itself is not stable.

\textbf{CLAIM 4:} Our protocol eradicates Byzantine attacks if not facing extreme privacy requirements.

\textbf{Evidence:} Figure \ref{fig:dp_byz_iid_label} shows the performance of our method. The testing accuracy almost always aligns with the Reference Accuracy. Such a phenomenon can be observed not only across different privacy levels but also across different datasets. 

The most discrepant results are observed for USPS and MNIST datasets when there are $60\%$ Byzantine workers in the high privacy regime with $\epsilon=0.125$. This is because, at this extreme privacy level, significant noise is added to the gradient, We are getting less confident in differentiating benign uploads from Byzantine ones when we are at such a high privacy level, fortunately, our first-stage aggregation guarantees that malicious upload has limited detrimental impact even if selected. 

We also observe that results for Colorectal present a larger variance than the rest datasets. This is because the dataset size is much smaller (only 5,000 samples in total) than the rest 
and it is limited by the intrinsic limitation that DP learning requires large-scale data.

\begin{figure*}[!ht] 
    \centering

     \includegraphics[width=0.90\linewidth]{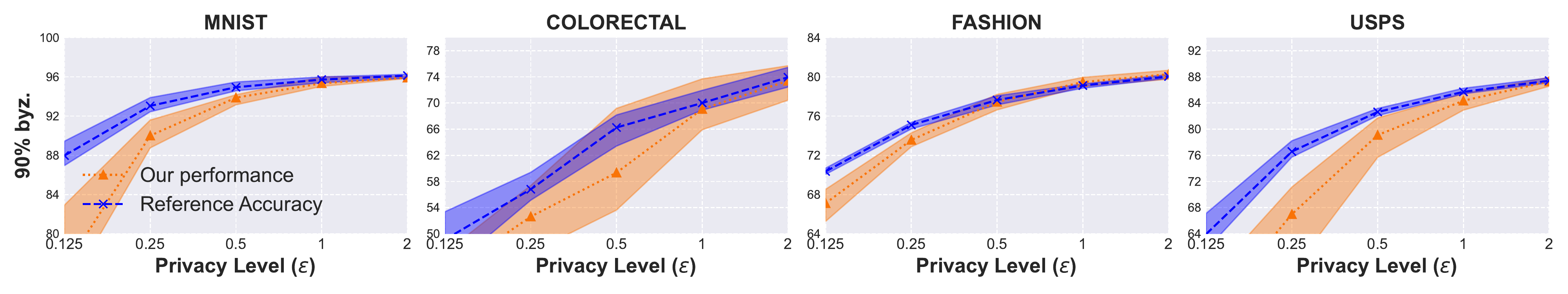}

    \caption{Byzantine-resilient performance (testing accuracy) when $90\%$ workers are Label-flipping Byzantine attackers.}
    
    \label{fig:dp_byz_iid_90_label}
\end{figure*}

\textbf{CLAIM 5:} Our protocol remains robust against majority attack.

\textbf{Evidence:} Results are shown in Figure \ref{fig:dp_byz_iid_90_label}. We can observe that even when $90\%$ workers are Byzantine (we have also simulated more stringent cases where $95\%$, $99\%$ workers are Byzantine, results can be found in \href{https://github.com/zihangxiang/-Practical-Differentially-Private-and-Byzantine-resilient-Federated.git}{\it supp. material}),
similar results can be observed compared with the cases where there are $60\%,40\%,20\%$ Byzantine workers. We observe a noticeable accuracy drop for certain datasets when $\epsilon=0.125$ and $\epsilon=0.25$ due to overwhelming random noise which guarantees high privacy. For $\epsilon\geq0.5$, we still gain privacy and Byzantine resilience without hurting too much performance.


\begin{figure*}[!ht] 
    \centering

     \includegraphics[width=0.9\linewidth]{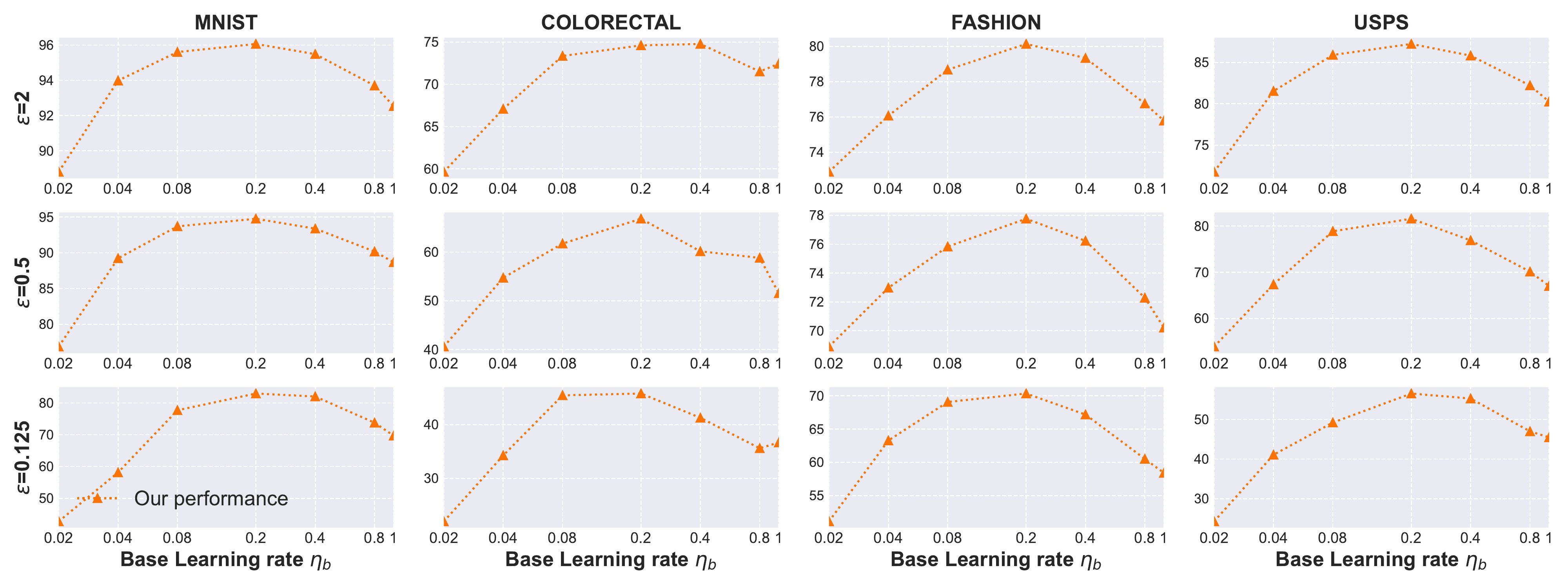}
   
    \caption{ Our hyper-parameter tuning results when facing $60\%$ Label-flipping attackers.}
    \label{fig:lr_tuning_s1_iid_label}
\end{figure*}

\textbf{CLAIM 6:} As a cherry on top, our protocol enables efficient hyper-parameter tuning by saving quadratic efforts.

\textbf{Evidence:} For a typical DP deep learning task, vanilla DP-SGD's running task spans on the 3-dimensional tuple $(\eta, C, \epsilon)$. In contrast, adopting normalizing together with our tuning strategy only needs to tune $\eta$ for one arbitrary $\epsilon$. That is, we only need to tune the learning rate $\eta_b$ for one privacy level $\epsilon$ with the corresponding noise multiplier $\sigma_b$, then we can use the learning rate $\eta = \frac{\eta_b\sigma_b}{\sigma}$ for any other privacy level 
with noise multiplier $\sigma$. We call $\eta_b$ and $\sigma_b$ at the privacy level we are tuning as \textit{``base learning rate''} and \textit{``base noise multiplier''}.

To evaluate the effectiveness of such a strategy, it suffices to confirm that
if we find the optimal base learning rate for one privacy level, we also find the optimal learning rate for other privacy levels by setting the learning rate according to such a strategy.
In this sense, we first choose the base case of $\sigma_b=0.79$ (corresponding to $\epsilon=2$). Then, for each privacy level, we tune the learning rate with respect to different base learning rates (the actual learning rate is computed according to the above strategy). In our experiment, we vary the base learning rate among $\{0.02,0.04,0.08,0.2,0.4,0.8,1\}$, so the actual learning rate will be $\{\frac{0.02\sigma_b}{\sigma},\frac{0.04\sigma_b}{\sigma},\frac{0.08\sigma_b}{\sigma},\frac{0.2\sigma_b}{\sigma},\frac{0.4\sigma_b}{\sigma},\frac{0.8\sigma_b}{\sigma},\frac{\sigma_b}{\sigma}\}$ for a specific privacy level with noise multiplier $\sigma$. 

Results on MNIST 
are shown in Figure \ref{fig:lr_tuning_s1_iid_label}, and we can see that for all the privacy levels we considered, the optimal point is the same ($\frac{0.2\sigma_b}{\sigma}$ for MNIST), and a similar phenomenon can also be observed on the other three datasets.

\begin{table}[!ht] 
\footnotesize
    \begin{subtable}[h]{\columnwidth}
        \centering
        \begin{tabular}{c||c|c|c|c|c|c|c|c}
        \toprule
        \multirow{2}{*}{TTBB} & \multicolumn{2}{|c|}{MNIST} & \multicolumn{2}{|c|}{COLOR.} & \multicolumn{2}{|c|}{FASHION} &  \multicolumn{2}{|c}{USPS} \\
        \cline{2-9}
           & \multicolumn{2}{|c|}{$\epsilon=$} & \multicolumn{2}{|c|}{$\epsilon=$} & \multicolumn{2}{|c|}{$\epsilon=$} &  \multicolumn{2}{|c}{$\epsilon=$}\\
           & $2$ & $.125$ & $2$ & $.125$& $2$ & $.125$  & $2$ & $.125$ \\
        \midrule
         0 & $.96$ & $.82$ & $.74$ & $.45$ & $.80$ & $.68$  &$.86$ & $.60$ \\
        .2 & $.96$ & $.82$ & $.74$ & $.41$ & $.80$ & $.68$  &$.86$ & $.60$ \\
        .4 & $.96$ & $.81$ & $.73$ & $.45$ & $.80$ & $.68$  &$.86$ & $.57$ \\
        .6 & $.96$ & $.81$ & $.73$ & $.44$ & $.80$ & $.69$  &$.86$ & $.57$ \\
        .8 & $.96$ & $.82$ & $.73$ & $.43$ & $.80$ & $.69$  &$.86$ & $.60$ \\
        \bottomrule
        \end{tabular}
    \end{subtable}

     \caption{Under Label-flipping attack with different TTBB. }
     \label{tab:TTBB_label}
\vspace{-0.5cm}
\end{table}

\begin{figure*}[!ht] 
    \centering
    \includegraphics[width=0.90\linewidth]{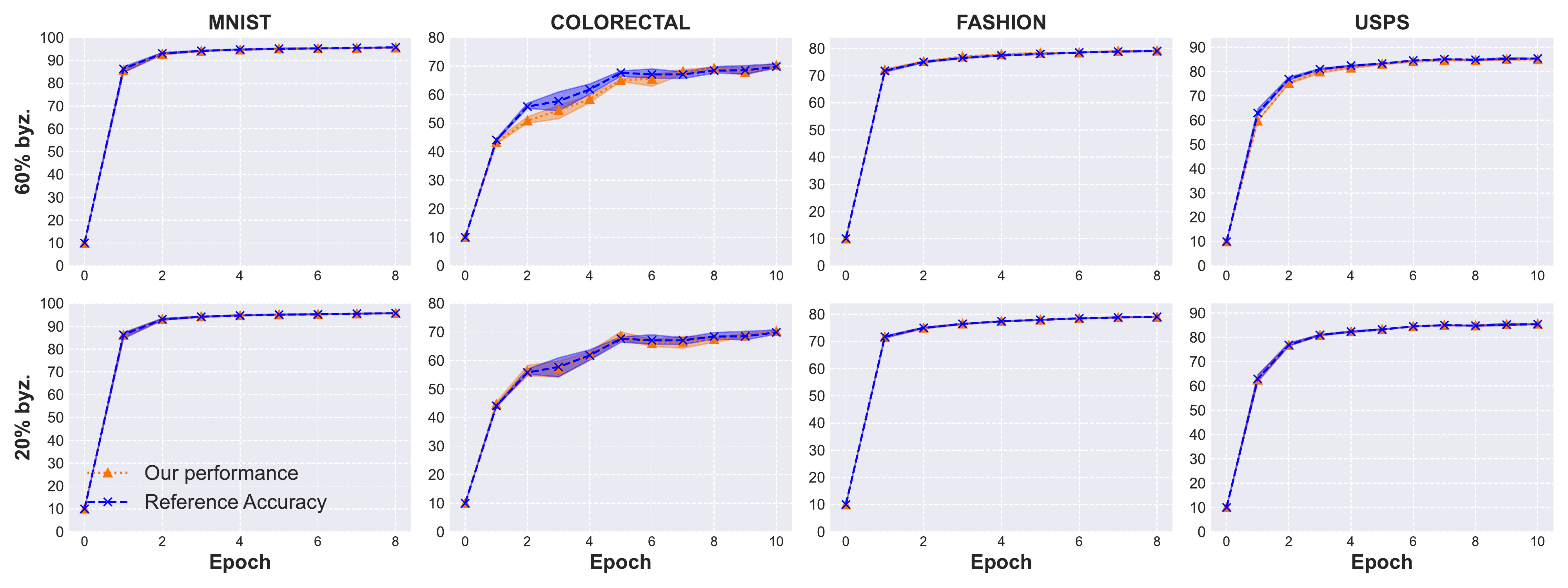}
    \caption{Byzantine-resilient convergence curves (testing accuracy) under Label-flipping attack (considering $20\%,60\%$ of the total workers are Byzantine, fixing $\epsilon=1$). 
     }
    \label{fig:dp_byz_iid_label_convergence}
\end{figure*}

\textbf{CLAIM 7:} Our protocol remains resilient against adaptive attack.

\textbf{Evidence:} Our robust and private learning framework is also resilient to \textit{adaptive attack}. We evaluate that by letting $60\%$ Byzantine workers be honest via copying the uploads of some random honest workers from the beginning of training and turning to Byzantine at different iterations to see if they can possibly have a significant impact. Results are shown in Table \ref{tab:TTBB_label}. The first column represents the Time To Be Byzantine  (TTBB), {\em i.e.}, if the total iteration is $T$, 0.2 TTBB means that Byzantine workers behave honestly within the first $0.2T$ iterations and then start to send Byzantine uploads thereafter. 

We can see that no matter when the Byzantine workers start to be Byzantine, they all have a negligible impact on the testing accuracy except for the case with extreme privacy requirements. We notice that there are some mild performance fluctuations when $\epsilon=0.125$ for Colorectal and USPS, again, due to the large variance of DP noises.

\begin{table}[!ht] 
\footnotesize
    \begin{subtable}[h]{\columnwidth}
        \centering
        \begin{tabular}{c||c|c|c|c|c|c|c|c}
        \toprule
        
        \multirow{2}{*}{$\gamma$} & \multicolumn{2}{|c|}{MNIST} & \multicolumn{2}{|c|}{COLOR.} & \multicolumn{2}{|c|}{FASHION} &  \multicolumn{2}{|c}{USPS} \\
        
        \cline{2-9}
       & \multicolumn{2}{|c|}{$\epsilon=$} & \multicolumn{2}{|c|}{$\epsilon=$} & \multicolumn{2}{|c|}{$\epsilon=$} &  \multicolumn{2}{|c}{$\epsilon=$}\\
 
       & $\frac{1}{8}$&$2$  & $\frac{1}{8}$&$2$ & $\frac{1}{8}$&$2$ & $\frac{1}{8}$&$2$  \\
        \midrule
        \textit{$20\%$} & .86&.95 & .48&.73 & .66&.78 & .64&.85 \\
        \textit{$35\%$} & .87&.96 & .47&.74 & .69&.79 & .63&.86 \\
        \textit{$50\%$ (exact)} & .88&.96 & .49&.74 & .69&.80 & .64&.87 \\
        \textit{$65\%$} & .85 & .96 & .45&.73 & .70&.79 & .56&.87 \\
        \textit{$80\%$} & .83 & .95 & .34&.74 & .69&.79 & .54&.85 \\
        
        \bottomrule
        \end{tabular}
    \end{subtable}
    
     \caption{
     $\gamma$ is treated as a prior belief in this experiment and we study the effect when there is a mismatch between such belief and the truth. We fix the setting that $50\%$ workers are honest and vary $\gamma$. For the case where the belief is exactly the truth ($\gamma=50\%$), we denote it as ``exact''.
     All results are obtained by taking the average of three runs with different random seeds.}
     
     \label{tab:ablation_on_gamma_iid_lf}
\vspace{-0.5cm}
\end{table}

\subsection{More Experimental Results}

\textbf{Convergence behavior:} The convergence curve is presented in Figure \ref{fig:dp_byz_iid_label_convergence}. As can be seen in Figure \ref{fig:dp_byz_iid_label_convergence}, the training converges in the first several epochs. The convergence behavior of our protocol aligns well with ``\textit{Reference Accuracy}'' even when we have $60\%$ Byzantine workers. Similar to previous results in our CLAIM 4, We observe a larger variance for Colorectal than that of the rest datasets. As expected, this is due to its significantly small dataset size and the nature of training with DP.

\textbf{Ablation study on $\gamma$:} Recall that previously we assumed the server knows that at least $\gamma n$ workers are honest, what if $\gamma$ is only a (prior) belief rather than the truth, and moreover, what if there is a mismatch between such belief and the truth? We further conduct an ablation study on $\gamma$ if it is only a belief. We can see from Table \ref{tab:ablation_on_gamma_iid_lf} that, in the case where $50\%$ workers are honest, as long as the server is conservative ($\gamma \leq 50\%$), we can still retain robustness. In contrast, we observe a notable utility drop for Colorectal and USPS under privacy level $\epsilon = .125$ when the server radically believes that $80\%$ workers are honest, this is because in our protocol, being radical ($\gamma$ is greater than the true honest portion) means the server tends to aggregate malicious uploads. Hence, the more radical, the worse the utility is expected to be. Based on such observation, the learned lesson is that we can always have robustness if we are not facing extreme privacy requirements and a conservative $\gamma$ is set.

\section{Conclusion}\label{sec:conclusion}
In this paper, with the aim to ensure both DP and Byzantine resilience for FL systems, we developed a learning protocol resulting from a co-design principle. We refactor the DP-SGD algorithm and tailor the Byzantine aggregation process towards each other to form an integrated protocol. For our DP-SGD variant, the small batch size property enables our first-stage Byzantine aggregation which trivially rejects many existing Byzantine attacks; the normalization technique enables our second-stage aggregation which provides a final sound filtering. As a cherry on top, normalizing also enables our efficient hyper-parameter tuning strategy which saves quadratic efforts. We also provide theoretical explanations behind the efficacy of our protocol. 

In the experiment part, we first provide evidence to support our contribution claim to both DP learning and Byzantine security tracks in separation, we then provide evidence of the effectiveness of our protocol tackling the two-fold issue, {\em i.e.}, an FL system needs to be privacy-preserving and Byzantine-resilient simultaneously. We have shown that our protocol does not incur ``side-effects'' to a system with no Byzantine attacker, and we have also seen that our protocol remains Byzantine-resilient even when there are up to $90\%$ distributive workers being Byzantine.

\bibliographystyle{ACM-Reference-Format}
\bibliography{reference}


\begin{thebibliography}{80}


\ifx \showCODEN    \undefined \def \showCODEN     #1{\unskip}     \fi
\ifx \showDOI      \undefined \def \showDOI       #1{#1}\fi
\ifx \showISBNx    \undefined \def \showISBNx     #1{\unskip}     \fi
\ifx \showISBNxiii \undefined \def \showISBNxiii  #1{\unskip}     \fi
\ifx \showISSN     \undefined \def \showISSN      #1{\unskip}     \fi
\ifx \showLCCN     \undefined \def \showLCCN      #1{\unskip}     \fi
\ifx \shownote     \undefined \def \shownote      #1{#1}          \fi
\ifx \showarticletitle \undefined \def \showarticletitle #1{#1}   \fi
\ifx \showURL      \undefined \def \showURL       {\relax}        \fi
\providecommand\bibfield[2]{#2}
\providecommand\bibinfo[2]{#2}
\providecommand\natexlab[1]{#1}
\providecommand\showeprint[2][]{arXiv:#2}

\bibitem[Abadi et~al\mbox{.}(2016)]%
        {abadi2016deep}
\bibfield{author}{\bibinfo{person}{Martin Abadi}, \bibinfo{person}{Andy Chu},
  \bibinfo{person}{Ian Goodfellow}, \bibinfo{person}{H~Brendan McMahan},
  \bibinfo{person}{Ilya Mironov}, \bibinfo{person}{Kunal Talwar}, {and}
  \bibinfo{person}{Li Zhang}.} \bibinfo{year}{2016}\natexlab{}.
\newblock \showarticletitle{Deep learning with differential privacy}. In
  \bibinfo{booktitle}{\emph{Proceedings of the 2016 ACM SIGSAC conference on
  computer and communications security}}. \bibinfo{pages}{308--318}.
\newblock


\bibitem[Agarwal et~al\mbox{.}(2021)]%
        {agarwal2021skellam}
\bibfield{author}{\bibinfo{person}{Naman Agarwal}, \bibinfo{person}{Peter
  Kairouz}, {and} \bibinfo{person}{Ziyu Liu}.} \bibinfo{year}{2021}\natexlab{}.
\newblock \showarticletitle{The skellam mechanism for differentially private
  federated learning}.
\newblock \bibinfo{journal}{\emph{Advances in Neural Information Processing
  Systems}}  \bibinfo{volume}{34} (\bibinfo{year}{2021}).
\newblock


\bibitem[Anil et~al\mbox{.}(2021)]%
        {anil2021large}
\bibfield{author}{\bibinfo{person}{Rohan Anil}, \bibinfo{person}{Badih Ghazi},
  \bibinfo{person}{Vineet Gupta}, \bibinfo{person}{Ravi Kumar}, {and}
  \bibinfo{person}{Pasin Manurangsi}.} \bibinfo{year}{2021}\natexlab{}.
\newblock \showarticletitle{Large-scale differentially private bert}.
\newblock \bibinfo{journal}{\emph{arXiv preprint arXiv:2108.01624}}
  (\bibinfo{year}{2021}).
\newblock


\bibitem[Asoodeh et~al\mbox{.}(2021)]%
        {asoodeh2021three}
\bibfield{author}{\bibinfo{person}{Shahab Asoodeh}, \bibinfo{person}{Jiachun
  Liao}, \bibinfo{person}{Flavio~P Calmon}, \bibinfo{person}{Oliver Kosut},
  {and} \bibinfo{person}{Lalitha Sankar}.} \bibinfo{year}{2021}\natexlab{}.
\newblock \showarticletitle{Three variants of differential privacy: Lossless
  conversion and applications}.
\newblock \bibinfo{journal}{\emph{IEEE Journal on Selected Areas in Information
  Theory}} \bibinfo{volume}{2}, \bibinfo{number}{1} (\bibinfo{year}{2021}),
  \bibinfo{pages}{208--222}.
\newblock


\bibitem[Bagdasaryan et~al\mbox{.}(2020)]%
        {data_p_bagdasaryan2020backdoor}
\bibfield{author}{\bibinfo{person}{Eugene Bagdasaryan},
  \bibinfo{person}{Andreas Veit}, \bibinfo{person}{Yiqing Hua},
  \bibinfo{person}{Deborah Estrin}, {and} \bibinfo{person}{Vitaly Shmatikov}.}
  \bibinfo{year}{2020}\natexlab{}.
\newblock \showarticletitle{How to backdoor federated learning}. In
  \bibinfo{booktitle}{\emph{International Conference on Artificial Intelligence
  and Statistics}}. PMLR, \bibinfo{pages}{2938--2948}.
\newblock


\bibitem[Baruch et~al\mbox{.}(2019)]%
        {baruch2019little}
\bibfield{author}{\bibinfo{person}{Gilad Baruch}, \bibinfo{person}{Moran
  Baruch}, {and} \bibinfo{person}{Yoav Goldberg}.}
  \bibinfo{year}{2019}\natexlab{}.
\newblock \showarticletitle{A little is enough: Circumventing defenses for
  distributed learning}.
\newblock \bibinfo{journal}{\emph{Advances in Neural Information Processing
  Systems}}  \bibinfo{volume}{32} (\bibinfo{year}{2019}).
\newblock


\bibitem[Bassily et~al\mbox{.}(2014)]%
        {bassily2014private}
\bibfield{author}{\bibinfo{person}{Raef Bassily}, \bibinfo{person}{Adam Smith},
  {and} \bibinfo{person}{Abhradeep Thakurta}.} \bibinfo{year}{2014}\natexlab{}.
\newblock \showarticletitle{Private empirical risk minimization: Efficient
  algorithms and tight error bounds}. In \bibinfo{booktitle}{\emph{2014 IEEE
  55th Annual Symposium on Foundations of Computer Science}}. IEEE,
  \bibinfo{pages}{464--473}.
\newblock


\bibitem[Bassily et~al\mbox{.}(2018)]%
        {bassily2018model}
\bibfield{author}{\bibinfo{person}{Raef Bassily}, \bibinfo{person}{Om Thakkar},
  {and} \bibinfo{person}{Abhradeep Guha~Thakurta}.}
  \bibinfo{year}{2018}\natexlab{}.
\newblock \showarticletitle{Model-agnostic private learning}.
\newblock \bibinfo{journal}{\emph{Advances in Neural Information Processing
  Systems}}  \bibinfo{volume}{31} (\bibinfo{year}{2018}).
\newblock


\bibitem[Blanchard et~al\mbox{.}(2017)]%
        {blanchard2017machine}
\bibfield{author}{\bibinfo{person}{Peva Blanchard}, \bibinfo{person}{El~Mahdi
  El~Mhamdi}, \bibinfo{person}{Rachid Guerraoui}, {and} \bibinfo{person}{Julien
  Stainer}.} \bibinfo{year}{2017}\natexlab{}.
\newblock \showarticletitle{Machine learning with adversaries: Byzantine
  tolerant gradient descent}. In \bibinfo{booktitle}{\emph{Proceedings of the
  31st International Conference on Neural Information Processing Systems}}.
  \bibinfo{pages}{118--128}.
\newblock


\bibitem[Bu et~al\mbox{.}(2022)]%
        {bu2022automatic}
\bibfield{author}{\bibinfo{person}{Zhiqi Bu}, \bibinfo{person}{Yu-Xiang Wang},
  \bibinfo{person}{Sheng Zha}, {and} \bibinfo{person}{George Karypis}.}
  \bibinfo{year}{2022}\natexlab{}.
\newblock \showarticletitle{Automatic Clipping: Differentially Private Deep
  Learning Made Easier and Stronger}.
\newblock \bibinfo{journal}{\emph{arXiv preprint arXiv:2206.07136}}
  (\bibinfo{year}{2022}).
\newblock


\bibitem[Cao et~al\mbox{.}(2020)]%
        {caofltrust}
\bibfield{author}{\bibinfo{person}{Xiaoyu Cao}, \bibinfo{person}{Minghong
  Fang}, \bibinfo{person}{Jia Liu}, {and} \bibinfo{person}{Neil~Zhenqiang
  Gong}.} \bibinfo{year}{2020}\natexlab{}.
\newblock \showarticletitle{Fltrust: Byzantine-robust federated learning via
  trust bootstrapping}.
\newblock \bibinfo{journal}{\emph{arXiv preprint arXiv:2012.13995}}
  (\bibinfo{year}{2020}).
\newblock


\bibitem[Chaudhuri et~al\mbox{.}(2011)]%
        {chaudhuri2011differentially}
\bibfield{author}{\bibinfo{person}{Kamalika Chaudhuri}, \bibinfo{person}{Claire
  Monteleoni}, {and} \bibinfo{person}{Anand~D Sarwate}.}
  \bibinfo{year}{2011}\natexlab{}.
\newblock \showarticletitle{Differentially private empirical risk
  minimization.}
\newblock \bibinfo{journal}{\emph{Journal of Machine Learning Research}}
  \bibinfo{volume}{12}, \bibinfo{number}{3} (\bibinfo{year}{2011}).
\newblock


\bibitem[Chen et~al\mbox{.}(2017a)]%
        {chen2017targeted}
\bibfield{author}{\bibinfo{person}{Xinyun Chen}, \bibinfo{person}{Chang Liu},
  \bibinfo{person}{Bo Li}, \bibinfo{person}{Kimberly Lu}, {and}
  \bibinfo{person}{Dawn Song}.} \bibinfo{year}{2017}\natexlab{a}.
\newblock \showarticletitle{Targeted backdoor attacks on deep learning systems
  using data poisoning}.
\newblock \bibinfo{journal}{\emph{arXiv preprint arXiv:1712.05526}}
  (\bibinfo{year}{2017}).
\newblock


\bibitem[Chen et~al\mbox{.}(2017b)]%
        {data_p_chen2017targeted}
\bibfield{author}{\bibinfo{person}{Xinyun Chen}, \bibinfo{person}{Chang Liu},
  \bibinfo{person}{Bo Li}, \bibinfo{person}{Kimberly Lu}, {and}
  \bibinfo{person}{Dawn Song}.} \bibinfo{year}{2017}\natexlab{b}.
\newblock \showarticletitle{Targeted backdoor attacks on deep learning systems
  using data poisoning}.
\newblock \bibinfo{journal}{\emph{arXiv preprint arXiv:1712.05526}}
  (\bibinfo{year}{2017}).
\newblock


\bibitem[Chen et~al\mbox{.}(2017c)]%
        {byz_chen2017distributed}
\bibfield{author}{\bibinfo{person}{Yudong Chen}, \bibinfo{person}{Lili Su},
  {and} \bibinfo{person}{Jiaming Xu}.} \bibinfo{year}{2017}\natexlab{c}.
\newblock \showarticletitle{Distributed statistical machine learning in
  adversarial settings: Byzantine gradient descent}.
\newblock \bibinfo{journal}{\emph{Proceedings of the ACM on Measurement and
  Analysis of Computing Systems}} \bibinfo{volume}{1}, \bibinfo{number}{2}
  (\bibinfo{year}{2017}), \bibinfo{pages}{1--25}.
\newblock


\bibitem[Choquette-Choo et~al\mbox{.}(2021)]%
        {choquette2021capc}
\bibfield{author}{\bibinfo{person}{Christopher~A Choquette-Choo},
  \bibinfo{person}{Natalie Dullerud}, \bibinfo{person}{Adam Dziedzic},
  \bibinfo{person}{Yunxiang Zhang}, \bibinfo{person}{Somesh Jha},
  \bibinfo{person}{Nicolas Papernot}, {and} \bibinfo{person}{Xiao Wang}.}
  \bibinfo{year}{2021}\natexlab{}.
\newblock \showarticletitle{Capc learning: Confidential and private
  collaborative learning}.
\newblock \bibinfo{journal}{\emph{arXiv preprint arXiv:2102.05188}}
  (\bibinfo{year}{2021}).
\newblock


\bibitem[Clanuwat et~al\mbox{.}(2018)]%
        {clanuwat2018deepkmnist}
\bibfield{author}{\bibinfo{person}{Tarin Clanuwat}, \bibinfo{person}{Mikel
  Bober-Irizar}, \bibinfo{person}{Asanobu Kitamoto}, \bibinfo{person}{Alex
  Lamb}, \bibinfo{person}{Kazuaki Yamamoto}, {and} \bibinfo{person}{David Ha}.}
  \bibinfo{year}{2018}\natexlab{}.
\newblock \bibinfo{booktitle}{\emph{Deep Learning for Classical Japanese
  Literature}}.
\newblock
\showeprint[arXiv]{cs.CV/1812.01718}~[cs.CV]


\bibitem[Cutkosky and Mehta(2020)]%
        {normalizedsgd}
\bibfield{author}{\bibinfo{person}{Ashok Cutkosky} {and} \bibinfo{person}{Harsh
  Mehta}.} \bibinfo{year}{2020}\natexlab{}.
\newblock \showarticletitle{Momentum improves normalized sgd}. In
  \bibinfo{booktitle}{\emph{International Conference on Machine Learning}}.
  PMLR, \bibinfo{pages}{2260--2268}.
\newblock


\bibitem[Das et~al\mbox{.}(2021)]%
        {das2021privacy}
\bibfield{author}{\bibinfo{person}{Rudrajit Das}, \bibinfo{person}{Abolfazl
  Hashemi}, \bibinfo{person}{Sujay Sanghavi}, {and} \bibinfo{person}{Inderjit~S
  Dhillon}.} \bibinfo{year}{2021}\natexlab{}.
\newblock \showarticletitle{Privacy-Preserving Federated Learning via
  Normalized (instead of Clipped) Updates}.
\newblock \bibinfo{journal}{\emph{arXiv preprint arXiv:2106.07094}}
  (\bibinfo{year}{2021}).
\newblock


\bibitem[De et~al\mbox{.}(2022)]%
        {de2022unlocking}
\bibfield{author}{\bibinfo{person}{Soham De}, \bibinfo{person}{Leonard
  Berrada}, \bibinfo{person}{Jamie Hayes}, \bibinfo{person}{Samuel~L Smith},
  {and} \bibinfo{person}{Borja Balle}.} \bibinfo{year}{2022}\natexlab{}.
\newblock \showarticletitle{Unlocking high-accuracy differentially private
  image classification through scale}.
\newblock \bibinfo{journal}{\emph{arXiv preprint arXiv:2204.13650}}
  (\bibinfo{year}{2022}).
\newblock


\bibitem[Dwork et~al\mbox{.}(2006)]%
        {dwork2006calibrating}
\bibfield{author}{\bibinfo{person}{Cynthia Dwork}, \bibinfo{person}{Frank
  McSherry}, \bibinfo{person}{Kobbi Nissim}, {and} \bibinfo{person}{Adam
  Smith}.} \bibinfo{year}{2006}\natexlab{}.
\newblock \showarticletitle{Calibrating noise to sensitivity in private data
  analysis}. In \bibinfo{booktitle}{\emph{Theory of cryptography conference}}.
  Springer, \bibinfo{pages}{265--284}.
\newblock


\bibitem[Fang et~al\mbox{.}(2020)]%
        {fang2020local}
\bibfield{author}{\bibinfo{person}{Minghong Fang}, \bibinfo{person}{Xiaoyu
  Cao}, \bibinfo{person}{Jinyuan Jia}, {and} \bibinfo{person}{Neil Gong}.}
  \bibinfo{year}{2020}\natexlab{}.
\newblock \showarticletitle{Local Model Poisoning Attacks to
  $\{$Byzantine-Robust$\}$ Federated Learning}. In
  \bibinfo{booktitle}{\emph{29th USENIX Security Symposium (USENIX Security
  20)}}. \bibinfo{pages}{1605--1622}.
\newblock


\bibitem[Fu et~al\mbox{.}(2021)]%
        {DBLP:conf/sigmod/FuSYJXT021}
\bibfield{author}{\bibinfo{person}{Fangcheng Fu}, \bibinfo{person}{Yingxia
  Shao}, \bibinfo{person}{Lele Yu}, \bibinfo{person}{Jiawei Jiang},
  \bibinfo{person}{Huanran Xue}, \bibinfo{person}{Yangyu Tao}, {and}
  \bibinfo{person}{Bin Cui}.} \bibinfo{year}{2021}\natexlab{}.
\newblock \showarticletitle{VF\({}^{\mbox{2}}\)Boost: Very Fast Vertical
  Federated Gradient Boosting for Cross-Enterprise Learning}. In
  \bibinfo{booktitle}{\emph{{SIGMOD} '21: International Conference on
  Management of Data, Virtual Event, China, June 20-25, 2021}},
  \bibfield{editor}{\bibinfo{person}{Guoliang Li}, \bibinfo{person}{Zhanhuai
  Li}, \bibinfo{person}{Stratos Idreos}, {and} \bibinfo{person}{Divesh
  Srivastava}} (Eds.). \bibinfo{publisher}{{ACM}}, \bibinfo{pages}{563--576}.
\newblock
\urldef\tempurl%
\url{https://doi.org/10.1145/3448016.3457241}
\showDOI{\tempurl}


\bibitem[Fu et~al\mbox{.}(2022)]%
        {DBLP:conf/sigmod/FuXCT022}
\bibfield{author}{\bibinfo{person}{Fangcheng Fu}, \bibinfo{person}{Huanran
  Xue}, \bibinfo{person}{Yong Cheng}, \bibinfo{person}{Yangyu Tao}, {and}
  \bibinfo{person}{Bin Cui}.} \bibinfo{year}{2022}\natexlab{}.
\newblock \showarticletitle{BlindFL: Vertical Federated Machine Learning
  without Peeking into Your Data}. In \bibinfo{booktitle}{\emph{{SIGMOD} '22:
  International Conference on Management of Data, Philadelphia, PA, USA, June
  12 - 17, 2022}}, \bibfield{editor}{\bibinfo{person}{Zachary Ives},
  \bibinfo{person}{Angela Bonifati}, {and} \bibinfo{person}{Amr~El Abbadi}}
  (Eds.). \bibinfo{publisher}{{ACM}}, \bibinfo{pages}{1316--1330}.
\newblock
\urldef\tempurl%
\url{https://doi.org/10.1145/3514221.3526127}
\showDOI{\tempurl}


\bibitem[gboard({[n.\,d.]})]%
        {gboard}
gboard \bibinfo{year}{[n.\,d.]}\natexlab{}.
\newblock \bibinfo{title}{Federated Learning: Collaborative Machine Learning
  without Centralized Training Data}.
\newblock
  \bibinfo{howpublished}{\url{https://ai.googleblog.com/2017/04/federated-learning-collaborative.html}}.
\newblock


\bibitem[gdpr({[n.\,d.]})]%
        {gdpr}
gdpr \bibinfo{year}{[n.\,d.]}\natexlab{}.
\newblock \bibinfo{title}{General Data Protection Regulation (GDPR)}.
\newblock \bibinfo{howpublished}{\url{https://gdpr.eu/what-is-gdpr/}}.
\newblock


\bibitem[Geyer et~al\mbox{.}(2017)]%
        {geyer2017differentially}
\bibfield{author}{\bibinfo{person}{Robin~C Geyer}, \bibinfo{person}{Tassilo
  Klein}, {and} \bibinfo{person}{Moin Nabi}.} \bibinfo{year}{2017}\natexlab{}.
\newblock \showarticletitle{Differentially private federated learning: A client
  level perspective}.
\newblock \bibinfo{journal}{\emph{arXiv preprint arXiv:1712.07557}}
  (\bibinfo{year}{2017}).
\newblock


\bibitem[Golatkar et~al\mbox{.}(2022)]%
        {golatkar2022mixed}
\bibfield{author}{\bibinfo{person}{Aditya Golatkar},
  \bibinfo{person}{Alessandro Achille}, \bibinfo{person}{Yu-Xiang Wang},
  \bibinfo{person}{Aaron Roth}, \bibinfo{person}{Michael Kearns}, {and}
  \bibinfo{person}{Stefano Soatto}.} \bibinfo{year}{2022}\natexlab{}.
\newblock \showarticletitle{Mixed Differential Privacy in Computer Vision}.
\newblock \bibinfo{journal}{\emph{arXiv preprint arXiv:2203.11481}}
  (\bibinfo{year}{2022}).
\newblock


\bibitem[Gopi et~al\mbox{.}(2021)]%
        {gopi2021numerical}
\bibfield{author}{\bibinfo{person}{Sivakanth Gopi}, \bibinfo{person}{Yin~Tat
  Lee}, {and} \bibinfo{person}{Lukas Wutschitz}.}
  \bibinfo{year}{2021}\natexlab{}.
\newblock \showarticletitle{Numerical Composition of Differential Privacy}.
\newblock \bibinfo{journal}{\emph{arXiv preprint arXiv:2106.02848}}
  (\bibinfo{year}{2021}).
\newblock


\bibitem[Guerraoui et~al\mbox{.}(2021a)]%
        {guerraoui2021combining}
\bibfield{author}{\bibinfo{person}{Rachid Guerraoui}, \bibinfo{person}{Nirupam
  Gupta}, \bibinfo{person}{Rafael Pinot}, \bibinfo{person}{Sebastien Rouault},
  {and} \bibinfo{person}{John Stephan}.} \bibinfo{year}{2021}\natexlab{a}.
\newblock \showarticletitle{Combining Differential Privacy and Byzantine
  Resilience in Distributed SGD}.
\newblock \bibinfo{journal}{\emph{arXiv preprint arXiv:2110.03991}}
  (\bibinfo{year}{2021}).
\newblock


\bibitem[Guerraoui et~al\mbox{.}(2021b)]%
        {dp_br_add_up}
\bibfield{author}{\bibinfo{person}{Rachid Guerraoui}, \bibinfo{person}{Nirupam
  Gupta}, \bibinfo{person}{Rafael Pinot}, \bibinfo{person}{Sebastien Rouault},
  {and} \bibinfo{person}{John Stephan}.} \bibinfo{year}{2021}\natexlab{b}.
\newblock \showarticletitle{Differential Privacy and Byzantine Resilience in
  SGD: Do They Add Up?}. In \bibinfo{booktitle}{\emph{Proceedings of the 2021
  ACM Symposium on Principles of Distributed Computing}}.
  \bibinfo{pages}{391--401}.
\newblock


\bibitem[Guerraoui et~al\mbox{.}(2018)]%
        {guerraoui2018hidden}
\bibfield{author}{\bibinfo{person}{Rachid Guerraoui},
  \bibinfo{person}{S{\'e}bastien Rouault}, {et~al\mbox{.}}}
  \bibinfo{year}{2018}\natexlab{}.
\newblock \showarticletitle{The hidden vulnerability of distributed learning in
  byzantium}. In \bibinfo{booktitle}{\emph{International Conference on Machine
  Learning}}. PMLR, \bibinfo{pages}{3521--3530}.
\newblock


\bibitem[Hamm et~al\mbox{.}(2016)]%
        {hamm2016learning}
\bibfield{author}{\bibinfo{person}{Jihun Hamm}, \bibinfo{person}{Yingjun Cao},
  {and} \bibinfo{person}{Mikhail Belkin}.} \bibinfo{year}{2016}\natexlab{}.
\newblock \showarticletitle{Learning privately from multiparty data}. In
  \bibinfo{booktitle}{\emph{International Conference on Machine Learning}}.
  PMLR, \bibinfo{pages}{555--563}.
\newblock


\bibitem[Hull(1994)]%
        {usps}
\bibfield{author}{\bibinfo{person}{Jonathan~J. Hull}.}
  \bibinfo{year}{1994}\natexlab{}.
\newblock \showarticletitle{A database for handwritten text recognition
  research}.
\newblock \bibinfo{journal}{\emph{IEEE Transactions on pattern analysis and
  machine intelligence}} \bibinfo{volume}{16}, \bibinfo{number}{5}
  (\bibinfo{year}{1994}), \bibinfo{pages}{550--554}.
\newblock


\bibitem[Iyengar et~al\mbox{.}(2019)]%
        {iyengar2019towards}
\bibfield{author}{\bibinfo{person}{Roger Iyengar}, \bibinfo{person}{Joseph~P
  Near}, \bibinfo{person}{Dawn Song}, \bibinfo{person}{Om Thakkar},
  \bibinfo{person}{Abhradeep Thakurta}, {and} \bibinfo{person}{Lun Wang}.}
  \bibinfo{year}{2019}\natexlab{}.
\newblock \showarticletitle{Towards practical differentially private convex
  optimization}. In \bibinfo{booktitle}{\emph{2019 IEEE Symposium on Security
  and Privacy (SP)}}. IEEE, \bibinfo{pages}{299--316}.
\newblock


\bibitem[Jagielski et~al\mbox{.}(2018)]%
        {utbyz_jagielski2018manipulating}
\bibfield{author}{\bibinfo{person}{Matthew Jagielski}, \bibinfo{person}{Alina
  Oprea}, \bibinfo{person}{Battista Biggio}, \bibinfo{person}{Chang Liu},
  \bibinfo{person}{Cristina Nita-Rotaru}, {and} \bibinfo{person}{Bo Li}.}
  \bibinfo{year}{2018}\natexlab{}.
\newblock \showarticletitle{Manipulating machine learning: Poisoning attacks
  and countermeasures for regression learning}. In
  \bibinfo{booktitle}{\emph{2018 IEEE Symposium on Security and Privacy (SP)}}.
  IEEE, \bibinfo{pages}{19--35}.
\newblock


\bibitem[Kather et~al\mbox{.}(2016)]%
        {colorectal}
\bibfield{author}{\bibinfo{person}{Jakob~Nikolas Kather},
  \bibinfo{person}{Cleo-Aron Weis}, \bibinfo{person}{Francesco Bianconi},
  \bibinfo{person}{Susanne~M Melchers}, \bibinfo{person}{Lothar~R Schad},
  \bibinfo{person}{Timo Gaiser}, \bibinfo{person}{Alexander Marx}, {and}
  \bibinfo{person}{Frank~Gerrit Z{\"o}llner}.} \bibinfo{year}{2016}\natexlab{}.
\newblock \showarticletitle{Multi-class texture analysis in colorectal cancer
  histology}.
\newblock \bibinfo{journal}{\emph{Scientific reports}} \bibinfo{volume}{6},
  \bibinfo{number}{1} (\bibinfo{year}{2016}), \bibinfo{pages}{1--11}.
\newblock


\bibitem[Kolmogorov(1933)]%
        {kolmogorov1933sulla}
\bibfield{author}{\bibinfo{person}{Andrey Kolmogorov}.}
  \bibinfo{year}{1933}\natexlab{}.
\newblock \showarticletitle{Sulla determinazione empirica di una lgge di
  distribuzione}.
\newblock \bibinfo{journal}{\emph{Inst. Ital. Attuari, Giorn.}}
  \bibinfo{volume}{4} (\bibinfo{year}{1933}), \bibinfo{pages}{83--91}.
\newblock


\bibitem[Kone{\v{c}}n{\`y} et~al\mbox{.}(2016)]%
        {konevcny2016federated}
\bibfield{author}{\bibinfo{person}{Jakub Kone{\v{c}}n{\`y}},
  \bibinfo{person}{H~Brendan McMahan}, \bibinfo{person}{Felix~X Yu},
  \bibinfo{person}{Peter Richt{\'a}rik}, \bibinfo{person}{Ananda~Theertha
  Suresh}, {and} \bibinfo{person}{Dave Bacon}.}
  \bibinfo{year}{2016}\natexlab{}.
\newblock \showarticletitle{Federated learning: Strategies for improving
  communication efficiency}.
\newblock \bibinfo{journal}{\emph{arXiv preprint arXiv:1610.05492}}
  (\bibinfo{year}{2016}).
\newblock


\bibitem[Kurakin et~al\mbox{.}(2022)]%
        {kurakin2022toward}
\bibfield{author}{\bibinfo{person}{Alexey Kurakin}, \bibinfo{person}{Steve
  Chien}, \bibinfo{person}{Shuang Song}, \bibinfo{person}{Roxana Geambasu},
  \bibinfo{person}{Andreas Terzis}, {and} \bibinfo{person}{Abhradeep
  Thakurta}.} \bibinfo{year}{2022}\natexlab{}.
\newblock \showarticletitle{Toward training at imagenet scale with differential
  privacy}.
\newblock \bibinfo{journal}{\emph{arXiv preprint arXiv:2201.12328}}
  (\bibinfo{year}{2022}).
\newblock


\bibitem[LeCun et~al\mbox{.}(1998)]%
        {mnist}
\bibfield{author}{\bibinfo{person}{Yann LeCun}, \bibinfo{person}{L{\'e}on
  Bottou}, \bibinfo{person}{Yoshua Bengio}, {and} \bibinfo{person}{Patrick
  Haffner}.} \bibinfo{year}{1998}\natexlab{}.
\newblock \showarticletitle{Gradient-based learning applied to document
  recognition}.
\newblock \bibinfo{journal}{\emph{Proc. IEEE}} \bibinfo{volume}{86},
  \bibinfo{number}{11} (\bibinfo{year}{1998}), \bibinfo{pages}{2278--2324}.
\newblock


\bibitem[Liu et~al\mbox{.}(2018)]%
        {DBLP:conf/ndss/LiuMALZW018}
\bibfield{author}{\bibinfo{person}{Yingqi Liu}, \bibinfo{person}{Shiqing Ma},
  \bibinfo{person}{Yousra Aafer}, \bibinfo{person}{Wen{-}Chuan Lee},
  \bibinfo{person}{Juan Zhai}, \bibinfo{person}{Weihang Wang}, {and}
  \bibinfo{person}{Xiangyu Zhang}.} \bibinfo{year}{2018}\natexlab{}.
\newblock \showarticletitle{Trojaning Attack on Neural Networks}. In
  \bibinfo{booktitle}{\emph{25th Annual Network and Distributed System Security
  Symposium, {NDSS} 2018, San Diego, California, USA, February 18-21, 2018}}.
  \bibinfo{publisher}{The Internet Society}.
\newblock
\urldef\tempurl%
\url{http://wp.internetsociety.org/ndss/wp-content/uploads/sites/25/2018/02/ndss2018\_03A-5\_Liu\_paper.pdf}
\showURL{%
\tempurl}


\bibitem[Ma et~al\mbox{.}(2022)]%
        {ma2022differentially}
\bibfield{author}{\bibinfo{person}{Xu Ma}, \bibinfo{person}{Xiaoqian Sun},
  \bibinfo{person}{Yuduo Wu}, \bibinfo{person}{Zheli Liu},
  \bibinfo{person}{Xiaofeng Chen}, {and} \bibinfo{person}{Changyu Dong}.}
  \bibinfo{year}{2022}\natexlab{}.
\newblock \showarticletitle{Differentially Private Byzantine-robust Federated
  Learning}.
\newblock \bibinfo{journal}{\emph{IEEE Transactions on Parallel and Distributed
  Systems}} (\bibinfo{year}{2022}).
\newblock


\bibitem[Marsaglia et~al\mbox{.}(2003)]%
        {marsaglia2003evaluating}
\bibfield{author}{\bibinfo{person}{George Marsaglia}, \bibinfo{person}{Wai~Wan
  Tsang}, {and} \bibinfo{person}{Jingbo Wang}.}
  \bibinfo{year}{2003}\natexlab{}.
\newblock \showarticletitle{Evaluating Kolmogorov's distribution}.
\newblock \bibinfo{journal}{\emph{Journal of statistical software}}
  \bibinfo{volume}{8} (\bibinfo{year}{2003}), \bibinfo{pages}{1--4}.
\newblock


\bibitem[McMahan et~al\mbox{.}(2017)]%
        {fedavg}
\bibfield{author}{\bibinfo{person}{Brendan McMahan}, \bibinfo{person}{Eider
  Moore}, \bibinfo{person}{Daniel Ramage}, \bibinfo{person}{Seth Hampson},
  {and} \bibinfo{person}{Blaise~Aguera y Arcas}.}
  \bibinfo{year}{2017}\natexlab{}.
\newblock \showarticletitle{Communication-efficient learning of deep networks
  from decentralized data}. In \bibinfo{booktitle}{\emph{Artificial
  intelligence and statistics}}. PMLR, \bibinfo{pages}{1273--1282}.
\newblock


\bibitem[Mironov et~al\mbox{.}(2019)]%
        {mironov2019r}
\bibfield{author}{\bibinfo{person}{Ilya Mironov}, \bibinfo{person}{Kunal
  Talwar}, {and} \bibinfo{person}{Li Zhang}.} \bibinfo{year}{2019}\natexlab{}.
\newblock \showarticletitle{R$\backslash$'enyi Differential Privacy of the
  Sampled Gaussian Mechanism}.
\newblock \bibinfo{journal}{\emph{arXiv preprint arXiv:1908.10530}}
  (\bibinfo{year}{2019}).
\newblock


\bibitem[Papernot et~al\mbox{.}(2016)]%
        {papernot2016semi}
\bibfield{author}{\bibinfo{person}{Nicolas Papernot},
  \bibinfo{person}{Mart{\'\i}n Abadi}, \bibinfo{person}{Ulfar Erlingsson},
  \bibinfo{person}{Ian Goodfellow}, {and} \bibinfo{person}{Kunal Talwar}.}
  \bibinfo{year}{2016}\natexlab{}.
\newblock \showarticletitle{Semi-supervised knowledge transfer for deep
  learning from private training data}.
\newblock \bibinfo{journal}{\emph{arXiv preprint arXiv:1610.05755}}
  (\bibinfo{year}{2016}).
\newblock


\bibitem[Papernot et~al\mbox{.}(2018)]%
        {papernot2018scalable}
\bibfield{author}{\bibinfo{person}{Nicolas Papernot}, \bibinfo{person}{Shuang
  Song}, \bibinfo{person}{Ilya Mironov}, \bibinfo{person}{Ananth Raghunathan},
  \bibinfo{person}{Kunal Talwar}, {and} \bibinfo{person}{{\'U}lfar
  Erlingsson}.} \bibinfo{year}{2018}\natexlab{}.
\newblock \showarticletitle{Scalable private learning with pate}.
\newblock \bibinfo{journal}{\emph{arXiv preprint arXiv:1802.08908}}
  (\bibinfo{year}{2018}).
\newblock


\bibitem[Paulik et~al\mbox{.}(2021)]%
        {apple_fed}
\bibfield{author}{\bibinfo{person}{Matthias Paulik}, \bibinfo{person}{Matt
  Seigel}, \bibinfo{person}{Henry Mason}, \bibinfo{person}{Dominic Telaar},
  \bibinfo{person}{Joris Kluivers}, \bibinfo{person}{Rogier van Dalen},
  \bibinfo{person}{Chi~Wai Lau}, \bibinfo{person}{Luke Carlson},
  \bibinfo{person}{Filip Granqvist}, \bibinfo{person}{Chris Vandevelde},
  {et~al\mbox{.}}} \bibinfo{year}{2021}\natexlab{}.
\newblock \showarticletitle{Federated evaluation and tuning for on-device
  personalization: System design \& applications}.
\newblock \bibinfo{journal}{\emph{arXiv preprint arXiv:2102.08503}}
  (\bibinfo{year}{2021}).
\newblock


\bibitem[Pillutla et~al\mbox{.}(2019)]%
        {pillutla2019robust}
\bibfield{author}{\bibinfo{person}{Krishna Pillutla}, \bibinfo{person}{Sham~M
  Kakade}, {and} \bibinfo{person}{Zaid Harchaoui}.}
  \bibinfo{year}{2019}\natexlab{}.
\newblock \showarticletitle{Robust aggregation for federated learning}.
\newblock \bibinfo{journal}{\emph{arXiv preprint arXiv:1912.13445}}
  (\bibinfo{year}{2019}).
\newblock


\bibitem[Rahman et~al\mbox{.}(2018)]%
        {rahman2018membership}
\bibfield{author}{\bibinfo{person}{Md~Atiqur Rahman}, \bibinfo{person}{Tanzila
  Rahman}, \bibinfo{person}{Robert Lagani{\`e}re}, \bibinfo{person}{Noman
  Mohammed}, {and} \bibinfo{person}{Yang Wang}.}
  \bibinfo{year}{2018}\natexlab{}.
\newblock \showarticletitle{Membership Inference Attack against Differentially
  Private Deep Learning Model.}
\newblock \bibinfo{journal}{\emph{Trans. Data Priv.}} \bibinfo{volume}{11},
  \bibinfo{number}{1} (\bibinfo{year}{2018}), \bibinfo{pages}{61--79}.
\newblock


\bibitem[Regatti et~al\mbox{.}(2020)]%
        {regatti2020bygars}
\bibfield{author}{\bibinfo{person}{Jayanth Regatti}, \bibinfo{person}{Hao
  Chen}, {and} \bibinfo{person}{Abhishek Gupta}.}
  \bibinfo{year}{2020}\natexlab{}.
\newblock \showarticletitle{Bygars: Byzantine sgd with arbitrary number of
  attackers}.
\newblock \bibinfo{journal}{\emph{arXiv preprint arXiv:2006.13421}}
  (\bibinfo{year}{2020}).
\newblock


\bibitem[Ren et~al\mbox{.}(2022)]%
        {DBLP:conf/sigmod/RenSYYZX22}
\bibfield{author}{\bibinfo{person}{Xuebin Ren}, \bibinfo{person}{Liang Shi},
  \bibinfo{person}{Weiren Yu}, \bibinfo{person}{Shusen Yang},
  \bibinfo{person}{Cong Zhao}, {and} \bibinfo{person}{Zongben Xu}.}
  \bibinfo{year}{2022}\natexlab{}.
\newblock \showarticletitle{{LDP-IDS:} Local Differential Privacy for Infinite
  Data Streams}. In \bibinfo{booktitle}{\emph{{SIGMOD} '22: International
  Conference on Management of Data, Philadelphia, PA, USA, June 12 - 17,
  2022}}, \bibfield{editor}{\bibinfo{person}{Zachary Ives},
  \bibinfo{person}{Angela Bonifati}, {and} \bibinfo{person}{Amr~El Abbadi}}
  (Eds.). \bibinfo{publisher}{{ACM}}, \bibinfo{pages}{1064--1077}.
\newblock
\urldef\tempurl%
\url{https://doi.org/10.1145/3514221.3526190}
\showDOI{\tempurl}


\bibitem[Rubinstein et~al\mbox{.}(2009)]%
        {utbyz_rubinstein2009antidote}
\bibfield{author}{\bibinfo{person}{Benjamin~IP Rubinstein},
  \bibinfo{person}{Blaine Nelson}, \bibinfo{person}{Ling Huang},
  \bibinfo{person}{Anthony~D Joseph}, \bibinfo{person}{Shing-hon Lau},
  \bibinfo{person}{Satish Rao}, \bibinfo{person}{Nina Taft}, {and}
  \bibinfo{person}{J~Doug Tygar}.} \bibinfo{year}{2009}\natexlab{}.
\newblock \showarticletitle{Antidote: understanding and defending against
  poisoning of anomaly detectors}. In \bibinfo{booktitle}{\emph{Proceedings of
  the 9th ACM SIGCOMM Conference on Internet Measurement}}.
  \bibinfo{pages}{1--14}.
\newblock


\bibitem[Shejwalkar and Houmansadr(2021)]%
        {shejwalkar2021manipulating}
\bibfield{author}{\bibinfo{person}{Virat Shejwalkar} {and}
  \bibinfo{person}{Amir Houmansadr}.} \bibinfo{year}{2021}\natexlab{}.
\newblock \showarticletitle{Manipulating the byzantine: Optimizing model
  poisoning attacks and defenses for federated learning}. In
  \bibinfo{booktitle}{\emph{NDSS}}.
\newblock


\bibitem[Shokri et~al\mbox{.}(2017)]%
        {MIA_shokri2017membership}
\bibfield{author}{\bibinfo{person}{Reza Shokri}, \bibinfo{person}{Marco
  Stronati}, \bibinfo{person}{Congzheng Song}, {and} \bibinfo{person}{Vitaly
  Shmatikov}.} \bibinfo{year}{2017}\natexlab{}.
\newblock \showarticletitle{Membership inference attacks against machine
  learning models}. In \bibinfo{booktitle}{\emph{2017 IEEE Symposium on
  Security and Privacy (SP)}}. IEEE, \bibinfo{pages}{3--18}.
\newblock


\bibitem[Smith et~al\mbox{.}(2017)]%
        {smith2017interaction}
\bibfield{author}{\bibinfo{person}{Adam Smith}, \bibinfo{person}{Abhradeep
  Thakurta}, {and} \bibinfo{person}{Jalaj Upadhyay}.}
  \bibinfo{year}{2017}\natexlab{}.
\newblock \showarticletitle{Is interaction necessary for distributed private
  learning?}. In \bibinfo{booktitle}{\emph{2017 IEEE Symposium on Security and
  Privacy (SP)}}. IEEE, \bibinfo{pages}{58--77}.
\newblock


\bibitem[Song et~al\mbox{.}(2013)]%
        {song2013stochastic}
\bibfield{author}{\bibinfo{person}{Shuang Song}, \bibinfo{person}{Kamalika
  Chaudhuri}, {and} \bibinfo{person}{Anand~D Sarwate}.}
  \bibinfo{year}{2013}\natexlab{}.
\newblock \showarticletitle{Stochastic gradient descent with differentially
  private updates}. In \bibinfo{booktitle}{\emph{2013 IEEE Global Conference on
  Signal and Information Processing}}. IEEE, \bibinfo{pages}{245--248}.
\newblock


\bibitem[tfp({[n.\,d.]})]%
        {tfp}
tfp \bibinfo{year}{[n.\,d.]}\natexlab{}.
\newblock \bibinfo{title}{{TensorFlow Privacy} Git repository}.
\newblock \bibinfo{howpublished}{\url{https://github.com/tensorflow/privacy}}.
\newblock


\bibitem[Truex et~al\mbox{.}(2020)]%
        {truex2020ldp}
\bibfield{author}{\bibinfo{person}{Stacey Truex}, \bibinfo{person}{Ling Liu},
  \bibinfo{person}{Ka-Ho Chow}, \bibinfo{person}{Mehmet~Emre Gursoy}, {and}
  \bibinfo{person}{Wenqi Wei}.} \bibinfo{year}{2020}\natexlab{}.
\newblock \showarticletitle{LDP-Fed: Federated learning with local differential
  privacy}. In \bibinfo{booktitle}{\emph{Proceedings of the Third ACM
  International Workshop on Edge Systems, Analytics and Networking}}.
  \bibinfo{pages}{61--66}.
\newblock


\bibitem[Wang et~al\mbox{.}(2022)]%
        {DBLP:conf/sigmod/WangBNM22}
\bibfield{author}{\bibinfo{person}{Chenghong Wang}, \bibinfo{person}{Johes
  Bater}, \bibinfo{person}{Kartik Nayak}, {and} \bibinfo{person}{Ashwin
  Machanavajjhala}.} \bibinfo{year}{2022}\natexlab{}.
\newblock \showarticletitle{IncShrink: Architecting Efficient Outsourced
  Databases using Incremental {MPC} and Differential Privacy}. In
  \bibinfo{booktitle}{\emph{{SIGMOD} '22: International Conference on
  Management of Data, Philadelphia, PA, USA, June 12 - 17, 2022}},
  \bibfield{editor}{\bibinfo{person}{Zachary Ives}, \bibinfo{person}{Angela
  Bonifati}, {and} \bibinfo{person}{Amr~El Abbadi}} (Eds.).
  \bibinfo{publisher}{{ACM}}, \bibinfo{pages}{818--832}.
\newblock
\urldef\tempurl%
\url{https://doi.org/10.1145/3514221.3526151}
\showDOI{\tempurl}


\bibitem[Wang et~al\mbox{.}(2019b)]%
        {wang2019differentially}
\bibfield{author}{\bibinfo{person}{Di Wang}, \bibinfo{person}{Changyou Chen},
  {and} \bibinfo{person}{Jinhui Xu}.} \bibinfo{year}{2019}\natexlab{b}.
\newblock \showarticletitle{Differentially private empirical risk minimization
  with non-convex loss functions}. In \bibinfo{booktitle}{\emph{International
  Conference on Machine Learning}}. PMLR, \bibinfo{pages}{6526--6535}.
\newblock


\bibitem[Wang et~al\mbox{.}(2020)]%
        {wang2020empirical}
\bibfield{author}{\bibinfo{person}{Di Wang}, \bibinfo{person}{Marco Gaboardi},
  \bibinfo{person}{Adam Smith}, {and} \bibinfo{person}{Jinhui Xu}.}
  \bibinfo{year}{2020}\natexlab{}.
\newblock \showarticletitle{Empirical risk minimization in the non-interactive
  local model of differential privacy}.
\newblock \bibinfo{journal}{\emph{Journal of machine learning research}}
  \bibinfo{volume}{21}, \bibinfo{number}{200} (\bibinfo{year}{2020}).
\newblock


\bibitem[Wang et~al\mbox{.}(2019c)]%
        {wang2019estimating}
\bibfield{author}{\bibinfo{person}{Di Wang}, \bibinfo{person}{Lijie Hu},
  \bibinfo{person}{Huanyu Zhang}, \bibinfo{person}{Marco Gaboardi}, {and}
  \bibinfo{person}{Jinhui Xu}.} \bibinfo{year}{2019}\natexlab{c}.
\newblock \showarticletitle{Estimating smooth GLM in non-interactive local
  differential privacy model with public unlabeled data}.
\newblock \bibinfo{journal}{\emph{arXiv preprint arXiv:1910.00482}}
  (\bibinfo{year}{2019}).
\newblock


\bibitem[Wang et~al\mbox{.}(2017)]%
        {wang2017differentially}
\bibfield{author}{\bibinfo{person}{Di Wang}, \bibinfo{person}{Minwei Ye}, {and}
  \bibinfo{person}{Jinhui Xu}.} \bibinfo{year}{2017}\natexlab{}.
\newblock \showarticletitle{Differentially private empirical risk minimization
  revisited: Faster and more general}.
\newblock \bibinfo{journal}{\emph{Advances in Neural Information Processing
  Systems}}  \bibinfo{volume}{30} (\bibinfo{year}{2017}).
\newblock


\bibitem[Wang et~al\mbox{.}(2019a)]%
        {wang2019subsampled}
\bibfield{author}{\bibinfo{person}{Yu-Xiang Wang}, \bibinfo{person}{Borja
  Balle}, {and} \bibinfo{person}{Shiva~Prasad Kasiviswanathan}.}
  \bibinfo{year}{2019}\natexlab{a}.
\newblock \showarticletitle{Subsampled r{\'e}nyi differential privacy and
  analytical moments accountant}. In \bibinfo{booktitle}{\emph{The 22nd
  International Conference on Artificial Intelligence and Statistics}}. PMLR,
  \bibinfo{pages}{1226--1235}.
\newblock


\bibitem[Wei et~al\mbox{.}(2020)]%
        {wei2020federated}
\bibfield{author}{\bibinfo{person}{Kang Wei}, \bibinfo{person}{Jun Li},
  \bibinfo{person}{Ming Ding}, \bibinfo{person}{Chuan Ma},
  \bibinfo{person}{Howard~H Yang}, \bibinfo{person}{Farhad Farokhi},
  \bibinfo{person}{Shi Jin}, \bibinfo{person}{Tony~QS Quek}, {and}
  \bibinfo{person}{H~Vincent Poor}.} \bibinfo{year}{2020}\natexlab{}.
\newblock \showarticletitle{Federated learning with differential privacy:
  Algorithms and performance analysis}.
\newblock \bibinfo{journal}{\emph{IEEE Transactions on Information Forensics
  and Security}}  \bibinfo{volume}{15} (\bibinfo{year}{2020}),
  \bibinfo{pages}{3454--3469}.
\newblock


\bibitem[{Wikipedia contributors}(2022)]%
        {enwiki:1097113055}
\bibfield{author}{\bibinfo{person}{{Wikipedia contributors}}.}
  \bibinfo{year}{2022}\natexlab{}.
\newblock \bibinfo{title}{68–95–99.7 rule --- {Wikipedia}{,} The Free
  Encyclopedia}.
\newblock
\newblock
\urldef\tempurl%
\url{https://en.wikipedia.org/w/index.php?title=68%E2%80%9395%E2%80%9399.7_rule&oldid=1097113055}
\showURL{%
\tempurl}
\newblock
\shownote{[Online; accessed 24-July-2022]}.


\bibitem[Xiao et~al\mbox{.}(2017)]%
        {fashion}
\bibfield{author}{\bibinfo{person}{Han Xiao}, \bibinfo{person}{Kashif Rasul},
  {and} \bibinfo{person}{Roland Vollgraf}.} \bibinfo{year}{2017}\natexlab{}.
\newblock \showarticletitle{Fashion-mnist: a novel image dataset for
  benchmarking machine learning algorithms}.
\newblock \bibinfo{journal}{\emph{arXiv preprint arXiv:1708.07747}}
  (\bibinfo{year}{2017}).
\newblock


\bibitem[Xie et~al\mbox{.}(2020)]%
        {byz_xie2020fall}
\bibfield{author}{\bibinfo{person}{Cong Xie}, \bibinfo{person}{Oluwasanmi
  Koyejo}, {and} \bibinfo{person}{Indranil Gupta}.}
  \bibinfo{year}{2020}\natexlab{}.
\newblock \showarticletitle{Fall of empires: Breaking Byzantine-tolerant SGD by
  inner product manipulation}. In \bibinfo{booktitle}{\emph{Uncertainty in
  Artificial Intelligence}}. PMLR, \bibinfo{pages}{261--270}.
\newblock


\bibitem[Xie et~al\mbox{.}(2022)]%
        {xie2022federatedscope}
\bibfield{author}{\bibinfo{person}{Yuexiang Xie}, \bibinfo{person}{Zhen Wang},
  \bibinfo{person}{Daoyuan Chen}, \bibinfo{person}{Dawei Gao},
  \bibinfo{person}{Liuyi Yao}, \bibinfo{person}{Weirui Kuang},
  \bibinfo{person}{Yaliang Li}, \bibinfo{person}{Bolin Ding}, {and}
  \bibinfo{person}{Jingren Zhou}.} \bibinfo{year}{2022}\natexlab{}.
\newblock \showarticletitle{FederatedScope: A Comprehensive and Flexible
  Federated Learning Platform via Message Passing}.
\newblock \bibinfo{journal}{\emph{arXiv preprint arXiv:2204.05011}}
  (\bibinfo{year}{2022}).
\newblock


\bibitem[Yang et~al\mbox{.}(2022)]%
        {yang2022normalized}
\bibfield{author}{\bibinfo{person}{Xiaodong Yang}, \bibinfo{person}{Huishuai
  Zhang}, \bibinfo{person}{Wei Chen}, {and} \bibinfo{person}{Tie-Yan Liu}.}
  \bibinfo{year}{2022}\natexlab{}.
\newblock \showarticletitle{Normalized/Clipped SGD with Perturbation for
  Differentially Private Non-Convex Optimization}.
\newblock \bibinfo{journal}{\emph{arXiv e-prints}} (\bibinfo{year}{2022}),
  \bibinfo{pages}{arXiv--2206}.
\newblock


\bibitem[Yin et~al\mbox{.}(2018)]%
        {yin2018Byzantine}
\bibfield{author}{\bibinfo{person}{Dong Yin}, \bibinfo{person}{Yudong Chen},
  \bibinfo{person}{Ramchandran Kannan}, {and} \bibinfo{person}{Peter
  Bartlett}.} \bibinfo{year}{2018}\natexlab{}.
\newblock \showarticletitle{Byzantine-robust distributed learning: Towards
  optimal statistical rates}. In \bibinfo{booktitle}{\emph{International
  Conference on Machine Learning}}. PMLR, \bibinfo{pages}{5650--5659}.
\newblock


\bibitem[Zhang et~al\mbox{.}(2021)]%
        {zhang2021understanding}
\bibfield{author}{\bibinfo{person}{Xinwei Zhang}, \bibinfo{person}{Xiangyi
  Chen}, \bibinfo{person}{Mingyi Hong}, \bibinfo{person}{Zhiwei~Steven Wu},
  {and} \bibinfo{person}{Jinfeng Yi}.} \bibinfo{year}{2021}\natexlab{}.
\newblock \showarticletitle{Understanding Clipping for Federated Learning:
  Convergence and Client-Level Differential Privacy}.
\newblock \bibinfo{journal}{\emph{arXiv preprint arXiv:2106.13673}}
  (\bibinfo{year}{2021}).
\newblock


\bibitem[Zheng et~al\mbox{.}(2020)]%
        {zheng2020sharp}
\bibfield{author}{\bibinfo{person}{Qinqing Zheng}, \bibinfo{person}{Jinshuo
  Dong}, \bibinfo{person}{Qi Long}, {and} \bibinfo{person}{Weijie Su}.}
  \bibinfo{year}{2020}\natexlab{}.
\newblock \showarticletitle{Sharp Composition Bounds for Gaussian Differential
  Privacy via Edgeworth Expansion}. In \bibinfo{booktitle}{\emph{International
  Conference on Machine Learning}}. PMLR, \bibinfo{pages}{11420--11435}.
\newblock


\bibitem[Zhu et~al\mbox{.}(2022)]%
        {zhu2022Byzantine}
\bibfield{author}{\bibinfo{person}{Banghua Zhu}, \bibinfo{person}{Lun Wang},
  \bibinfo{person}{Qi Pang}, \bibinfo{person}{Shuai Wang},
  \bibinfo{person}{Jiantao Jiao}, \bibinfo{person}{Dawn Song}, {and}
  \bibinfo{person}{Michael~I Jordan}.} \bibinfo{year}{2022}\natexlab{}.
\newblock \showarticletitle{Byzantine-Robust Federated Learning with Optimal
  Statistical Rates and Privacy Guarantees}.
\newblock \bibinfo{journal}{\emph{arXiv preprint arXiv:2205.11765}}
  (\bibinfo{year}{2022}).
\newblock


\bibitem[Zhu and Ling(2022)]%
        {zhu2022bridging}
\bibfield{author}{\bibinfo{person}{Heng Zhu} {and} \bibinfo{person}{Qing
  Ling}.} \bibinfo{year}{2022}\natexlab{}.
\newblock \showarticletitle{Bridging Differential Privacy and
  Byzantine-Robustness via Model Aggregation}.
\newblock \bibinfo{journal}{\emph{arXiv preprint arXiv:2205.00107}}
  (\bibinfo{year}{2022}).
\newblock


\bibitem[Zhu et~al\mbox{.}(2019)]%
        {zhu2019deepleak}
\bibfield{author}{\bibinfo{person}{Ligeng Zhu}, \bibinfo{person}{Zhijian Liu},
  {and} \bibinfo{person}{Song Han}.} \bibinfo{year}{2019}\natexlab{}.
\newblock \showarticletitle{Deep leakage from gradients}.
\newblock \bibinfo{journal}{\emph{Advances in Neural Information Processing
  Systems}}  \bibinfo{volume}{32} (\bibinfo{year}{2019}).
\newblock


\bibitem[Zhu and Wang(2019)]%
        {zhu2019poission}
\bibfield{author}{\bibinfo{person}{Yuqing Zhu} {and} \bibinfo{person}{Yu-Xiang
  Wang}.} \bibinfo{year}{2019}\natexlab{}.
\newblock \showarticletitle{Poission subsampled r{\'e}nyi differential
  privacy}. In \bibinfo{booktitle}{\emph{International Conference on Machine
  Learning}}. PMLR, \bibinfo{pages}{7634--7642}.
\newblock


\bibitem[Zhu et~al\mbox{.}(2020)]%
        {zhu2020private}
\bibfield{author}{\bibinfo{person}{Yuqing Zhu}, \bibinfo{person}{Xiang Yu},
  \bibinfo{person}{Manmohan Chandraker}, {and} \bibinfo{person}{Yu-Xiang
  Wang}.} \bibinfo{year}{2020}\natexlab{}.
\newblock \showarticletitle{Private-knn: Practical differential privacy for
  computer vision}. In \bibinfo{booktitle}{\emph{Proceedings of the IEEE/CVF
  Conference on Computer Vision and Pattern Recognition}}.
  \bibinfo{pages}{11854--11862}.
\newblock


\end{thebibliography}

\appendix  \label{sec:appendix}

\section{Supplementary Material}\label{appendix:add_detail}

\subsection{Implementation Detail}
\noindent\textbf{Hardware and software setup:} The experiments are run on a machine with CentOS Linux release 7.9.2009 (Core). Hardware includes Intel(R) Xeon(R) Gold 6142 CPU @ 2.60GHz and Tesla V100-SXM2-32GB. The code is implemented with PyTorch 1.11.0 and CUDA 10.2. It takes roughly 600 GPU hours to generate all experiment results.



\noindent\textbf{Model setup:} We use CNN network for MNIST and Colorectal and Fully-connected layer for Fashion and USPS, the model size $d$ is 21802, 33736, 25450, 25450 for MNIST, Colorectal, Fashion, and USPS, respectively. Detail setups of network architectures for each dataset are summarised as follows:

1) For MNIST, the network setup is shown in Table \ref{tab:mnist_network}; For Fashion and USPS, the network setup is shown in Table \ref{tab:fashion_usps_network}; For Colorectal, We make the CNN for Colorectal have a residual connection, due to its complexity, we do not describe it here, the detail can be found in our code.


\begin{table}[!ht] 
\footnotesize
    \begin{subtable}[h]{\columnwidth}
        \centering
        
        \begin{tabular}{c|c}
        \toprule
         Layer & Setup\\
        \midrule
         CNN & InputChannel=1, OutputChannel=16, KernelSize=5 \\
         ELU & -\\
         GroupNorm & NumGroups=4, NumChannels=16\\
         \midrule
         CNN & InputChannel=16, OutputChannel=16, KernelSize=5 \\
         ELU & -\\
         GroupNorm & NumGroups=4, NumChannels=16\\
         \midrule
         CNN & InputChannel=16, OutputChannel=16, KernelSize=5 \\
         ELU & -\\
         GroupNorm & NumGroups=4, NumChannels=16\\
         \midrule
         AdaptiveAvgPool & (4,4)\\
         \midrule
         Linear & (256, 32)\\
         ELU & -\\
         \midrule
         Linear & (32, 10)\\
        
        \bottomrule
        \end{tabular}
    \end{subtable}
     
     \caption{Network setup for MNIST. }
     \label{tab:mnist_network}
\end{table}

\begin{table}[!ht] 
\footnotesize
    \begin{subtable}[h]{\columnwidth}
        \centering
        
        \begin{tabular}{c|c}
        \toprule
         Layer & Setup\\
        \midrule
        flatten & -\\
        \midrule
        Linear & (784, 32)\\
        ELU & -\\
        \midrule
        Linear & (32, 10)\\
        
        \bottomrule
        \end{tabular}
    \end{subtable}
    \caption{Network setup for Fashion and USPS. }
    \label{tab:fashion_usps_network}
\end{table}

    
    


\subsection{List of Notations}
We summarise some useful notations we adopt throughout this paper in Table \ref{tab:notations}
\begin{table}[!ht] 
\footnotesize
\centering
  \begin{tabular}{r | p{0.7\linewidth}}
        \toprule
        Notations & Meaning \\
        \midrule
        $d$ & model size\\
        $\sigma$ & DP noise multiplier\\
        $\eta$ & learning rate\\
        $T$ & number of total iterations\\
        $w^{t-1}$ & model sent from the server to all workers at the beginning of $t$-th iteration\\
        $g_i^t$ & workers's uploading to server from $i$-th worker at $t$-th iteration\\
        $b_c$ & local batch size for each worker\\
        $D_p$ & data possessed by the server\\
        $f$ & loss function\\
        $G_s^t$ & set of selected gradient vectors by the server at round $t$\\
        
        \bottomrule
       \end{tabular}
    \caption{Some notations used throughout the paper.}
    \label{tab:notations}
\end{table}

\subsection{Byzantine-resilient Aggregation Methods}\label{app:byz_method}
Giving (gradient) vectors $\{g_i\}_{i=1}^n$, we summarise Krum \cite{blanchard2017machine}, RFA \cite{pillutla2019robust}, coordinate-wise median \cite{yin2018Byzantine} and Trimmed Mean \cite{yin2018Byzantine} as follows:

\begin{itemize}
    \item Krum: It returns a vector $g_{s}$, which is the closest vector to the mean after excluding $\gamma n+2$ furthest away points ($\gamma$ is some constant). Mathematically, let $\mathbb{S} \subseteq\{1, 2, \cdots, n\}$ be an index set with size at least $(n-\gamma n-2)$. Then,
    $$
    \operatorname{Krum}\left(g_{1}, \ldots, g_{n}\right)=\underset{g_{s}}{\arg \min } \min _{\mathbb{S}} \sum_{j \in \mathbb{S}}\left\|g_{s}-g_{j}\right\|^{2}
    $$
    \item Robust Federated Averaging (RFA): It just returns the geometric median of all vectors: 
    $$
    \operatorname{RFA}\left(g_{1}, \ldots, g_{n}\right)=\underset{{g}}{\arg \min } \sum_{i=1}^{n}\left\|{g}-g_{i}\right\|_{2}
    $$
    \item Coordinate-wise median: Let $v[j]$ refer to the $j$-th coordinate of a vector $v$, then the output is just the median of each coordinate of these vectors: 
    $$
    \operatorname{CM}\left(g_{1}, \ldots, g_{n}\right)\left[j\right]=\operatorname{median}\left(g_{1}\left[j\right], \ldots,g_{n}\left[j\right]\right)
    $$

    \item Trimmed Mean: For the $j$-the coordinate, it computes the average after excluding excluding $\gamma n$ largest and smallest values, {\em i.e.,} $$
    \operatorname{TM}\left(g_{1}, \ldots, g_{n}\right)\left[j\right]=\frac{1}{n-2 \gamma n} \sum_{k=\gamma n}^{n-\gamma n}g_{\Pi_{j}(i)}\left[j\right], 
    $$
    where $g_{\Pi_{j}(k)}\left[j\right]$  represents the $k$-th largest value among all $\{g_i[j]\}_{s=1}^n$. 
\end{itemize}

\begin{algorithm}[!ht]
\caption{$\mathbf{GetNonIID}(D,n)$}\label{alg:noniid_generating}
\begin{algorithmic}[1]
\small
\renewcommand{\algorithmicrequire}{\textbf{Input:}}
\renewcommand{\algorithmicensure}{\textbf{Output:}}
\Require $D$, the dataset to be distributed to all workers; $n$, number of workers
    \State Partition $D$ (by class) into $G_1,G_2,\cdots,G_H$
       \State $D_i \gets []$, $T_i \gets []$, for $i=1,2,\cdots,n$
    
    \For{$k=1,2,\cdots,H$} \label{alg:noniid_generating_rand_begin}

        \State Generate a list $V$ of uniform R.V.s and normalize $V$
        \label{alg:noniid_generating_iid_and_noniid}

            \State Partition $G_k$ according to $V$
            \State Append parts of $G_k$ to $T_i$ for $i=1,2,\cdots,n$
        
        
    
    \EndFor\label{alg:noniid_generating_rand_end}
    
    
    \State Concatenate all $T_i$ into $L$
    \State $s \gets \left\lceil\frac{|L|}{n}\right\rceil$
    \For{$i=1,2,\cdots,n$}
       \State $D_i \gets L[(i-1)\cdot s:i\cdot s]$ 
    \EndFor

\Ensure $D_1,D_2,\cdots,D_n$
\end{algorithmic}
\end{algorithm}

\subsection{Dataset Details and Distribution Setups}\label{app:datasets_detail}
\noindent\textbf{Datasets:} 
1) MNIST \cite{mnist} is a hand-written digit dataset.  It contains 60,000 training examples and 10,000 testing examples, equally distributed among the 10 classes. 2) Colorectal \cite{colorectal}, which has 8 classes and 5,000 histological RGB images of size $150 \times 150$.  The image size is significantly greater than MNIST (which is $28\times 28$). We randomly take 4600 examples for training and take the rest 400 examples for testing. 3) Fashion \cite{fashion} is an image classiﬁcation dataset with 10 classes and it has 60,000 training images and 10,000 testing images. 4) USPS \cite{usps}. This dataset has 10 classes and 7,291 samples for training and 2,007 for testing.

\noindent\textbf{Generating non-i.i.d. data distribution:}  
The method for generating non-i.i.d. data distributions is summarised in Algorithm \ref{alg:noniid_generating}, and the method for i.i.d. case can be trivially obtained according to line \ref{alg:noniid_generating_iid_and_noniid}. The intuition for why our method can generate non-i.i.d. data distribution is that, according to line \ref{alg:noniid_generating_iid_and_noniid} in Algorithm \ref{alg:noniid_generating}, the probability that the non-i.i.d. case happens to be i.i.d. is negligible. In practice, applying Algorithm \ref{alg:noniid_generating} gives us reasonable result which is non-i.i.d. distributed. We provide the non-i.i.d. distribution simulation result for dataset MNIST as evidence. We can see from Figure \ref{fig:noniid_distribution} that, the data distribution is non-i.i.d.: for a certain class, its sampling probability for each worker is different. Take class 1 for example, its sampling ratio for worker 0 (first sub-figure of first row) is around 0.2, whereas for worker 1 (second sub-figure of first row), it is 0 and for worker 3 (fourth sub-figure of first row), it is almost 0.3. In contrast, for the i.i.d. case, the probability for all classes and for all workers should be 0.1.

\begin{figure*}[!ht] 
    \centering
    \includegraphics[width=0.90\linewidth]{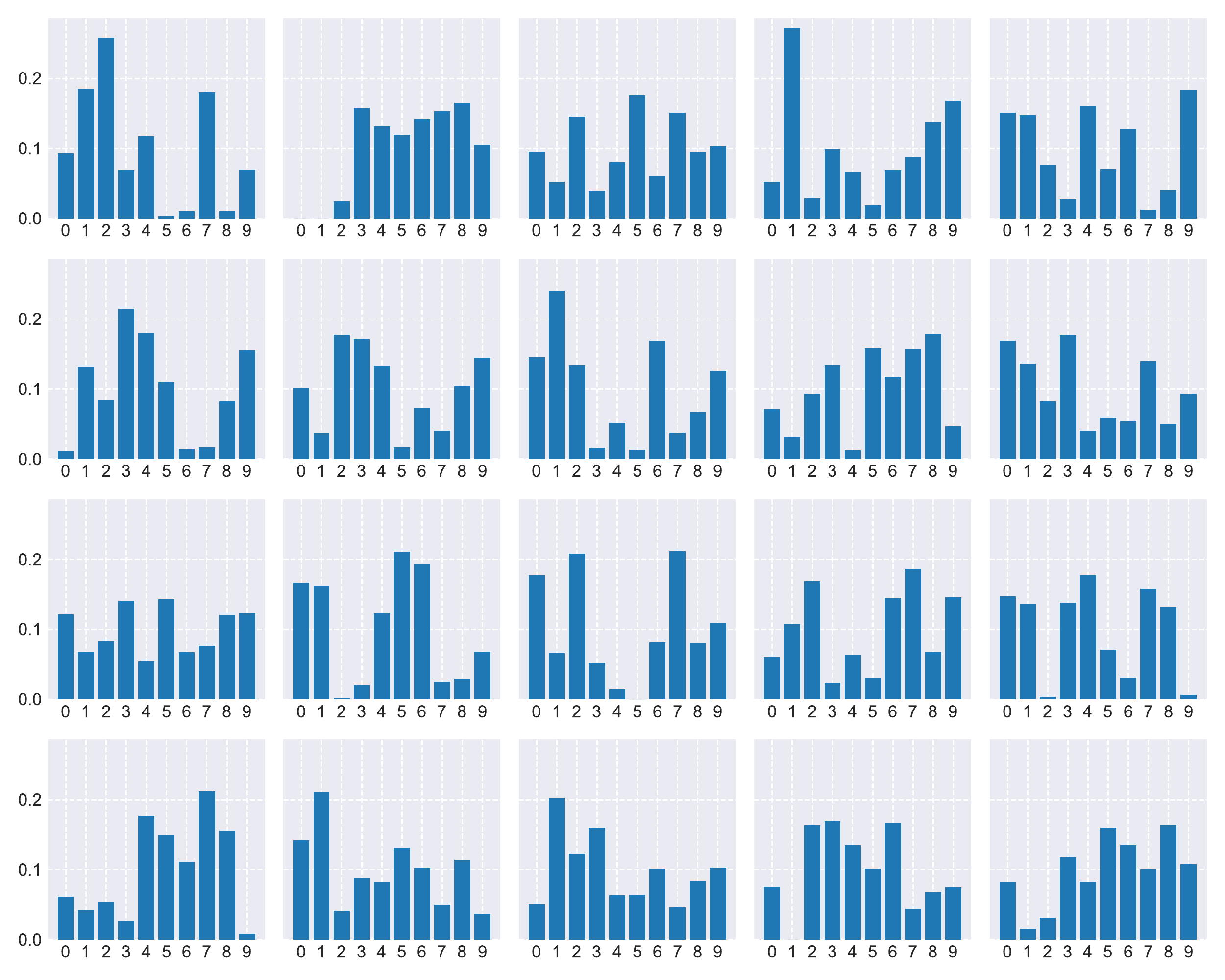}
    \caption{Simulation result for non-i.i.d. distribution for dataset MNIST to 20 workers. 0-9 in the horizontal axis stands for class labels, the vertical axis is the ratio of the number of a certain class to the size of the datasets each worker holds. There are 20 subplots, standing for 20 different workers. }
    
    \label{fig:noniid_distribution}
\end{figure*}



\subsection{Proofs} 

\subsubsection{Proof of Theorem 1}

\begin{proof} \label{proof:general_convergence}
Our parameter $w$ is a $d$ dimensional vector, considering that our parameter updating rule in the $t$-th iteration is:

\begin{equation}
    w^{t} = w^{t-1} - \frac{\eta^t }{|B^t|}\left(\sum\limits_{g^t \in B^t} \frac{g^t}{\left\|g^t\right\|}+ z\right)
\end{equation}
where $B^t$ is the current batch of per-example gradient (we fix the batch size to be $|B^t|=b_c$), $z \sim \mathcal{N}(0, \sigma^2 I)$ is the DP noise and $g^t=\nabla F(w^{t-1}) + \xi^t $ is the returned result by our stochastic gradient oracle. Considering that we assume $F$ is $L$-smooth, then we have:

\begin{equation}
    \begin{aligned}
        F(w^{t}) - F(w^{t-1}) \leq & \left\langle \nabla F(w^{t-1}), w^{t} - w^{t-1} \right\rangle \\
        & + \frac{L}{2}\left\|w^{t} - w^{t-1}\right\|
    \nonumber
    \end{aligned}
\end{equation}

which translates to:
\begin{equation}\label{equ:decent_update}
    \begin{aligned}
         F(w^{t}) - F(w^{t-1}) \leq & -\frac{\eta^t }{b_c}\left\langle \nabla F(w^{t-1}),  \sum\limits_{g^t \in B^t} \frac{g^t}{\left\|g^t\right\|}+ z \right\rangle \\
         & + \frac{L(\eta^t)^2}{2}\left\|\frac{1}{b_c}\sum\limits_{g^t \in B^t} \frac{g^t}{\left\|g^t\right\|}+ \frac{z}{b_c}\right\|
    \end{aligned}
\end{equation}
note that each $g^t$ inside $B^t$ is i.i.d. (because each $g^t$ is calculated by one data example and each data example is sampled from dataset independently with replacement), conditioned on $w^{t-1}$ and take expectation on both sides with respect to the randomness of data sampling and DP noise:

\begin{equation}\label{equ:decent_bound}
    \begin{aligned}
         \mathbb{E}(F(w^{t}) )- F(w^{t-1}) \leq & -\frac{\eta^t }{b_c}\sum\limits_{g^t \in B^t} \mathbb{E}\left\langle \nabla F(w^{t-1}), \frac{g^t}{\left\|g^t\right\|} \right\rangle \\
         & + \frac{L(\eta^t)^2}{2}\mathbb{E}\left\|\frac{1}{b_c}\sum\limits_{g^t \in B^t} \frac{g^t}{\left\|g^t\right\|}+ \frac{z}{b_c}\right\|\\
         \leq & -\frac{\eta^t }{b_c}\sum\limits_{g^t \in B^t} \mathbb{E}\left\langle \nabla F(w^{t-1}), \frac{g^t}{\left\|g^t\right\|} \right\rangle \\
         & + \frac{L(\eta^t)^2}{2}(1+ \frac{\sigma^2d}{b_c^2})\\
         = & -\eta^t\mathbb{E}\left\langle \nabla F(w^{t-1}), \frac{g^t}{\left\|g^t\right\|} \right\rangle \\
         & + \frac{L(\eta^t)^2}{2}\left(1+ \frac{\sigma^2d}{b_c^2}\right)
    \end{aligned}
\end{equation}

we now deal with the term $\left\langle \nabla F(w^{t-1}), \frac{g^t}{\left\|g^t\right\|} \right\rangle$ first, if $\left\|\nabla F(w^{t-1})\right\| > 2\left\|  \xi^t \right\|$, we then have: 
\begin{equation}
    \begin{aligned}
         \left\langle \nabla F(w^{t-1}), \frac{g^t}{\left\|g^t\right\|} \right\rangle = & \frac{\left \|\nabla F(w^{t-1})  \right\|^2 +  \left\langle \nabla F(w^{t-1}),  \xi^t \right\rangle}{\left \|\nabla F(w^{t-1}) + \xi^t \right\|} \\
         = & \frac{\left \|\nabla F(w^{t-1})  \right\|^2 }{\left \|\nabla F(w^{t-1}) + \xi^t \right\|} \\
         & + \frac{  \left \|\nabla F(w^{t-1})  \right\|\left \|\xi^t  \right\|cos(\theta)}{\left \|\nabla F(w^{t-1}) + \xi^t \right\|} \\
         \geq & \frac{\left \|\nabla F(w^{t-1})  \right\|^2 - \frac{1}{2}  \left \|\nabla F(w^{t-1})  \right\|^2}{\left \|\nabla F(w^{t-1}) + \xi^t \right\|} \\
         = & \frac{\left \|\nabla F(w^{t-1})  \right\|^2}{2\left \|\nabla F(w^{t-1}) + \xi^t \right\|} \\
         \geq & \frac{\left \|\nabla F(w^{t-1})  \right\|^2}{2\left(\left \|\nabla F(w^{t-1})\right\| + \left\|\xi^t \right\|\right)} \\
         \geq & \frac{\left \|\nabla F(w^{t-1})  \right\|}{3} \\
         \geq & \frac{\left \|\nabla F(w^{t-1})  \right\|}{3} - \frac{8\left\| \xi^t \right\|}{3}
    \nonumber
    \end{aligned}
\end{equation}
if $\left\|\nabla F(w^{t-1})\right\| \leq 2\left\|  \xi^t \right\|$, we then have:

\begin{equation}
    \begin{aligned}
         \left\langle \nabla F(w^{t-1}), \frac{g^t}{\left\|g^t\right\|} \right\rangle \geq & -\left\|\nabla F(w^{t-1})\right\|\\
         = & \frac{\left\|\nabla F(w^{t-1})\right\|}{3} - \frac{4\left\|\nabla F(w^{t-1})\right\|}{3}\\
         \geq & \frac{\left\|\nabla F(w^{t-1})\right\|}{3} - \frac{8\left\| \xi^t \right\|}{3}
    \nonumber
    \end{aligned}
\end{equation}
Hence we have:
\begin{equation} \label{equ:inner_lower_bound}
     \left\langle \nabla F(w^{t-1}), \frac{g^t}{\left\|g^t\right\|} \right\rangle \geq \frac{\left\|\nabla F(w^{t-1})\right\|}{3} - \frac{8\left\| \xi^t \right\|}{3}
\end{equation}
applying Equation \ref{equ:inner_lower_bound} to Equation \ref{equ:decent_bound} yields:

\begin{equation}\label{equ:decent_bound_final}
    \begin{aligned}
         \mathbb{E}(F(w^{t}) )- F(w^{t-1})  \leq & -\eta^t \frac{\mathbb{E}\left\|\nabla F(w^{t-1})\right\|}{3} + \eta^t\frac{8\mathbb{E}\left\| \xi^t \right\|}{3} \\
         & + \frac{L(\eta^t)^2}{2}\left(1+ \frac{\sigma^2d}{b_c^2}\right)\\
    \end{aligned}
\end{equation}
assume our loss function $F(w^t)>0$, and the result $g^t=\nabla F(w^{t-1}) + \xi^t $ which is returned by our stochastic gradient oracle has bounded variance, if we set a constant learning rate $\eta^t = \eta$, take iterated expectation, summing up both side from iteration 1 to $T$, divide both sides by $T\eta$, rearrange terms, we have:

\begin{equation}\label{equ:decent_bound_final_summed_up}
    \begin{aligned}
         \frac{1}{T}\sum_{t=1}^{t=T}\mathbb{E}\left\|\nabla F(w^t)\right\| \leq & \frac{3F(w^0)}{T\eta} + \frac{3L\eta}{2}\left(1 + \frac{\sigma^2d}{b_c^2}\right)\\
         & + \frac{8}{T}\sum_{t=1}^{t=T}\mathbb{E}\left\| \xi^t \right\|\\
         \leq & \frac{3F(w^{0})}{T\eta} + \frac{3L\eta}{2}\left(1 + \frac{\sigma^2d}{b_c^2}\right)\\
         & + 8\nu
    \end{aligned}
\end{equation}

\end{proof}

\subsubsection{Proof of Theorem 3}\label{app:proof_privacy_guarantee}
The following is proof for Theorem 3.

\begin{proof}
Algorithm 1 deals with each worker's privacy, we will show that it is  $(\epsilon,\delta)$-DP for each worker-hold dataset. Note that the $l_2$-sensitivity of the term $\sum\limits_{j \in [b_c]}\frac{\phi_i[j]}{ \left\| \phi_i[j]\right\|_2}$ is 2, thus by the same proof in Theorem 1 of \cite{abadi2016deep}. We can see that if $\sigma \geq c_2\frac{q\sqrt{T\ln \frac{1}{\delta}}}{\epsilon}$, Algorithm 1 guarantee $(\epsilon,\delta)$-DP for each worker-hold dataset.
\end{proof}

\subsubsection{Proof of Theorem 4}\label{appendix:proof_of_thm_byz}

The following is proof for Theorem 4. Before the proof, we first recall some technical lemmas. 
\begin{lemma}[\cite{wang2019differentially}]\label{gaussian}
For the Gaussian random variable $x\sim \mathcal{N}(0, \sigma^2I_d)$, with probability at least $1-\zeta$ for any $\zeta\in (0, 1)$ we have 
\begin{equation}
    \|\sigma\|_2\leq \sqrt{2d\log \frac{1}{\zeta}}\sigma
    \nonumber
\end{equation}
\end{lemma}
\begin{lemma}[\cite{caofltrust}]
Under Assumptions 1 and 2 and assume the parameter space $\Theta \subseteq B(0,r\sqrt{d})$. Then for any $\zeta\in (0, 1)$, if $\Delta_1\leq \frac{\sigma_1^2}{\gamma_1} $ and $\Delta_2\leq \frac{\sigma_2^2}{\gamma_2}$, then we have the following for all $t\in [T]$: 
\begin{equation}
   Pr\{ \|g_s^{t-1}-\nabla F(w^{k-1})\|\leq 8\Delta_2 \|w^{k-1}-w^*\|+4\Delta_1 \}\geq 1-\zeta, 
   \nonumber
\end{equation}
where $\Delta_2=\sigma_2\sqrt{\frac{2}{|D_p|}}\sqrt{K_1+K_2}$ with $K_1=d\log \frac{\max\{L, L_1\}}{\sigma_2}$ and $K_2=\frac{d}{2}\log\frac{|D_p|}{d}+\log \frac{6\sigma^2 r\sqrt{|D_p|}}{\gamma_2\sigma_1\zeta}$, and $\Delta_1=\sqrt{2}\sigma_1\sqrt{\frac{d\log 6+\log(3/\delta)}{|D_p|}}$.
\end{lemma}
\begin{proof}
For a fixed iteration $t$, we first assume the following is true (later we will show it will hold for all $t\in [T]$): 
\begin{equation}
     \left\|w^{t-1}-w^*\right\| \leq \frac{r}{2}\sqrt{d}.
     \nonumber
\end{equation}
Since we have $\eta_{t-1}= \left\|g_s^{t-1}\right\|_2 \eta_0 $, then  
\begin{equation}\label{equ:thm2_target}
    \begin{aligned}
        &\left\|w^{t-1}-w^{*}\right\| =\left\|w^{t-1}-\eta_{t-1}\left(\frac{1}{n} \sum_{g \in G_{s}^t}g\right)-w^{*}\right\|  \\
        =& \|w^{t-1}-\eta_{0} \left(\frac{\left\| g_{s}^{t-1}\right\|}{n} \sum_{g \in G^t_{s}} g\right) \\
        &+\eta_{0} \nabla  F\left(w^{t-1}\right)-\eta_{0} \nabla  F\left(w^{t-1}\right)-w^{*} \| \\
        \leq& \underbrace{\left\|w^{t-1}-\eta_{0} \nabla  F\left(w^{t-1}\right)-w^{*}\right\| }_{A}\\
        +&\eta_{0}\underbrace{\left\|\nabla  F\left(w^{t-1}\right)-\frac{\left\|g_{s}^{t-1}\right\|}{n} \sum_{g \in G_{s}^{t}} g\right\| }_{B}.
    \end{aligned}
\end{equation}
We first consider term $A$ and we consider the following lemma. 
\begin{lemma}\label{lem:0th}
\cite{caofltrust}, Assume Assumption 1 holds. If we set the learning rate satisfies $\alpha\leq \mu /\left(2 L^2\right)$, then we have the following in any  $t \geq 1$ :
\begin{equation}
    \left\| {w}^{t-1}- {w}^{*}-\alpha \nabla  F\left( {w}^{t-1}\right)\right\| \\
    \leq \sqrt{1-\frac{\mu^2}{4 L^2}}\left\| {w}^{t-1}- {w}^{*}\right\|
    \nonumber
\end{equation}

\begin{proof}
Since $\nabla  F\left( {w}^{*}\right)=0$, we have the following:
\begin{equation}\label{equ:0th}
    \begin{aligned}
    &\left\| {w}^{t-1}- {w}^{*}-\alpha \nabla  F\left( {w}^{t-1}\right)\right\|^2  \\
    &=\left\| {w}^{t-1}- {w}^{*}-\alpha\left(\nabla  F\left( {w}^{t-1}\right)-\nabla  F\left( {w}^{*}\right)\right)\right\|^2 \\
    &=\left\| {w}^{t-1}- {w}^{*}\right\|^2+\alpha^2\left\|\nabla  F\left( {w}^{t-1}\right)-\nabla  F\left( {w}^{*}\right)\right\|^2 \\
    &-2 \alpha\left\langle {w}^{t-1}- {w}^{*}, \nabla  F\left( {w}^{t-1}\right)-\nabla  F\left( {w}^{*}\right)\right\rangle
    \end{aligned}
\end{equation}

By Assumption 1, we have:
\begin{equation}\label{equ:1st}
    \left\|\nabla  F\left( {w}^{t-1}\right)-\nabla  F\left( {w}^{*}\right)\right\| \leq L\left\| {w}^{t-1}- {w}^{*}\right\|
\end{equation}
\begin{align}
      & F\left( {w}^{*}\right)+\left\langle\nabla  F\left( {w}^{*}\right),  {w}^{t-1}- {w}^{*}\right\rangle \notag \\
   & \leq  F\left( {w}^{t-1}\right)-\frac{\mu}2\left\| {w}^{t-1}- {w}^{*}\right\|^2  \label{equ:2nd}
\end{align}
\begin{equation}\label{equ:3rd}
     F\left( {w}^{t-1}\right)+\left\langle\nabla  F\left( {w}^{t-1}\right),  {w}^{*}- {w}^{t-1}\right\rangle \leq  F\left( {w}^{*}\right)
\end{equation}
Summing up inequalities \ref{equ:2nd} and \ref{equ:3rd}, we have:
\begin{equation}\label{equ:4th}
    \begin{aligned}
         &\left\langle {w}^{*}- {w}^{t-1}, \nabla  F\left( {w}^{t-1}\right)-\nabla  F\left( {w}^{*}\right)\right\rangle\\
         &\leq-\frac{\mu}2\left\|{w}^{t-1}-{w}^{*}\right\|^2.
    \end{aligned}
\end{equation} 
Substituting inequalities \ref{equ:1st} and \ref{equ:4th} into \ref{equ:0th}, we have:
\begin{multline}
    \left\|{w}^{t-1}-{w}^{*}-\alpha \nabla F\left({w}^{t-1}\right)\right\|^2 \\
    \leq\left(1+\alpha^2 L^2-\alpha \mu\right)\left\|{w}^{t-1}-{w}^{*}\right\|^2.
\end{multline} 
By choosing $\alpha\leq \frac{\mu}{ 2 L^2}$, we have:
\begin{equation}
    \left\|{w}^{t-1}-{w}^{*}-\alpha \nabla F\left({w}^{t-1}\right)\right\|^2 \\
    \leq\left(1-\frac{\mu^2}{4 L^2}\right)\left\|{w}^{t-1}-{w}^{*}\right\|^2,
\end{equation} 
which concludes proof for lemma \ref{lem:0th}.
\end{proof}
\end{lemma}
For term $B$, note that,

\begin{equation}\label{equ:B}
    \begin{aligned}
        &\left\|\nabla  F\left(w^{t-1}\right)-\frac{\left\|g_{s}^{t-1}\right\|}{n} \sum_{g \in G_{s}^t} g\right\|_2\\
        &\leq \left\| {\left\|g_{s}^{t-1}\right\| \frac{1}{n}\sum_{g \in G_{s}^t}g -g_s^{t-1}}\right\| + \left\|g_s^{t-1}-  \nabla  F\left(w^{t-1}\right) \right\|\\
        &\leq \underbrace{\left\|g_{s}^{t-1}\right\|\left\| { \frac{1}{n}\sum_{g \in G_{s}^{t}}g}\right\|}_{C} + \left\|{g_s^{t-1}}\right\| + \left\|g_s^{t-1}-  \nabla  F\left(w^{t-1}\right) \right\|.
    \end{aligned}
\end{equation}
For term $C$, we can bound it as:
\begin{equation}
    \begin{aligned}
        C&=\left\|g_{s}^{t-1}\right\|\left\| { \frac{1}{n}\sum_{g \in G_{s}^{t}}g}\right\|\\
        &=\left\|g_{s}^{t-1}\right\| \left\| { \frac{1}{n}\sum_{g \in G_{s}^{t}}\Tilde{g} + z}\right\|\\
        &\leq \underbrace{\left\|g_{s}^{t-1}\right\| \left\| { \frac{1}{n}\sum_{g \in G_{s}^{t}}\Tilde{g}}\right\|}_{D} + \underbrace{\left\|g_{s}^{t-1}\right\| \left\|z\right\|}_{E},
    \end{aligned}
\end{equation}
here $\Tilde{g}=\frac{1}{b_c}\sum\limits_{j \in [b_c]}\frac{\phi_i[j]}{ \left\| \phi_i[j]\right\|}_2$ is just $g$ without noise(refer to line 11) in Algorithm 1 and $z \sim \mathcal{N}(0, \frac{|G^t_s|}{n^2b_c^2}\sigma^2)$ which is the gaussian noise ensuring DP. 

For term $D$, $D\leq \left\|g_{s}^{k-1}\right\| $ since $\left\|\Tilde{g}\right\|\leq 1$ and $|G_s^t|\leq n$; For term $E$, by Lemma \ref{gaussian} we have with probability at least  $1-\xi$ that, $$\left\|z\right\|\leq \frac{\sqrt{2d|G_s^t|\ln\frac{1}{\xi}}}{nb_c}\sigma\leq  \frac{\sqrt{2d\ln\frac{1}{\xi}}}{\sqrt{n} b_c}\sigma.$$ This leads to that, w.p $1-\xi$,
\begin{equation}
    C \leq \left\|g_{s}^{t-1}\right\| \left (1 + \frac{ \sqrt{2d\ln\frac{1}{\xi}}\sigma}{\sqrt{n} b_c}\right)
\end{equation}
Applying this inequality to \ref{equ:B} give us,
\begin{equation}\label{equ:B_weaker}
    \begin{aligned}
        B\leq& 2\left\|g_{s}^{t-1}\right\| + \left\|g_s^{t-1}-  \nabla  F\left(w^{t-1}\right) \right\| +\frac{ \left\|g_{s}^{t-1}\right\|  \sqrt{2d\ln\frac{1}{\xi}}\sigma}{\sqrt{n} b_c}\\
        \leq& 3\left\|g_{s}^{t-1}-\nabla  F\left(w^{t-1}\right)\right\| + 2\left\|\nabla  F\left(w^{t-1}\right)\right\| \\
        &+\frac{ \left\|g_{s}^{t-1}-\nabla  F\left(w^{t-1}\right)\right\|  \sqrt{2d\ln\frac{1}{\xi}}\sigma}{\sqrt{n} b_c}\\
        &+\frac{ \left\|\nabla F\left(w^{t-1}\right)\right\|  \sqrt{2d\ln\frac{1}{\xi}}\sigma}{\sqrt{n} b_c}\\
    \end{aligned}
\end{equation}
For $\left\|g_s^{t-1}-  \nabla  F\left(w^{k-1}\right) \right\| $, by Lemma 2 we have, w.p $1-\xi$, 
\begin{equation}\label{equ:grad_dif}
    \left\|g_{s}^{t-1}-\nabla F\left(w^{k-1}\right)\right\| \leq 8 \Delta_2\|w^{t-1}-w\| +4 \Delta_1. 
\end{equation}
For $\left\|\nabla  F\left(w^{t-1}\right) \right\|$, we have, 
\begin{equation}\label{equ:grad_abs}
    \begin{aligned}
        \left\|\nabla  F\left(w^{t-1}\right) \right\| \leq & \left\|\nabla  F\left( {w}^{t-1}\right)-\nabla  F\left( {w}^{*}\right)\right\|  +\left\| \nabla  F\left( {w}^{*}\right)\right\| \\
        = & \left\|\nabla  F\left( {w}^{t-1}\right)-\nabla  F\left( {w}^{*}\right)\right\| \\
        \leq& L\left\| {w}^{t-1}- {w}^{*}\right\|
    \end{aligned}
\end{equation}
Applying (\ref{equ:grad_dif}) and (\ref{equ:grad_abs}) to (\ref{equ:B_weaker}), then combine it with Lemma \ref{lem:0th} to (\ref{equ:thm2_target}), we have,
\begin{equation}\label{equ:dif_to_optimal}
    \begin{aligned}
        &\left\|w^t-w^{*}\right\| \leq \sqrt{1-\frac{u^{2}}{4 L^{2}}}\left\|w^{t-1}-w^{*}\right\|\\
        &+\left(24 \eta_{0} \Delta_{2}+2 \eta_{0} L\right)\left\|w_{t-1}-w^{*}\right\|\\
        &+ \left(8 \Delta_{2} \eta_{0}+\eta_{0} L\right) \frac{ \sqrt{2d \ln\frac{1}{\xi}} \sigma}{\sqrt{n} b_{c}}\left\|w_{k-1}-w^{*}\right\|\\
        &+\frac{4\eta_{0}  \sqrt{2d\ln\frac{1}{\xi}}\sigma}{\sqrt{n} b_{c}} \Delta_1 +12\eta_0\Delta_1.
    \end{aligned}
\end{equation}
Note that for the DP noise multiplier $\sigma$ (refer to Theorem 3), if we take $\sigma=c_2\frac{b_c\sqrt{T\ln \frac{1}{\delta}}}{|D|\epsilon}$, then we have $\frac{ \sqrt{2\ln\frac{1}{\xi}} \sigma}{\sqrt{n} b_{c}}= \mathcal{O}\left(\frac{\sqrt{T\ln\frac{1}{\delta}\ln\frac{1}{\xi}}}{|D|\sqrt{n}\epsilon}\right)$. Recall our assumption in Theorem 4 that $|D|\sqrt{n} \geq \Omega\left(\frac{\sqrt{T\ln{\frac{1}{\delta}\ln{\frac{1}{\xi}}}}\sqrt{d}}{\epsilon}\right)$. Thus, we have $\frac{ \sqrt{2T\ln\frac{1}{\xi} {d}} \sigma}{\sqrt{n} b_{c}\epsilon} \leq 1$. By taking $\sigma$ as specified in Theorem 3, we can update (\ref{equ:dif_to_optimal}) as
\begin{equation}
    \begin{aligned}
        \left\|w^t-w^*\right\| \leq& \left(\sqrt{1-\frac{u^{2}}{4 L^{2}}}+32 \eta_{0} \Delta_{2}+3 \eta_{0} L\right)\left\|w^{t-1}-w^{*}\right\| \\
        &+\frac{4 c_{1} \sqrt{\ln \frac{1}{\xi} \ln \frac{1}{\delta} T} \eta_{0}}{\sqrt{n}|D|\epsilon} \Delta_{1} \sqrt{d}+12 \eta_{0} \Delta_{1}\\
        =&\left(1-\rho\right)\left\|w^{t-1}-w^{*}\right\| + \mathcal{O}\left(\frac{d\ln\frac{1}{\xi}\sqrt{T\ln\frac{1}{\delta}}\sigma_1}{|D|\sqrt{|D_p| n}\epsilon}\right)\\
        &+O( \eta_{0}\sigma_1\sqrt{\frac{d}{|D_p|}}), 
    \nonumber
    \end{aligned}
\end{equation}
where $\rho = 1- \sqrt{1-\frac{u^{2}}{4 L^{2}}}-32 \eta_{0} \Delta_{2}-3 \eta_{0} L$.

Since $\left\|w^{t-1}-w^{*}\right\| \subset B(0, r\sqrt{d})$, if $|D|\sqrt{|D_p| n}$ is sufficiently large such that $|D|\sqrt{|D_p| n}\geq \Omega\left(\frac{\sqrt{d}\ln{\frac{1}{\xi}}\sqrt{\ln{\frac{1}{\delta}}T}}{\epsilon r\rho}\right)$ and $\eta_0 \sigma_1/\sqrt{|D_p|} \leq \mathcal{O}\left( \frac{r\sqrt{d}}{\rho}\right)$, then $\left\|w^k-w^*\right\| \in  B(0, r\sqrt{d})$. Thus, take the union w.r.t all iterations, with probability at least $1-\xi T$, we have,

\begin{equation}
    \begin{aligned}
        \left\|w^T-w^*\right\| \leq& \left(1-\rho\right)^T \left\|w^0-w^*\right\|\\
        &+ \mathcal{O}\left(\frac{\eta_0 d\ln\frac{1}{\xi}\sqrt{T\ln\frac{1}{\delta}}\sigma_1}{|D|\sqrt{b_sn}\rho\epsilon}+\frac{\eta_0\Delta_1}{\rho}\right)
    \nonumber
    \end{aligned}
\end{equation}
By taking $ T= \mathcal{O} \left( \ln_{1-\rho}{ \frac{1}{\sqrt{n}|D| \sqrt{|D_p|}} }\right) = \mathcal{O} \left(\frac{1}{\rho} \ln{ \left(\sqrt{n}|D| \sqrt{\left|D_{0}\right|}\right)}\right)$, we have,
\begin{equation}
    \left\|w^T-w^*\right\| \leq \Tilde{\mathcal{O}}\left( \frac{1}{\rho^2} \frac{d\ln{\frac{1}{\zeta}}\sqrt{\ln{\frac{1}{\delta}}}\sigma_1}{ |D|\sqrt{nb_s}\epsilon} + \frac{1}{\rho} \frac{\sigma_1\sqrt{d\ln{\frac{1}{\zeta}}}}{\sqrt{|D_p|}} \right). 
    \nonumber
\end{equation}

\end{proof}

\begin{table}[!ht] 
\footnotesize
    \begin{subtable}[h]{\columnwidth}
        \centering
        \begin{tabular}{c||c|c|c|c|c|c|c|c}
        \toprule
        
        \multirow{2}{*}{$\gamma$} & \multicolumn{2}{|c|}{MNIST} & \multicolumn{2}{|c|}{COLOR.} & \multicolumn{2}{|c|}{FASHION} &  \multicolumn{2}{|c}{USPS} \\
        
        \cline{2-9}
       & \multicolumn{2}{|c|}{$\epsilon=$} & \multicolumn{2}{|c|}{$\epsilon=$} & \multicolumn{2}{|c|}{$\epsilon=$} &  \multicolumn{2}{|c}{$\epsilon=$}\\
 
       & $\frac{1}{8}$&$2$  & $\frac{1}{8}$&$2$ & $\frac{1}{8}$&$2$ & $\frac{1}{8}$&$2$  \\
        \midrule
        \textit{$20\%$} & .87&.95 & .48&.73 & .69&.78 & .64&.85 \\
        \textit{$35\%$} & .88&.96 & .47&.74 & .69&.79 & .63&.86 \\
        \textit{exact ($50\%$) } & .88&.96 & .49&.74 & .69&.80 & .64&.87 \\
        \textit{$65\%$} & .85 & .96 & .48&.73 & .70&.79 & .64&.87 \\
        \textit{$80\%$} & .84 & .95 & .48&.74 & .69&.79 & .64&.85 \\
        
        \bottomrule
        \end{tabular}
    \end{subtable}
    
     \caption{Experimental result (test accuracy) under Gaussian attack for different $\gamma$ in i.i.d. setting.}
     \label{tab:ablation_on_gamma_iid_gaussian}
\vspace{-0.5cm}
\end{table}

\begin{table}[!ht] 
\footnotesize
    \begin{subtable}[h]{\columnwidth}
        \centering
        \begin{tabular}{c||c|c|c|c|c|c|c|c}
        \toprule
        
        \multirow{2}{*}{$\gamma$} & \multicolumn{2}{|c|}{MNIST} & \multicolumn{2}{|c|}{COLOR.} & \multicolumn{2}{|c|}{FASHION} &  \multicolumn{2}{|c}{USPS} \\
        
        \cline{2-9}
       & \multicolumn{2}{|c|}{$\epsilon=$} & \multicolumn{2}{|c|}{$\epsilon=$} & \multicolumn{2}{|c|}{$\epsilon=$} &  \multicolumn{2}{|c}{$\epsilon=$}\\
 
       & $\frac{1}{8}$&$2$  & $\frac{1}{8}$&$2$ & $\frac{1}{8}$&$2$ & $\frac{1}{8}$&$2$  \\
        \midrule
        \textit{$20\%$} & .86&.95 & .48&.73 & .68&.78 & .64&.85 \\
        \textit{$35\%$} & .87&.96 & .47&.74 & .69&.79 & .63&.86 \\
        \textit{exact ($50\%$) } & .88&.96 & .49&.74 & .69&.80 & .64&.87 \\
        \textit{$65\%$} & .84 &.96 & .43&.73 & .68&.79 & .52&.87 \\
        \textit{$80\%$} & .80 &.95 & .30&.74 & .60&.79 & .47&.85 \\
        
        \bottomrule
        \end{tabular}
    \end{subtable}
    
     \caption{Experimental result (test accuracy) under  Optimized Local Model Poisoning attack for different $\gamma$ in i.i.d. setting.}
     \label{tab:ablation_on_gamma_iid_local}
\vspace{-0.5cm}
\end{table}

\begin{table}[!ht] 
\footnotesize
    \begin{subtable}[h]{\columnwidth}
        \centering
        \begin{tabular}{c||c|c|c|c|c|c|c|c}
        \toprule
        
        \multirow{2}{*}{$\gamma$} & \multicolumn{2}{|c|}{MNIST} & \multicolumn{2}{|c|}{COLOR.} & \multicolumn{2}{|c|}{FASHION} &  \multicolumn{2}{|c}{USPS} \\
        
        \cline{2-9}
       & \multicolumn{2}{|c|}{$\epsilon=$} & \multicolumn{2}{|c|}{$\epsilon=$} & \multicolumn{2}{|c|}{$\epsilon=$} &  \multicolumn{2}{|c}{$\epsilon=$}\\
 
       & $\frac{1}{8}$&$2$  & $\frac{1}{8}$&$2$ & $\frac{1}{8}$&$2$ & $\frac{1}{8}$&$2$  \\
        \midrule
        \textit{$20\%$} & .86&.94 & .48&.73 & .66&.78 & .64&.85 \\
        \textit{$35\%$} & .86&.95 & .47&.74 & .69&.79 & .64&.86 \\
        \textit{exact ($50\%$) } & .88&.95 & .49&.74 & .69&.80 & .64&.87 \\
        \textit{$65\%$} & .84 & .95 & .44&.73 & .70&.79 & .56&.87 \\
        \textit{$80\%$} & .83 & .94 & .31&.74 & .69&.79 & .52&.85 \\
        
        \bottomrule
        \end{tabular}
    \end{subtable}
    
     \caption{Experimental result (test accuracy) under Label-flipping attack for different $\gamma$ in non-i.i.d. setting. }
     \label{tab:ablation_on_gamma_noniid_lf}
\vspace{-0.5cm}
\end{table}

\begin{table}[!ht] 
\footnotesize
    \begin{subtable}[h]{\columnwidth}
        \centering
        \begin{tabular}{c||c|c|c|c|c|c|c|c}
        \toprule
        
        \multirow{2}{*}{$\gamma$} & \multicolumn{2}{|c|}{MNIST} & \multicolumn{2}{|c|}{COLOR.} & \multicolumn{2}{|c|}{FASHION} &  \multicolumn{2}{|c}{USPS} \\
        
        \cline{2-9}
       & \multicolumn{2}{|c|}{$\epsilon=$} & \multicolumn{2}{|c|}{$\epsilon=$} & \multicolumn{2}{|c|}{$\epsilon=$} &  \multicolumn{2}{|c}{$\epsilon=$}\\
 
       & $\frac{1}{8}$&$2$  & $\frac{1}{8}$&$2$ & $\frac{1}{8}$&$2$ & $\frac{1}{8}$&$2$  \\
        \midrule
        \textit{$20\%$} & .86&.95 & .47&.73 & .69&.78 & .64&.85 \\
        \textit{$35\%$} & .87&.96 & .47&.74 & .69&.79 & .63&.86 \\
        \textit{exact ($50\%$) } & .88&.96 & .49&.74 & .69&.80 & .64&.87 \\
        \textit{$65\%$} & .85&.96 & .49&.73 & .69&.79 & .64&.87 \\
        \textit{$80\%$} & .83&.95 & .47&.74 & .69&.79 & .61&.85 \\
        
        \bottomrule
        \end{tabular}
    \end{subtable}
    
     \caption{Experimental result (test accuracy) under Gaussian attack for different $\gamma$ in non-i.i.d. setting. }
     \label{tab:ablation_on_gamma_noniid_gaussian}
\vspace{-0.5cm}
\end{table}

\begin{table}[!ht] 
\footnotesize
    \begin{subtable}[h]{\columnwidth}
        \centering
        \begin{tabular}{c||c|c|c|c|c|c|c|c}
        \toprule
        
        \multirow{2}{*}{$\gamma$} & \multicolumn{2}{|c|}{MNIST} & \multicolumn{2}{|c|}{COLOR.} & \multicolumn{2}{|c|}{FASHION} &  \multicolumn{2}{|c}{USPS} \\
        
        \cline{2-9}
       & \multicolumn{2}{|c|}{$\epsilon=$} & \multicolumn{2}{|c|}{$\epsilon=$} & \multicolumn{2}{|c|}{$\epsilon=$} &  \multicolumn{2}{|c}{$\epsilon=$}\\
 
       & $\frac{1}{8}$&$2$  & $\frac{1}{8}$&$2$ & $\frac{1}{8}$&$2$ & $\frac{1}{8}$&$2$  \\
        \midrule
        \textit{$20\%$} & .85&.95 & .45&.73 & .68&.78 & .64&.85 \\
        \textit{$35\%$} & .87&.96 & .46&.74 & .69&.79 & .64&.86 \\
        \textit{$50\%$ (exact) } & .87&.96 & .49&.74 & .69&.80 & .64&.87 \\
        \textit{$65\%$} & .84 &.96 & .40&.73 & .64&.79 & .50&.87 \\
        \textit{$80\%$} & .80 &.95 & .32&.74 & .58&.79 & .43&.85 \\
        
        \bottomrule
        \end{tabular}
    \end{subtable}
    
     \caption{Experimental result (test accuracy) under  Optimized Local Model Poisoning attack for different $\gamma$ in non-i.i.d. setting. }
     \label{tab:ablation_on_gamma_noniid_local}
\vspace{-0.5cm}
\end{table}

\begin{table}[!ht] 
\small
    \begin{subtable}[h]{\columnwidth}
    
    \centering
    \begin{tabular}{c||c|c|c|c}
    \toprule

    $\epsilon$ & MNIST & COLOR. & FASHION & USPS\\

    \midrule
    \textit{Non-DP} & $.98\pm.003$ & $.80\pm.010$ & $.88\pm.005$ & $.92\pm.004$ \\
    \textit{$2$} & $.96\pm.006$ & $.74\pm.013$ & $.80\pm.007$ & $.87\pm.011$ \\
    \textit{$1$} & $.95\pm.005$ & $.70\pm.017$ & $.79\pm.005$ & $.86\pm.012$ \\
    \textit{$.5$} & $.95\pm.011$ & $.66\pm.026$ & $.78\pm.010$ & $.82\pm.016$ \\
    \textit{$.25$} & $.93\pm.017$ & $.56\pm.023$ & $.75\pm.008$ & $.76\pm.023$ \\
    \textit{$.125$} & $.88\pm.021$  & $.50\pm.041$ & $.70\pm.007$ & $.64\pm.041$ \\
    
    \bottomrule
    \end{tabular}
    \end{subtable}
    
     \caption{Experimental result (test accuracy) on the ``side-effect'' that DP introduces in the i.i.d. setting. }
     \label{tab:ablation_on_dp_iid}
\vspace{-0.5cm}
\end{table}

\begin{table}[!ht] 
\small
    \begin{subtable}[h]{\columnwidth}
    
    \centering
    \begin{tabular}{c||c|c|c|c}
    \toprule

    $\epsilon$ & MNIST & COLOR. & FASHION & USPS\\

    \midrule
    \textit{Non-DP} & $.98\pm.004$ & $.80\pm.013$ & $.88\pm.010$ & $.91\pm.003$ \\
    \textit{$2$} & $.96\pm.009$ & $.74\pm.023$ & $.80\pm.005$ & $.87\pm.013$ \\
    \textit{$1$} & $.95\pm.009$ & $.70\pm.024$ & $.79\pm.010$ & $.86\pm.010$ \\
    \textit{$.5$} & $.95\pm.011$ & $.66\pm.021$ & $.77\pm.012$ & $.82\pm.019$ \\
    \textit{$.25$} & $.93\pm.015$ & $.55\pm.027$ & $.75\pm.017$ & $.76\pm.028$ \\
    \textit{$.125$} & $.88\pm.024$  & $.50\pm.045$ & $.70\pm.012$ & $.63\pm.041$ \\
    
    \bottomrule
    \end{tabular}
    \end{subtable}
    
     \caption{Experimental result (test accuracy) on the ``side-effect'' that DP introduces in the non-i.i.d. setting.}
     \label{tab:ablation_on_dp_noniid}
\vspace{-0.5cm}
\end{table}

\begin{table}[!ht] 
\small
\centering
    
    \begin{subtable}[h]{\columnwidth}
    \centering
    \caption{Under Gaussian attack }
      \begin{tabular}{c|c|c|c|c}
        \toprule
        byz.& MNIST & Colorectal & Fashion & USPS \\
        \midrule
        40\% & $.09\pm.00$ & $.15\pm.05$ & $.10\pm.00$  &$.10\pm.00$  \\
        20\% & $.12\pm.00$ & $.15\pm.05$ & $.13\pm.00$  &$.20\pm.00$  \\
        \bottomrule
      \end{tabular}
    \end{subtable}
    
     \begin{subtable}[h]{\columnwidth}
       \caption{Under Label-flipping attack }
        \centering
            \begin{tabular}{c|c|c|c|c}
            \toprule
            byz.& MNIST & Colorectal & Fashion & USPS \\
            \midrule
            40\% & $.01\pm.00$ & $.07\pm.05$ & $.02\pm.00$  &$.04\pm.01$  \\
            20\% & $.07\pm.04$ & $.09\pm.04$ & $.06\pm.03$  &$.08\pm.03$  \\
            \bottomrule
           \end{tabular}
    \end{subtable}
    
     \begin{subtable}[h]{\columnwidth}
        \centering
      \caption{Under Optimized Local Model Poisoning attack}
            \begin{tabular}{c|c|c|c|c}
            \toprule
            byz.& MNIST & Colorectal & Fashion & USPS \\
            \midrule
            40\% & $.09\pm.00$ & $.12\pm.05$ & $.10\pm.00$  &$.17\pm.00$  \\
            20\% & $.09\pm.00$ & $.12\pm.05$ & $.10\pm.00$  &$.17\pm.00$  \\
            \bottomrule
          \end{tabular}
    \end{subtable}
     
     \caption{Testing accuracy under  different Byzantine attacks at privacy level $\epsilon=2$. In this experiment, the server-own auxiliary data is sampled from KMNIST dataset whose data space $\mathbb{X}$ is different from that of all datasets we use for training.}
     \label{tab:different_data_space_x}

\vspace{-0.5cm}
\end{table}

\subsection{More Experimental Results}\label{appendix:all_results}

\textbf{Ablation study on $\gamma$:} Here, for experiments under other attacks, we present the additional results for the ablation study on $\gamma$ which is treated as a belief instead of the truth. We present results under Gaussian attack in Table \ref{tab:ablation_on_gamma_iid_gaussian}, Optimized Local Model Poisoning attack in Table \ref{tab:ablation_on_gamma_iid_local} under i.i.d. setting. We present results under Label-flipping attack in Table \ref{tab:ablation_on_gamma_noniid_lf}, Gaussian attack in Table \ref{tab:ablation_on_gamma_noniid_gaussian}, Optimized Local Model Poisoning attack in Table \ref{tab:ablation_on_gamma_noniid_local} under the non-i.i.d setting. One notable setting is that we fix the portion of honest workers to be half of all workers. And $\gamma$ is treated to be the server's prior belief that $\gamma n$ workers are honest where $n$ is the total number of workers. ``$50\%$ (exact)'' means that the server's belief is exactly the same as the true situation, {\em i.e.}, half of all workers are honest. All results are obtained by taking the average of three runs with different seeds.

We can see similar results as that in the main body, {\em i.e.}, when the belief is too radical and the privacy requirement is extreme, the utility suffers, except for the Gaussian attack where we almost still have robustness, this is because Gaussian attack purely injects noise to the aggregation, the utility is not significantly affected.

\textbf{Abalation study on DP:} 
To see the ``side-effect'' that DP introduces, we summarise our experimental results in Table \ref{tab:ablation_on_dp_iid} for the i.i.d. case and Table \ref{tab:ablation_on_dp_noniid} for the non-i.i.d. case, respectively. For a fair comparison, all experiments are done with the same hyperparameter setup. We use ``Non-DP'' to stand for the case where no DP is applied. We can see from those results that by introducing DP, the utility drops compared with the non-DP case. This is due to the noise introduced by the DP-SGD protocol. The more private (more noise is added), the worse utility we will get.

Note that our Byzantine robustness is built on our refactored DP protocol, so for the non-DP case, we are only able to run the experiment without Byzantine robustness and accordingly, without Byzantine attack (or we do not have any utility otherwise). And for that specific scenario, there is existing work identifying and improving the utility of DP learning compared with non-DP. We refer the reader to them \cite{asoodeh2021three, mironov2019r, gopi2021numerical, zheng2020sharp, wang2019subsampled, zhu2019poission} as they are more involved.

\textbf{When the data space of auxiliary data is different from it of the local-hold data:} Recall that previously we assumed the sever-own auxiliary data follows the same distribution 
as that of the local-hold data $\mathbb{X}$. 
Here we conduct experiments for the case that server-own data is sampled from a different data space $\mathbb{X}'$. 
The experiments are carried out as follows, for the auxiliary data, we sample it from the KMNIST \cite{clanuwat2018deepkmnist} dataset and for the rest of the experimental settings, we keep them identical to those mentioned in our experiment section.

Results are presented in Table \ref{tab:different_data_space_x}. We can see that overall, the training does not lead to  useful utility if we sample the server-own auxiliary data from a different data space. We can see that when under the  Gaussian attack, the training performs better than random guessing if $20\%$ workers are Byzantine, this is because a Gaussian attack purely injects noise and such noise is not too detrimental. In contrast, when under the Label-flipping attack, the model performs even worse than random guessing, this is because the Label-flipping attack tries to let the model predict wrongly for each label.

Based on such observation, we claim that it is necessary  to assume the serve-own auxiliary data is sampled from the same data space as that of the local-hold data. However, the above negative results are possibly due to that KMNIST ``looks'' too different from the local-hold data. If we can find some dataset that ``looks'' more similar, positive results may be obtained. We leave such work as a future study.

\textbf{Experimental results for more stringent cases ($95\%$ and $99\%$ workers are Byzantine):} We previously present experimental results for the case where $90\%$ workers are Byzantine in the main context. Here we give results for more stringent cases to test our protocol's robustness. Note that simulating stringent cases requires simulating more workers which incurs more computational burden. For instance, simulating the $99\%$ case will be 10+ times heavier than that of the $90\%$ case.

Results are presented in Figure 
\ref{fig:dp_byz_95_label_iid},  
\ref{fig:dp_byz_95_gaussian_iid},
\ref{fig:dp_byz_95_local_iid},
\ref{fig:dp_byz_99_label_iid},  
\ref{fig:dp_byz_99_gaussian_iid},
\ref{fig:dp_byz_99_local_iid} (\textit{i.i.d.} setting for $95\%$ and $99\%$),
\ref{fig:dp_byz_95_label_noniid},  
\ref{fig:dp_byz_95_gaussian_noniid},
\ref{fig:dp_byz_95_local_noniid},
\ref{fig:dp_byz_99_label_noniid},  
\ref{fig:dp_byz_99_gaussian_noniid},
\ref{fig:dp_byz_99_local_noniid} (\textit{non-i.i.d.} setting for $95\%$ and $99\%$). 
Note that we do not have results for the Colorectal dataset, this is because we encountered CUDA memory overflow when running on that dataset. As we can see from these results,  our protocol is still robust when $\epsilon=2$. And for stronger privacy requirements, we can see the utility drops dramatically. It is unsurprising that as the number of Byzantine workers increases and the DP noise increases, the training will tend to aggregate malicious uploads, hence hurting the utility. We can also see that we still have good robustness even when the attack is Optimized Local Model Poisoning, this shows that when the number of malicious workers is large, our protocol can better defend the Optimized Local Model Poisoning attack compared with the other two attacks.

\begin{figure*}[!ht] 
    \centering

     \includegraphics[width=0.9\linewidth]{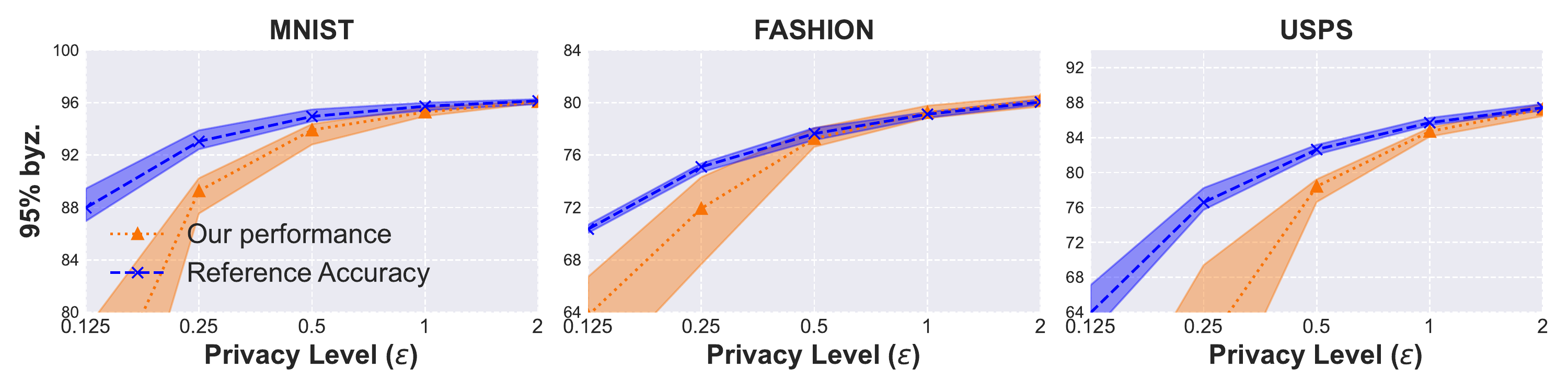}

    \caption{ Byzantine-resilient performance (testing accuracy) in i.i.d. setting when $95\%$ workers are Label-flipping Byzantine attackers. }
    \label{fig:dp_byz_95_label_iid}
\end{figure*}

\begin{figure*}[!ht] 
    \centering

     \includegraphics[width=0.9\linewidth]{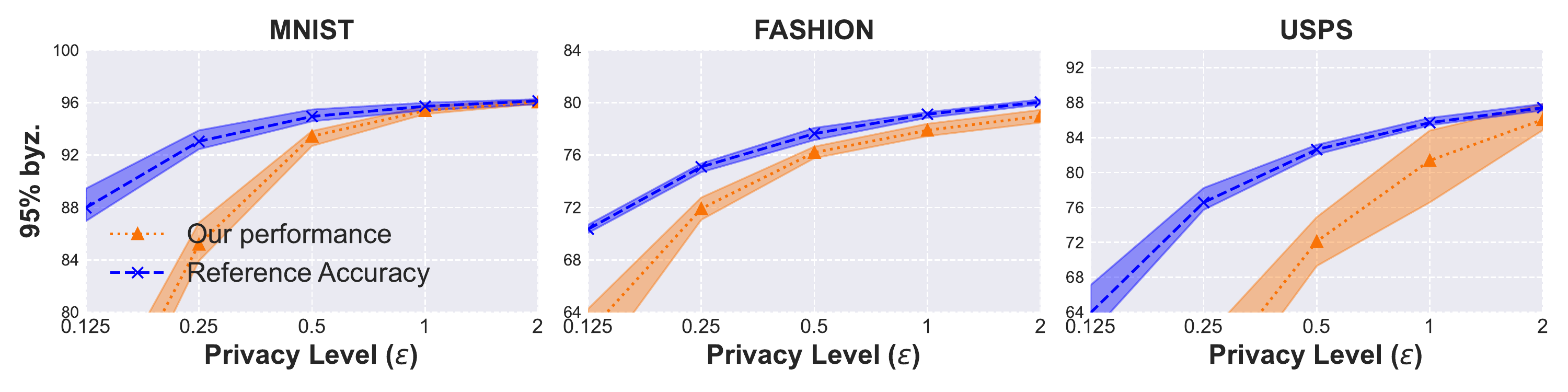}

    \caption{ Byzantine-resilient performance (testing accuracy) in i.i.d. setting when $95\%$ workers are Gaussian Byzantine attackers. }
    \label{fig:dp_byz_95_gaussian_iid}
\end{figure*}

\begin{figure*}[!ht] 
    \centering

     \includegraphics[width=0.9\linewidth]{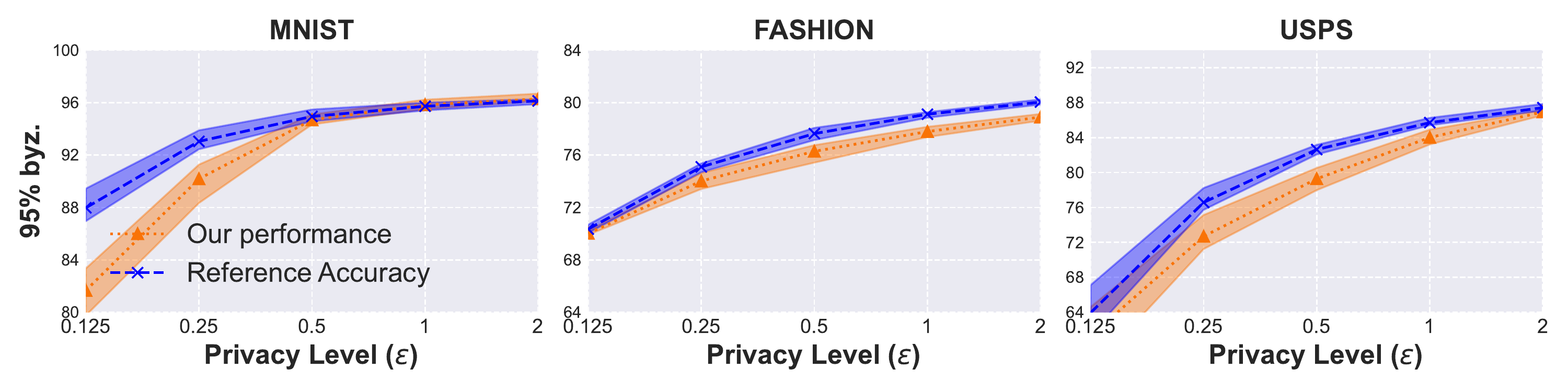}

    \caption{ Byzantine-resilient performance (testing accuracy) in i.i.d. setting when $95\%$ workers are Optimized Local Model Poisoning  Byzantine attackers. }
    \label{fig:dp_byz_95_local_iid}
\end{figure*}

\begin{figure*}[!ht] 
    \centering

     \includegraphics[width=0.9\linewidth]{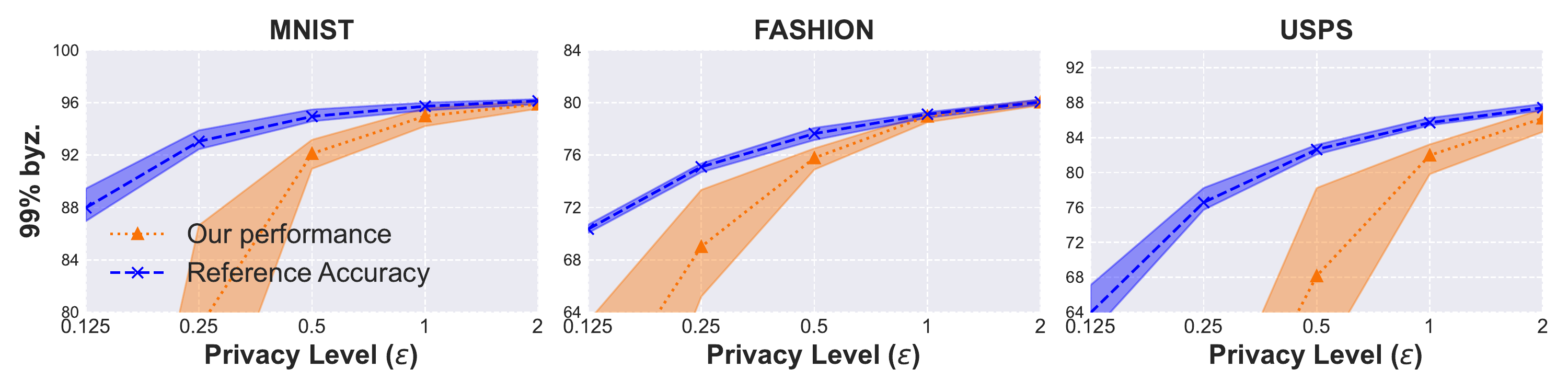}

    \caption{ Byzantine-resilient performance (testing accuracy) in i.i.d. setting when $99\%$ workers are Label-flipping Byzantine attackers. }
    \label{fig:dp_byz_99_label_iid}
\end{figure*}

\begin{figure*}[!ht] 
    \centering

     \includegraphics[width=0.9\linewidth]{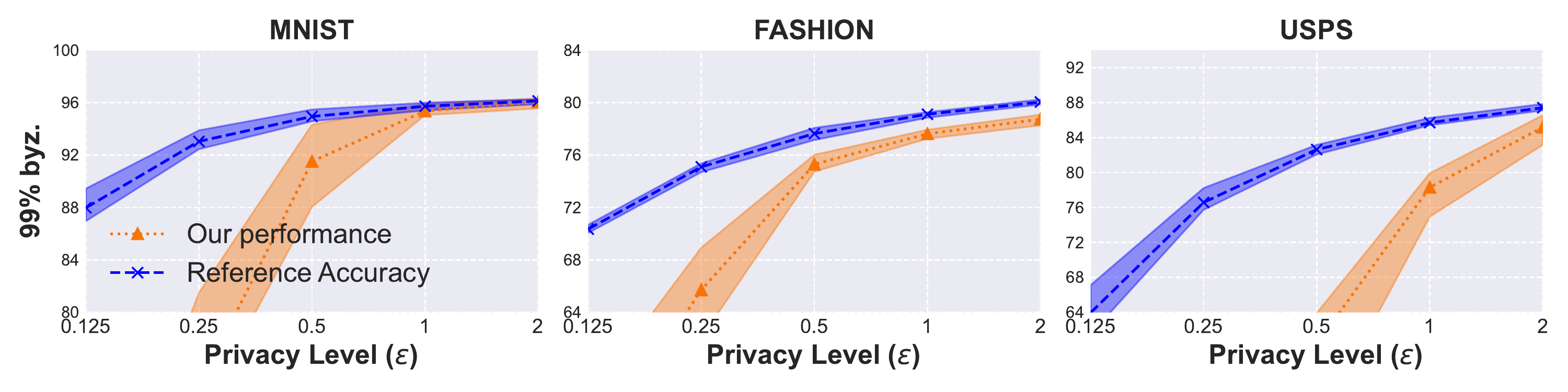}

    \caption{ Byzantine-resilient performance (testing accuracy) in i.i.d. setting when $99\%$ workers are Gaussian Byzantine attackers. }
    \label{fig:dp_byz_99_gaussian_iid}
\end{figure*}

\begin{figure*}[!ht] 
    \centering

     \includegraphics[width=0.9\linewidth]{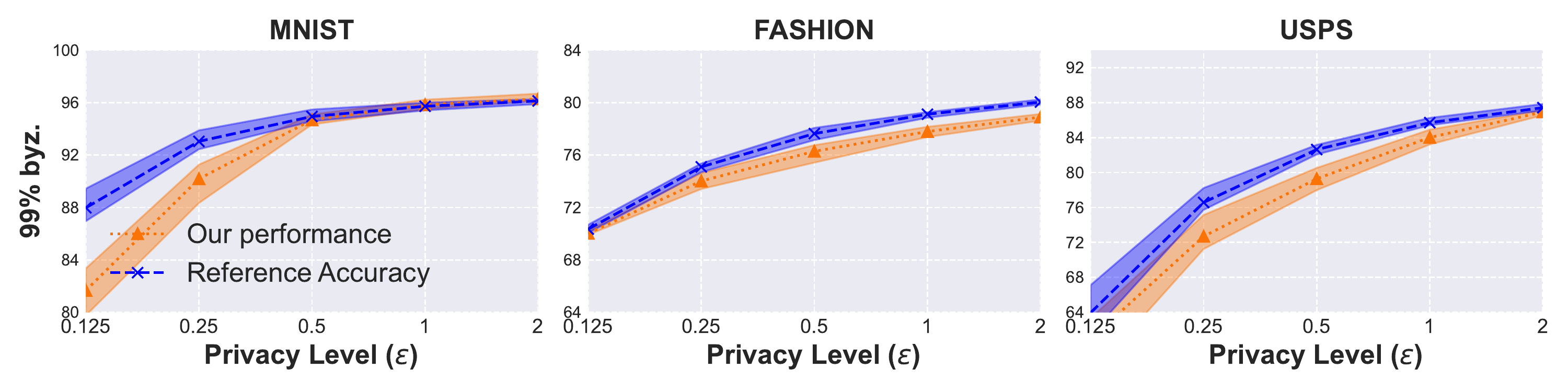}

    \caption{ Byzantine-resilient performance (testing accuracy) in i.i.d. setting when $99\%$ workers are Optimized Local Model Poisoning  Byzantine attackers. }
    \label{fig:dp_byz_99_local_iid}
\end{figure*}

\begin{figure*}[!ht] 
    \centering

     \includegraphics[width=0.9\linewidth]{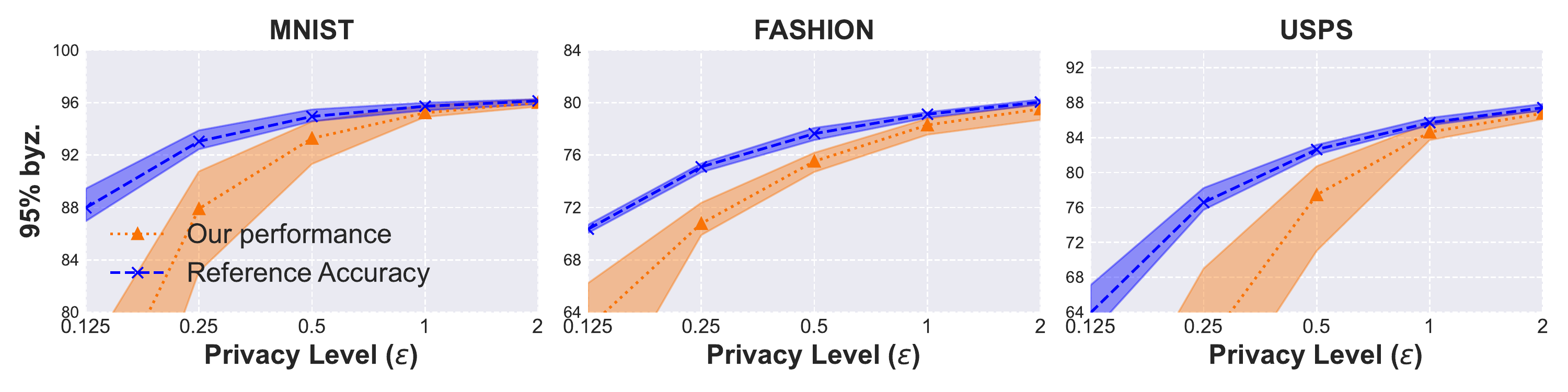}

    \caption{ Byzantine-resilient performance (testing accuracy) in non-i.i.d. setting when $95\%$ workers are Label-flipping Byzantine attackers. }
    \label{fig:dp_byz_95_label_noniid}
\end{figure*}

\begin{figure*}[!ht] 
    \centering

     \includegraphics[width=0.9\linewidth]{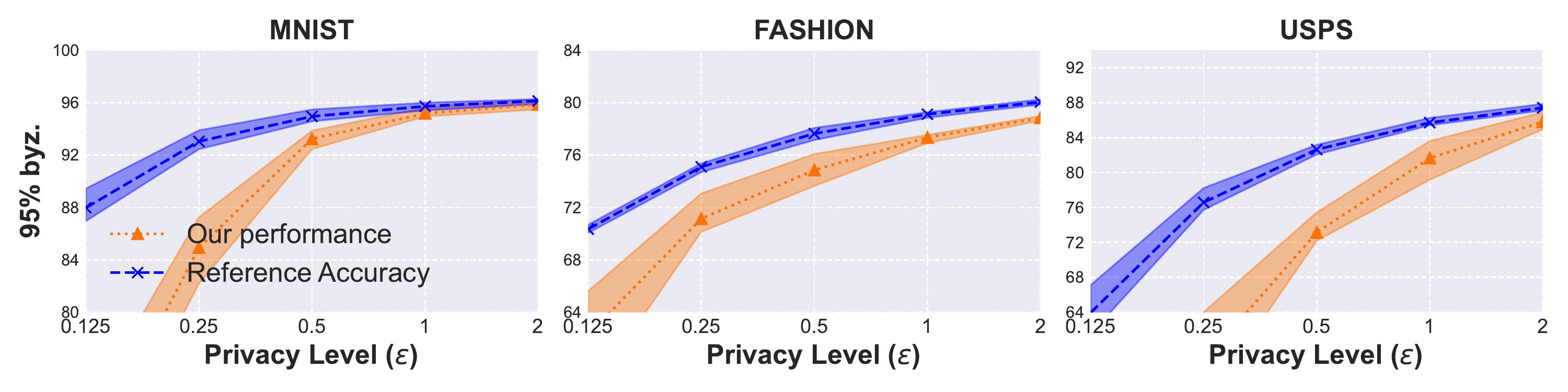}

    \caption{ Byzantine-resilient performance (testing accuracy) in non-i.i.d. setting when $95\%$ workers are Gaussian Byzantine attackers. }
    \label{fig:dp_byz_95_gaussian_noniid}
\end{figure*}

\begin{figure*}[!ht] 
    \centering

     \includegraphics[width=0.9\linewidth]{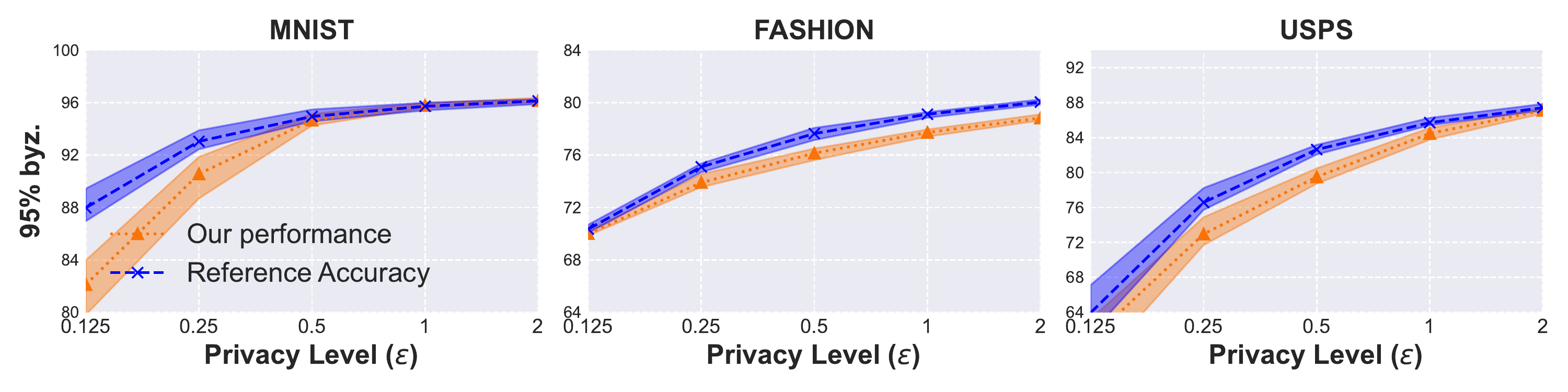}

    \caption{ Byzantine-resilient performance (testing accuracy) in non-i.i.d. setting when $95\%$ workers are Optimized Local Model Poisoning  Byzantine attackers. }
    \label{fig:dp_byz_95_local_noniid}
\end{figure*}

\begin{figure*}[!ht] 
    \centering

     \includegraphics[width=0.9\linewidth]{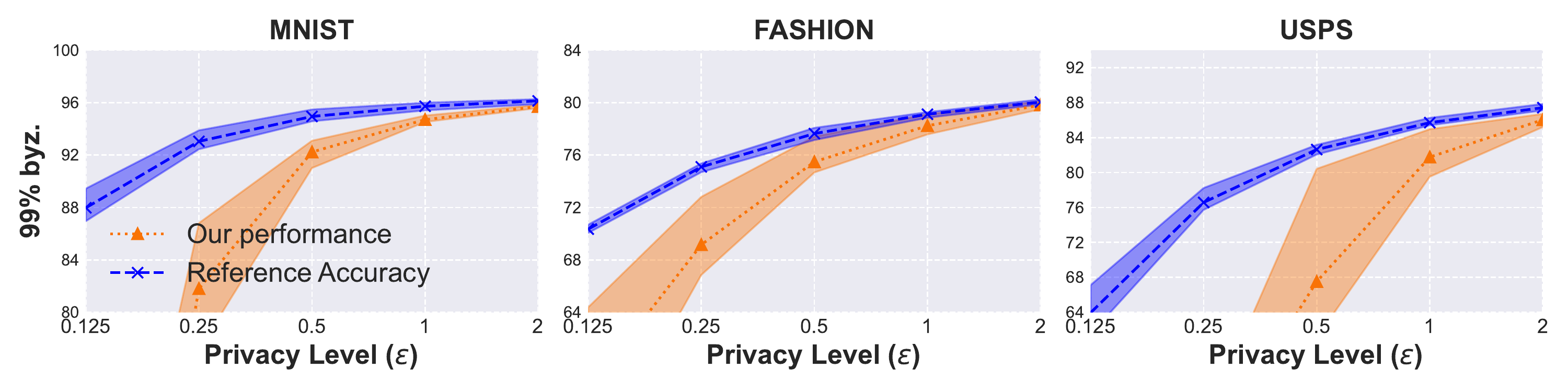}

    \caption{ Byzantine-resilient performance (testing accuracy) in non-i.i.d. setting when $99\%$ workers are Label-flipping Byzantine attackers. }
    \label{fig:dp_byz_99_label_noniid}
\end{figure*}

\begin{figure*}[!ht] 
    \centering

     \includegraphics[width=0.9\linewidth]{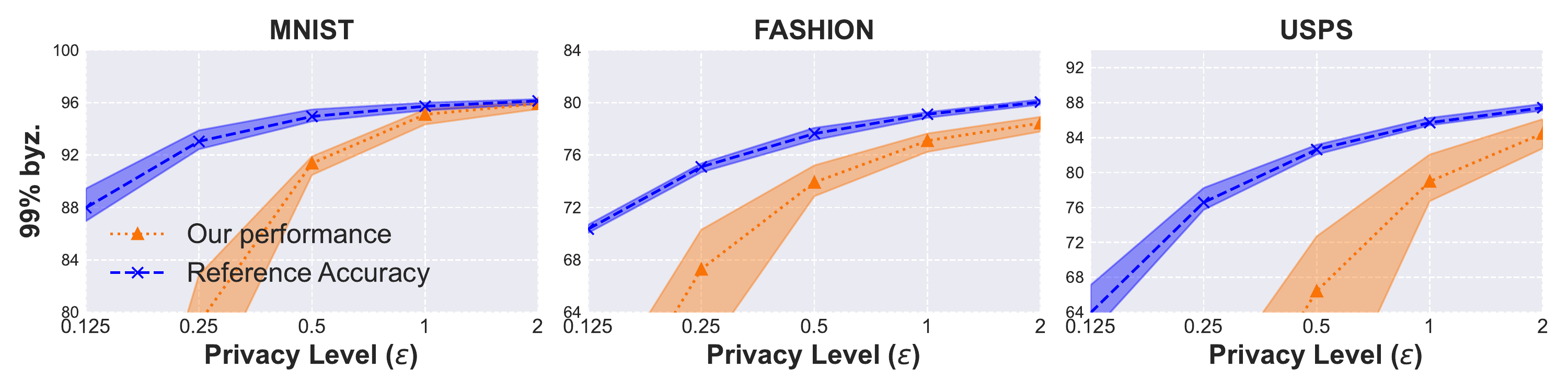}

    \caption{ Byzantine-resilient performance (testing accuracy) in non-i.i.d. setting when $99\%$ workers are Gaussian Byzantine attackers. }
    \label{fig:dp_byz_99_gaussian_noniid}
\end{figure*}

\begin{figure*}[!ht] 
    \centering

     \includegraphics[width=0.9\linewidth]{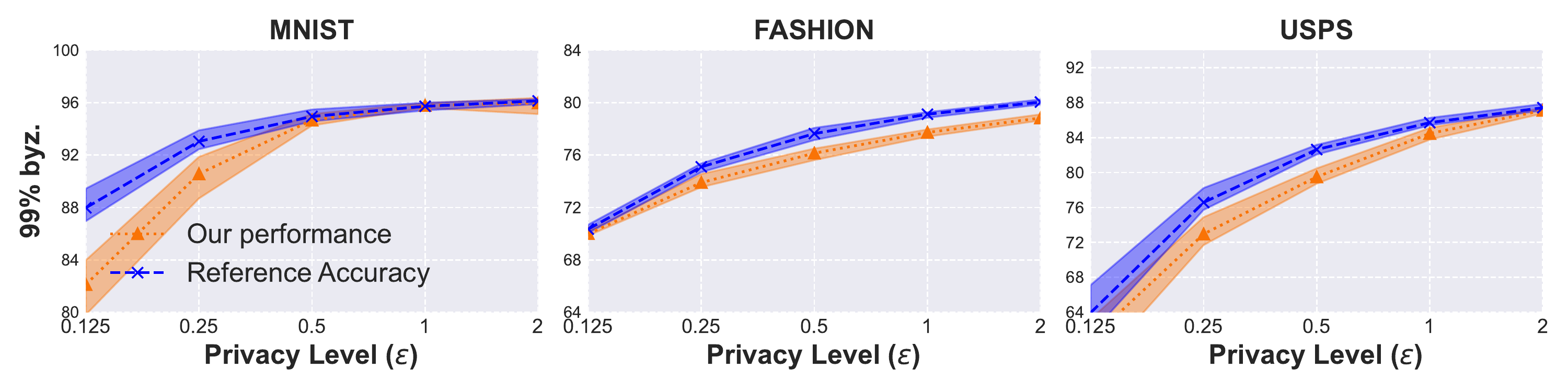}

    \caption{ Byzantine-resilient performance (testing accuracy) in non-i.i.d. setting when $99\%$ workers are Optimized Local Model Poisoning  Byzantine attackers. }
    \label{fig:dp_byz_99_local_noniid}
\end{figure*}

\textbf{Experimental results for i.i.d. setting:}
We present the experimental results on general Byzantine resilience, resilience facing $90\%$ attackers, efficient hyper-parameter tuning in Figure 
\ref{fig:dp_byz_iid_gaussian}, \ref{fig:dp_byz_iid_90_gaussian}, \ref{fig:lr_tuning_s1_iid_gaussian} (Gaussian attack), \ref{fig:dp_byz_iid_local}, \ref{fig:dp_byz_iid_90_local}, \ref{fig:lr_tuning_s1_iid_local} (Optimized Local Model Poisoning attack). As we have observed and analyzed before, similar results can also be observed when our learning protocol is under other attacks.

\textbf{Experimental results for non-i.i.d. setting:}
We present all experiments in non-i.i.d. settings as counterparts to previous results. Results are summarise in Figure \ref{fig:dp_byz_label_noniid}, \ref{fig:dp_byz_90_label_noniid}, \ref{fig:lr_tuning_s1_label_noniid} (Label-flipping attack), \ref{fig:dp_byz_gaussian_noniid}, \ref{fig:dp_byz_90_gaussian_noniid}, \ref{fig:lr_tuning_s1_gaussian_noniid} (Gaussian attack), 
\ref{fig:dp_byz_local_noniid}, \ref{fig:dp_byz_90_local_noniid}, \ref{fig:lr_tuning_s1_local_noniid} (Optimized Local Model Poisoning attack). As we have observed in the i.i.d. case, similar results can also be observed when our learning protocol is under various attacks in the non-i.i.d. setting.

\textbf{Resilience against adaptive attacker:} 
We present the remaining results in Figure \ref{fig:smart_attacker_label_iid}, \ref{fig:smart_attacker_gaussian_iid}, \ref{fig:smart_attacker_local_iid} (i.i.d. setting),
\ref{fig:smart_attacker_label_noniid}, \ref{fig:smart_attacker_gaussian_noniid}, \ref{fig:smart_attacker_local_noniid} (non-i.i.d. setting). All experiments are conducted when there are $60\%$ Byzantine attackers. As we have observed before, our learning protocol is also resilient to adaptive attackers under various attack instantiations. 
\begin{figure*}[!ht] 
    \centering

     \includegraphics[width=0.9\linewidth]{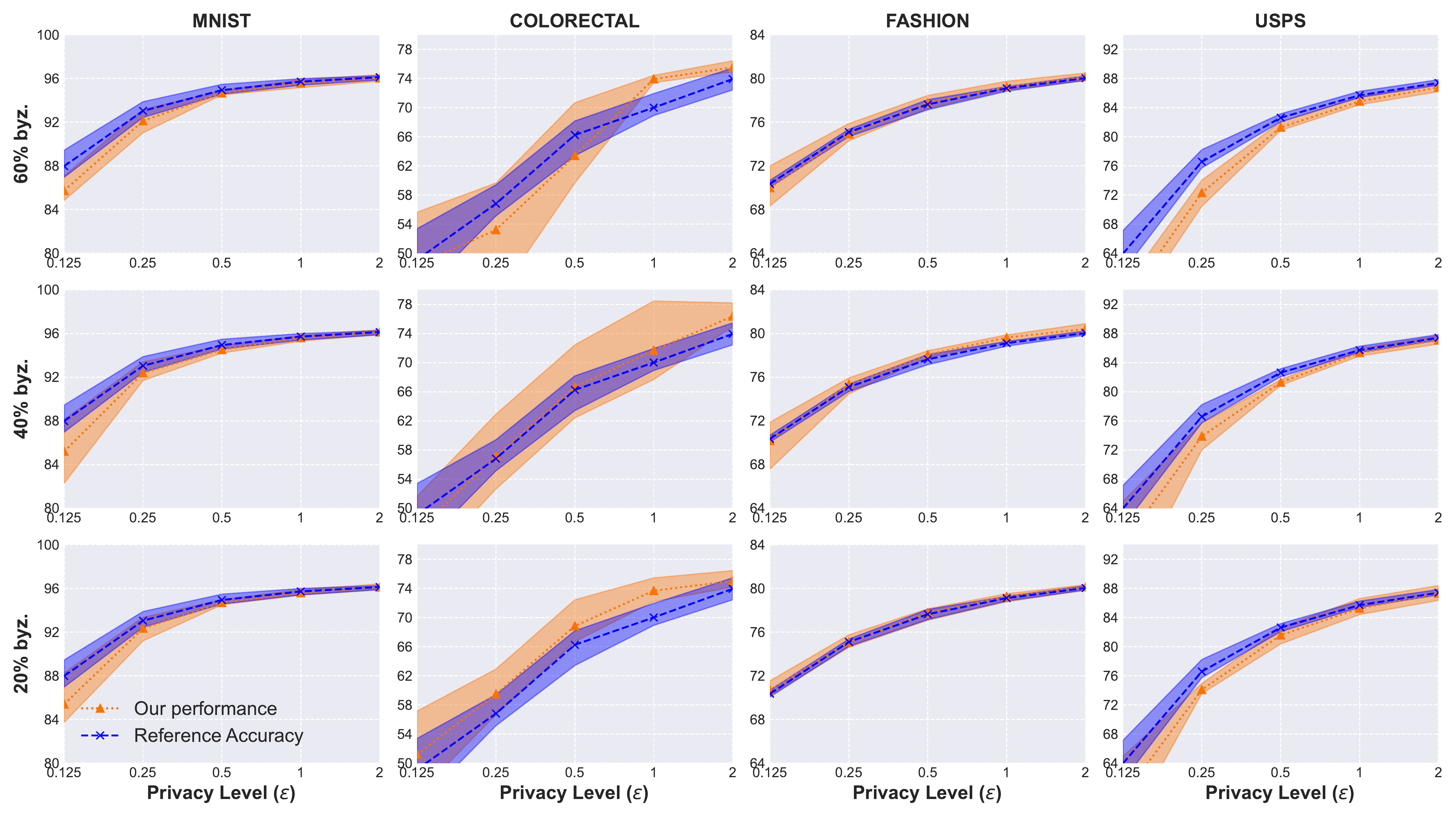}

    \caption{Byzantine-resilient performance (testing accuracy) under Gaussian attack in i.i.d. setting. The experiment is conducted under 3 different attacking levels ($20\%,40\%,60\%$ of total workers are Byzantine).}
    \label{fig:dp_byz_iid_gaussian}
\end{figure*}

\begin{figure*}[!ht] 
    \centering

     \includegraphics[width=0.9\linewidth]{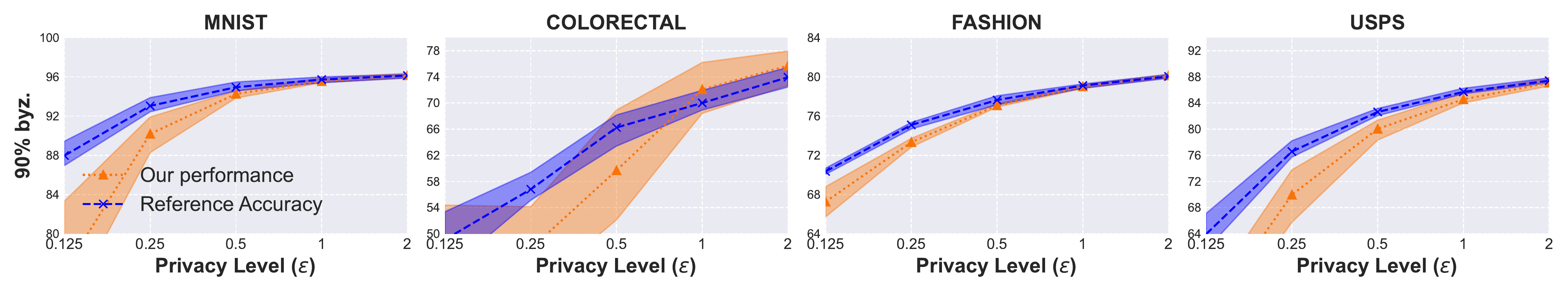}

    \caption{Byzantine-resilient performance (testing accuracy) in i.i.d. setting when $90\%$ workers are Gaussian Byzantine attackers.}
    \label{fig:dp_byz_iid_90_gaussian}
\end{figure*}

\begin{figure*}[!ht] 
    \centering

     \includegraphics[width=0.9\linewidth]{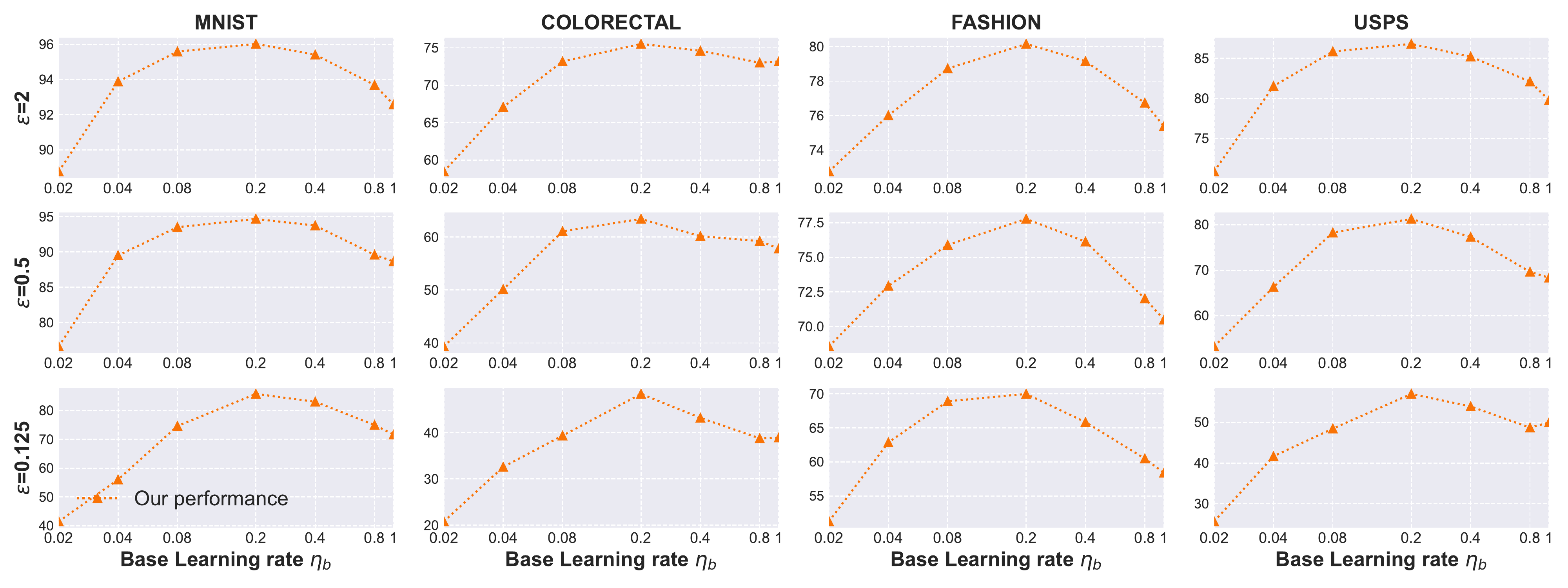}

    \caption{ Our hyper-parameter tuning result under i.i.d. setting when facing $60\%$ Gaussian attackers.}
    \label{fig:lr_tuning_s1_iid_gaussian}
	
\end{figure*}

\begin{figure*}[!ht] 
    \centering

     \includegraphics[width=0.9\linewidth]{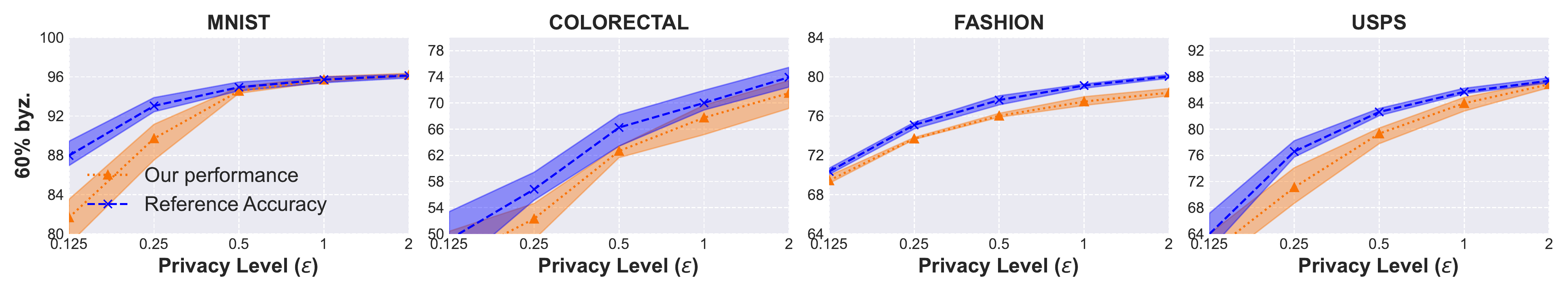}

    \caption{Byzantine-resilient performance (testing accuracy) under Optimized Local Model Poisoning attacks in i.i.d. setting. The experiment is conducted where there are $60\%$ attackers.}
    \label{fig:dp_byz_iid_local}
\end{figure*}

\begin{figure*}[!ht] 
    \centering

     \includegraphics[width=0.9\linewidth]{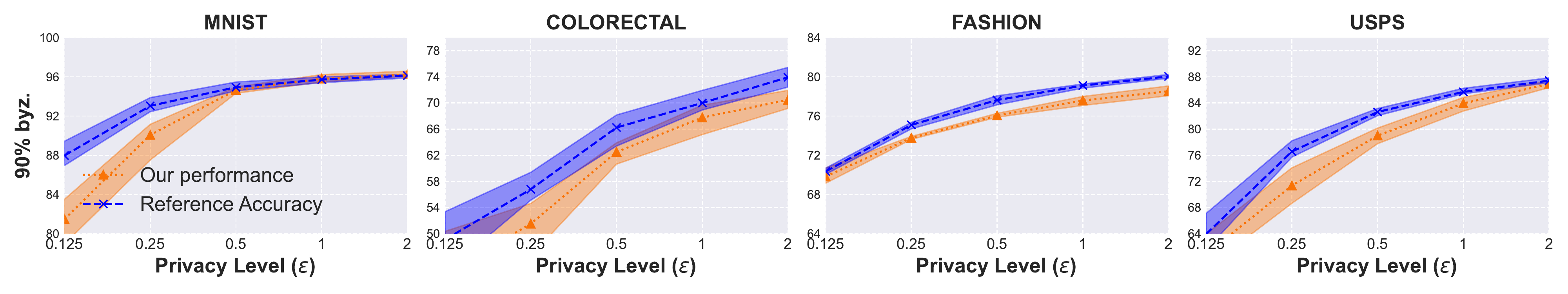}

    \caption{Byzantine-resilient performance (testing accuracy) in i.i.d. setting when $90\%$ workers are Optimized Local Model Poisoning Byzantine attackers.}
    \label{fig:dp_byz_iid_90_local}
\end{figure*}

\begin{figure*}[!ht] 
    \centering

     \includegraphics[width=0.9\linewidth]{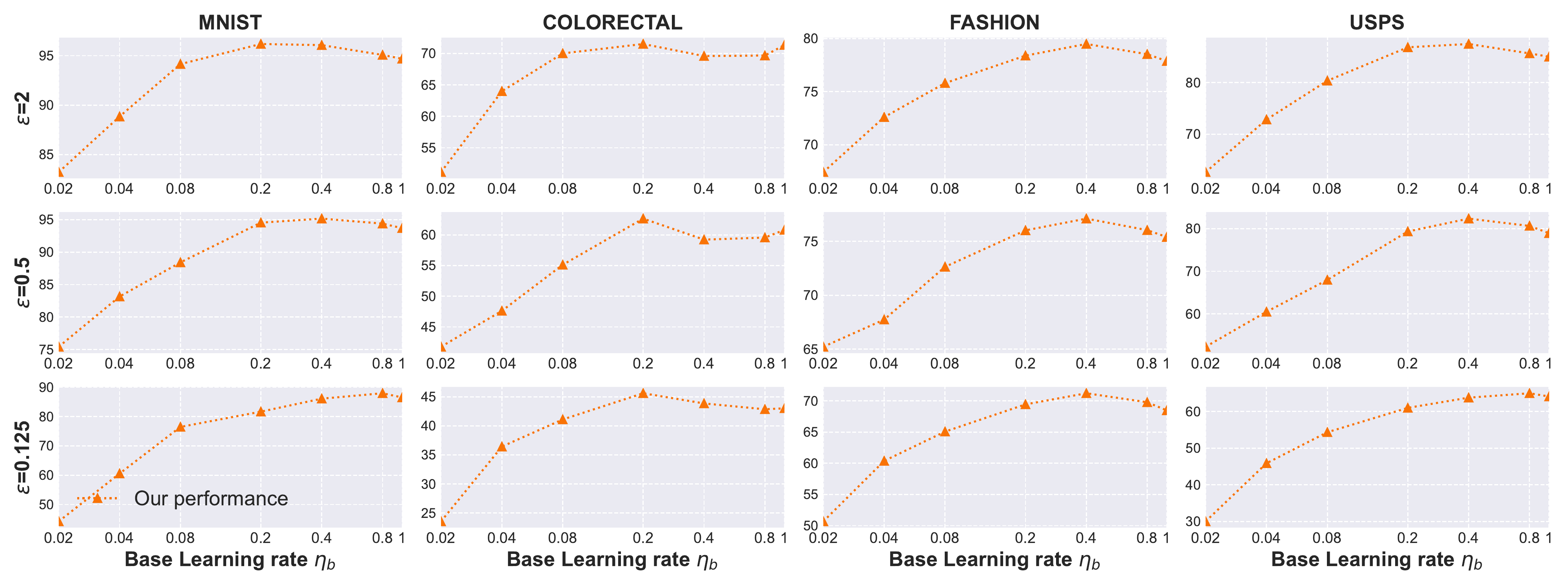}

    \caption{ Our hyper-parameter tuning result under i.i.d. setting when facing $60\%$ Optimized Local Model Poisoning attackers.}
    \label{fig:lr_tuning_s1_iid_local}
	
\end{figure*}


\begin{figure*}[!ht] 
    \centering

     \includegraphics[width=0.9\linewidth]{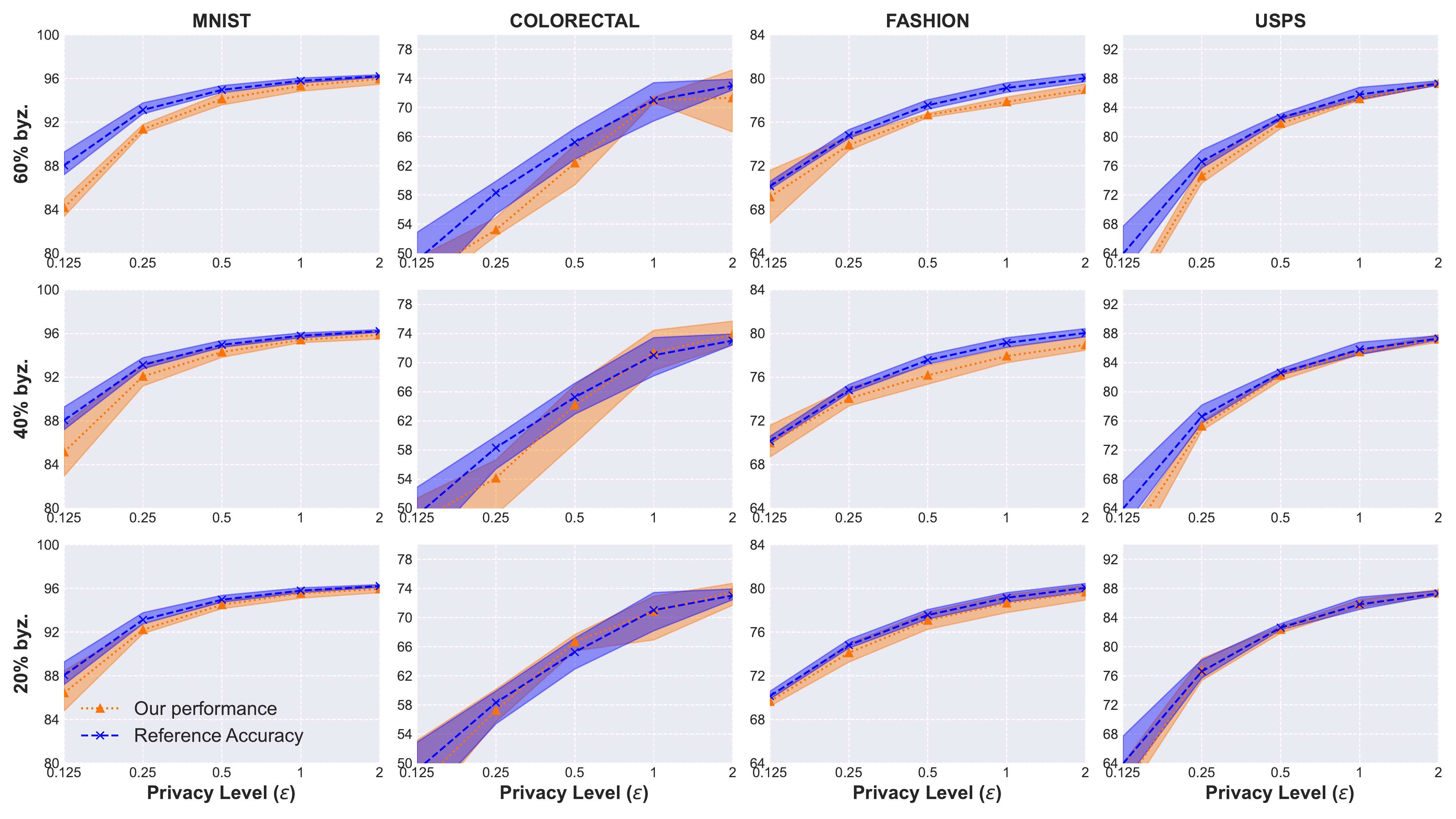}

    \caption{Byzantine-resilient performance (testing accuracy) under Label-flipping attack in non-i.i.d. setting. The experiment is conducted under 3 different attacking levels ($20\%,40\%,60\%$ of total workers are Byzantine).}
    \label{fig:dp_byz_label_noniid}
\end{figure*}

\begin{figure*}[!ht] 
    \centering

     \includegraphics[width=0.9\linewidth]{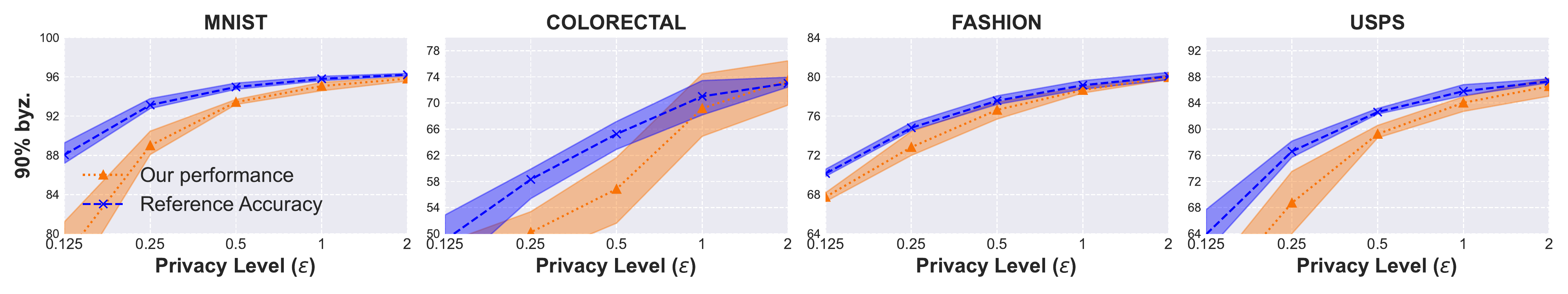}

    \caption{Byzantine-resilient performance (testing accuracy) in non-i.i.d. setting when $90\%$ workers are Label-flipping Byzantine attackers.}
    \label{fig:dp_byz_90_label_noniid}
\end{figure*}

\begin{figure*}[!ht] 
    \centering

     \includegraphics[width=0.9\linewidth]{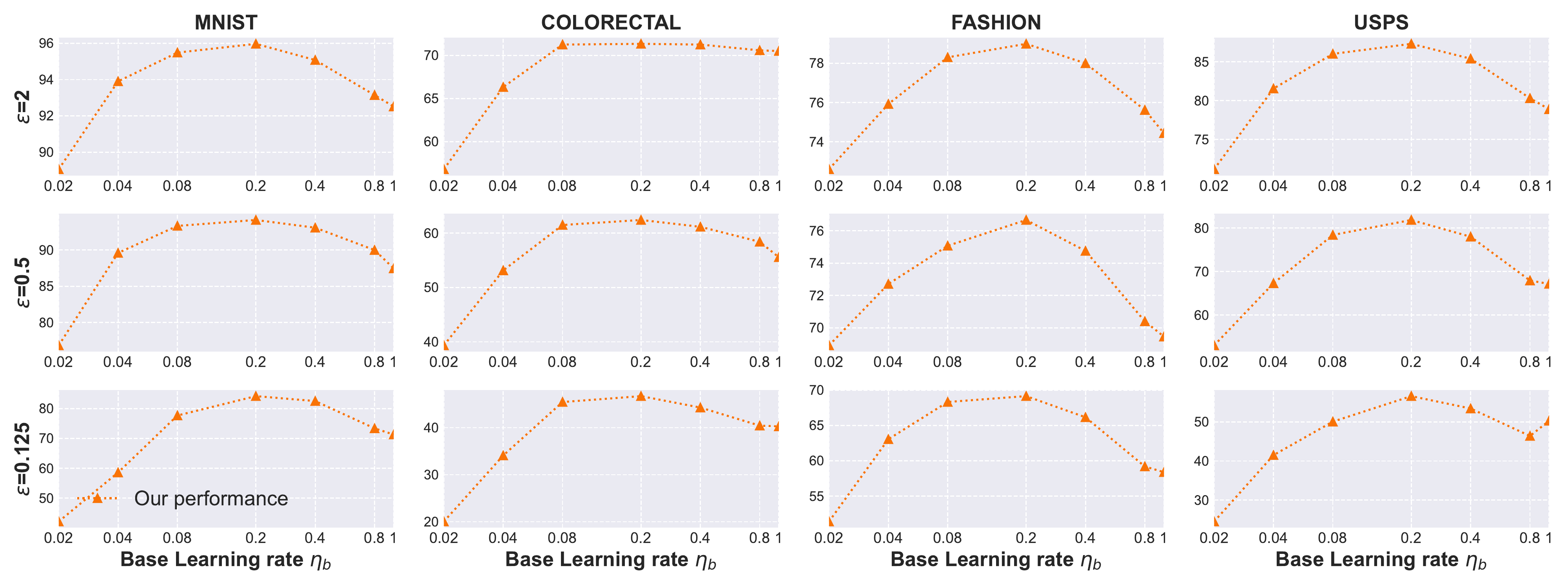}
   
    \caption{ Our hyper-parameter tuning result under non-i.i.d. setting when facing $60\%$ Label-flipping attackers.}
    \label{fig:lr_tuning_s1_label_noniid}
	
\end{figure*}

\begin{figure*}[!ht] 
    \centering

     \includegraphics[width=0.9\linewidth]{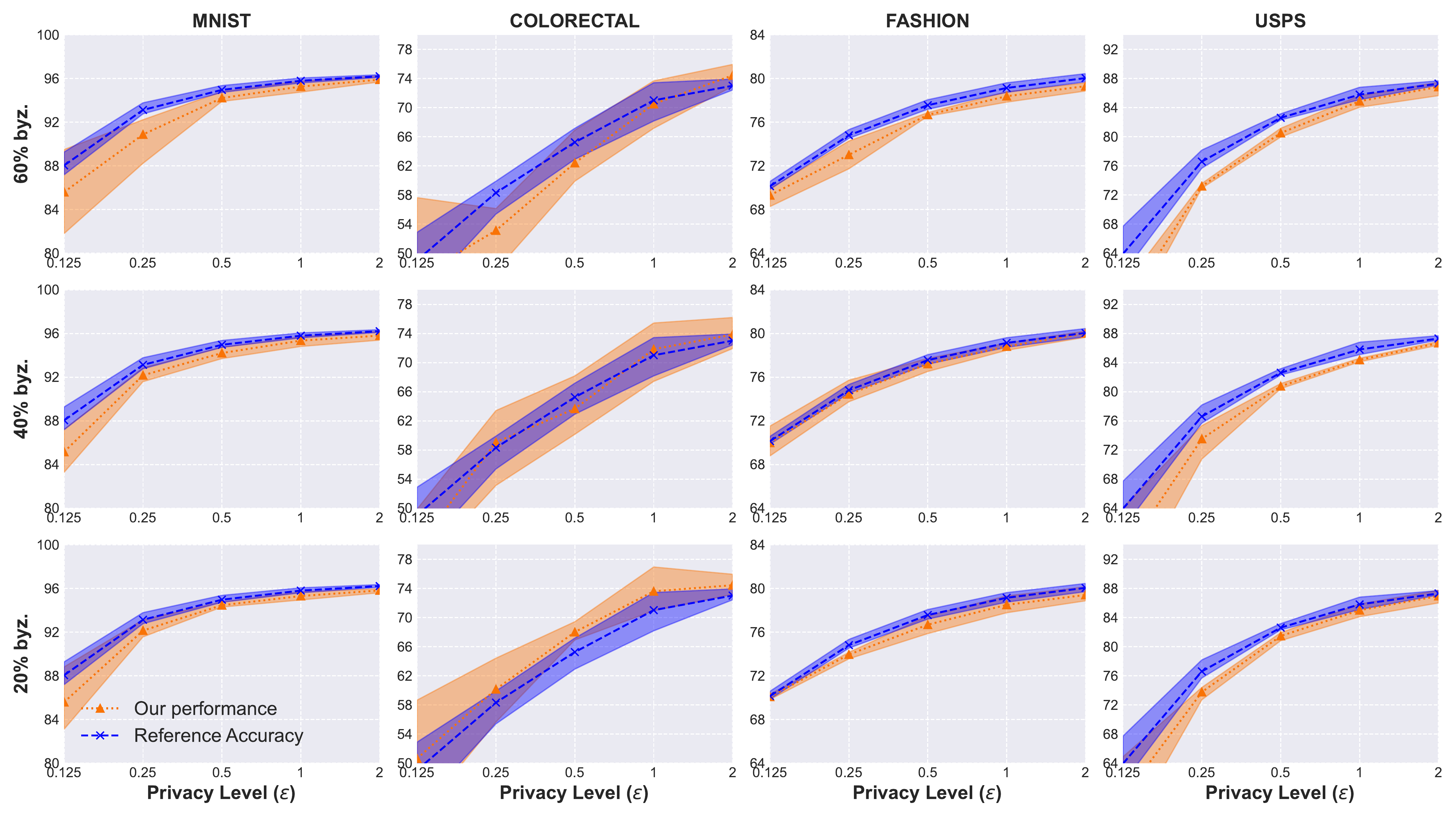}

    \caption{Byzantine-resilient performance (testing accuracy) under Gaussian attack in non-i.i.d. setting. The experiment is conducted under 3 different attacking levels ($20\%,40\%,60\%$ of total workers are Byzantine).}
    
    \label{fig:dp_byz_gaussian_noniid}
\end{figure*}

\begin{figure*}[!ht] 
    \centering
   
     \includegraphics[width=0.9\linewidth]{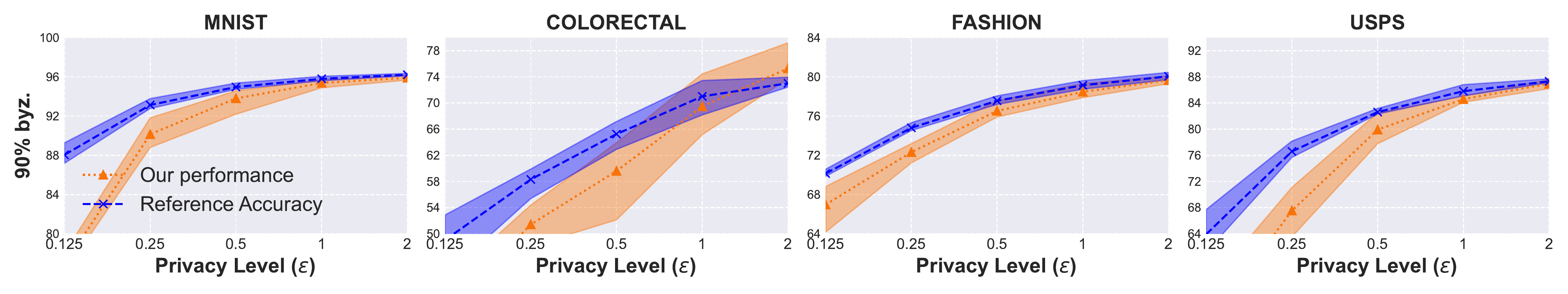}

    \caption{Byzantine-resilient performance (testing accuracy) in non-i.i.d. setting when $90\%$ workers are Gaussian Byzantine attackers.}
    \label{fig:dp_byz_90_gaussian_noniid}
\end{figure*}

\begin{figure*}[!ht] 
    \centering

     \includegraphics[width=0.9\linewidth]{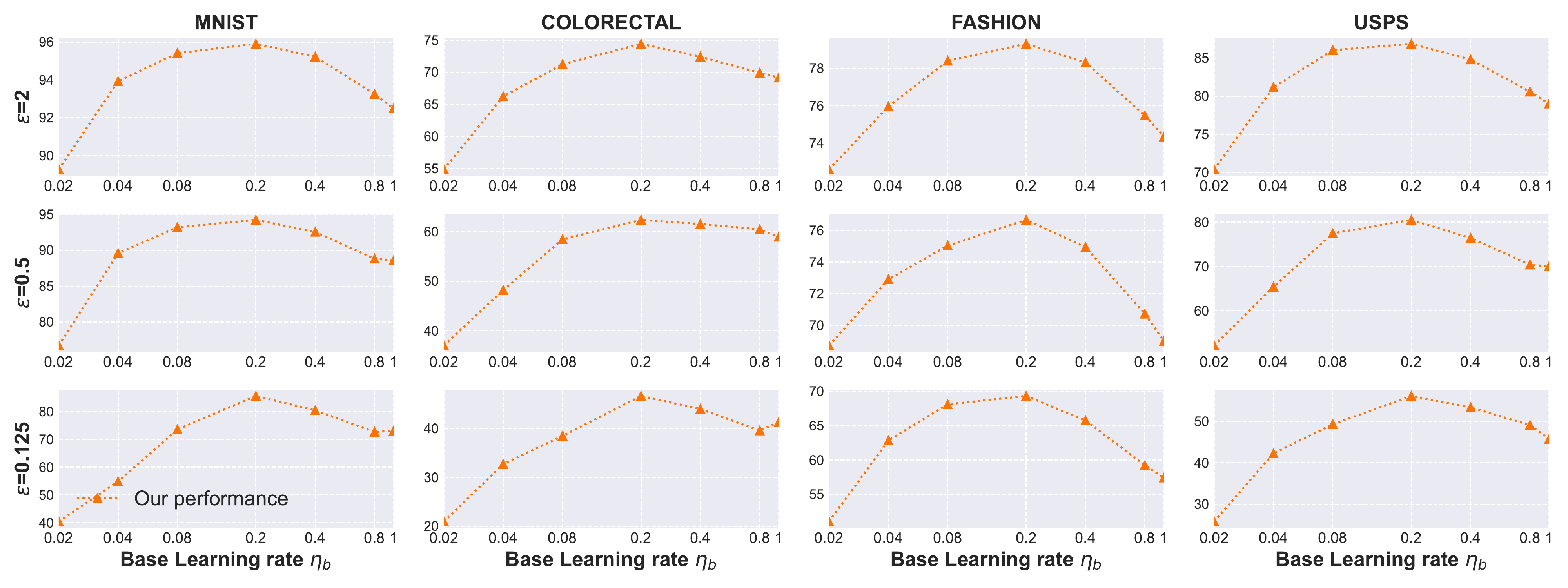}

    \caption{ Our hyper-parameter tuning result under non-i.i.d. setting when facing $60\%$ Gaussian attackers.}
    \label{fig:lr_tuning_s1_gaussian_noniid}
	
\end{figure*}

\begin{figure*}[!ht] 
    \centering

     \includegraphics[width=0.9\linewidth]{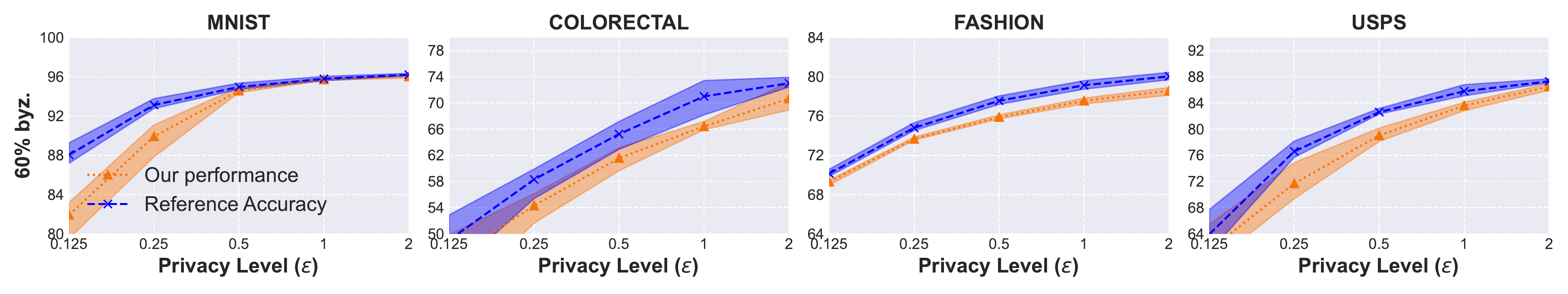}

    \caption{Byzantine-resilient performance (testing accuracy) under Optimized Local Model Poisoning attack in non-i.i.d. setting. The experiment is conducted where there are $60\%$ attackers.}
    \label{fig:dp_byz_local_noniid}
\end{figure*}

\begin{figure*}[!ht] 
    \centering

     \includegraphics[width=0.9\linewidth]{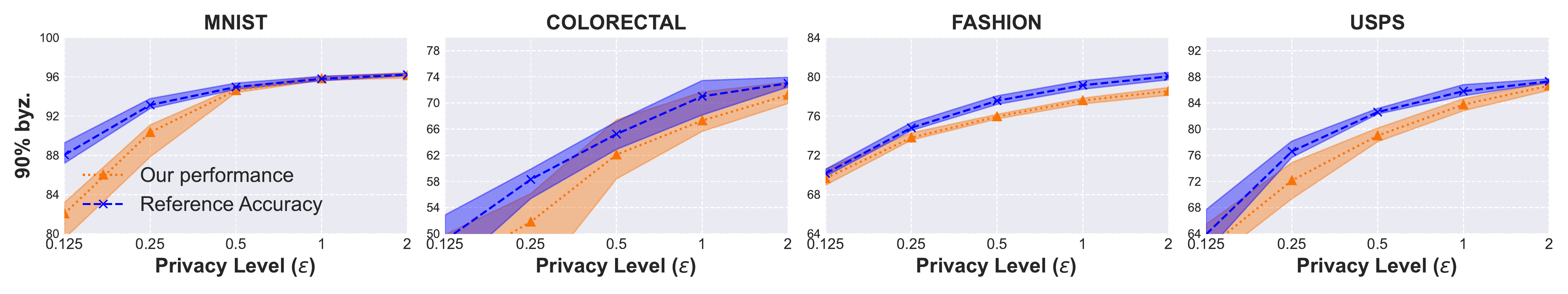}

    \caption{Byzantine-resilient performance (testing accuracy) in non-i.i.d. setting when $90\%$ workers are Optimized Local Model Poisoning Byzantine attackers.}
    \label{fig:dp_byz_90_local_noniid}
\end{figure*}

\begin{figure*}[!ht] 
    \centering

     \includegraphics[width=0.9\linewidth]{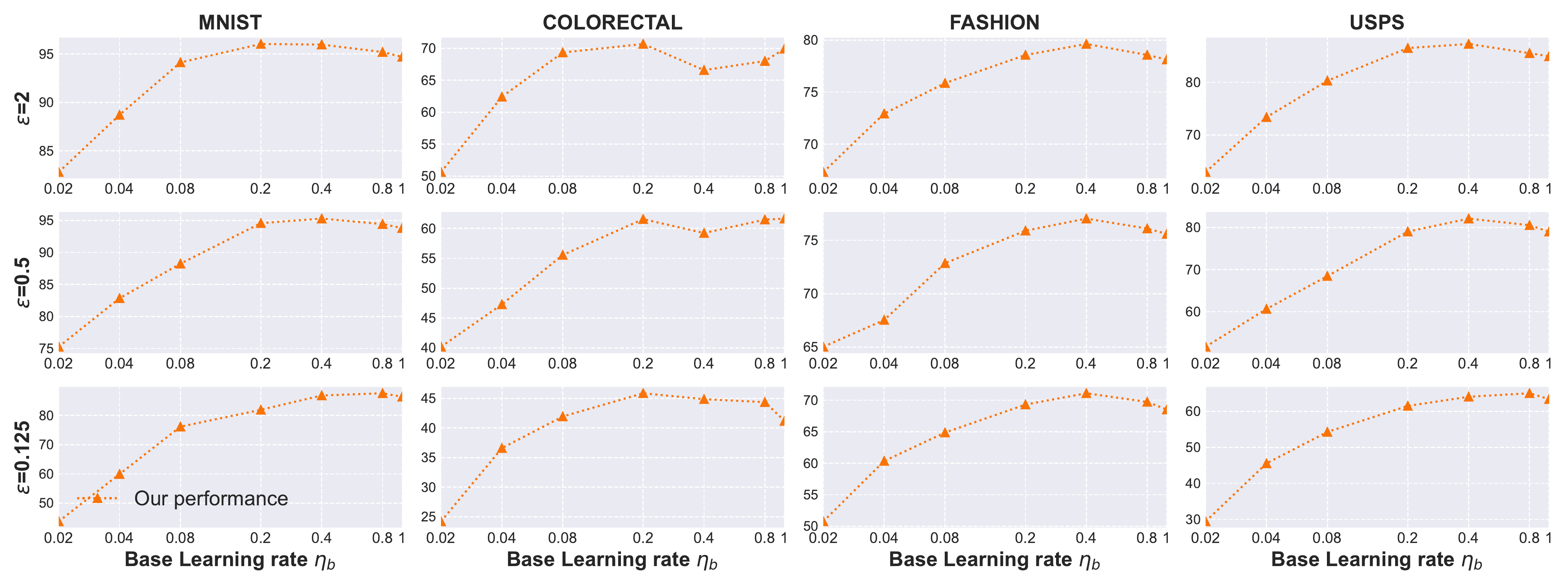}

    \caption{ Our hyper-parameter tuning result under non-i.i.d. setting when facing $60\%$ Optimized Local Model Poisoning attackers.}
    \label{fig:lr_tuning_s1_local_noniid}
	
\end{figure*}

\begin{figure*}[!ht] 
    \centering

     \includegraphics[width=0.9\linewidth]{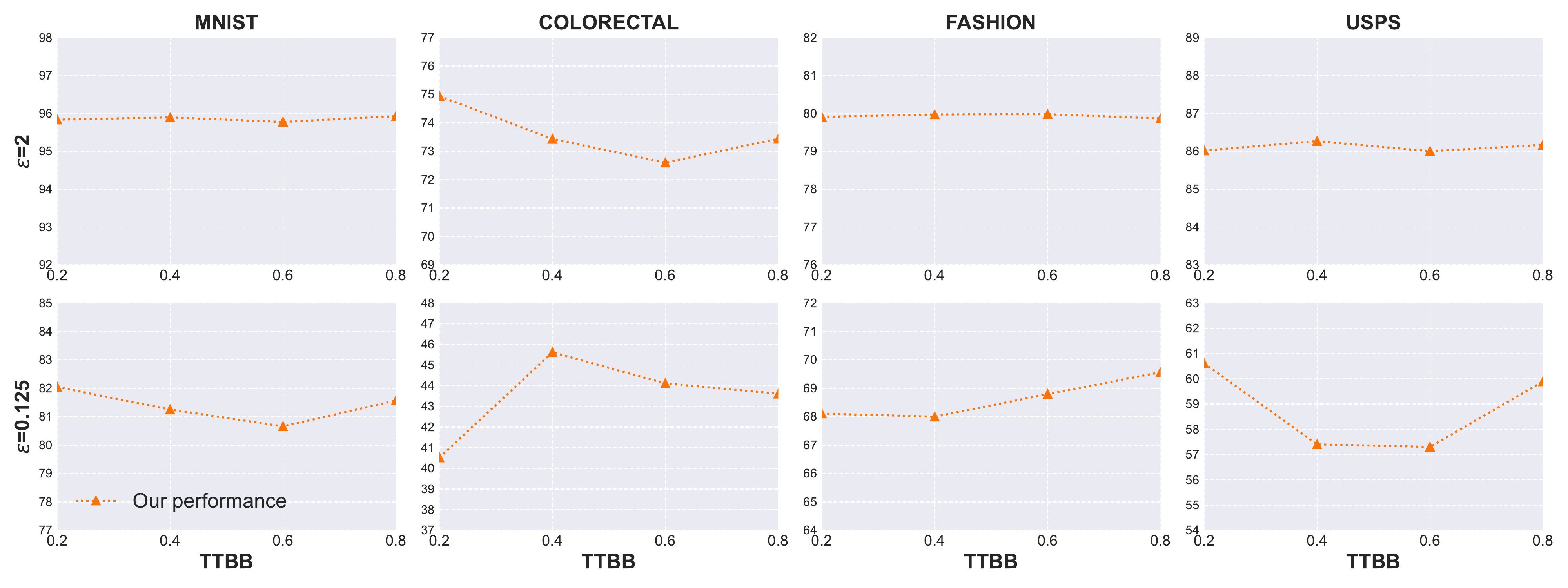}

    \caption{Experiment under Label-flipping attack in i.i.d. setting with different TTBB.}
    \label{fig:smart_attacker_label_iid}
\end{figure*}

\begin{figure*}[!ht] 
    \centering

     \includegraphics[width=0.9\linewidth]{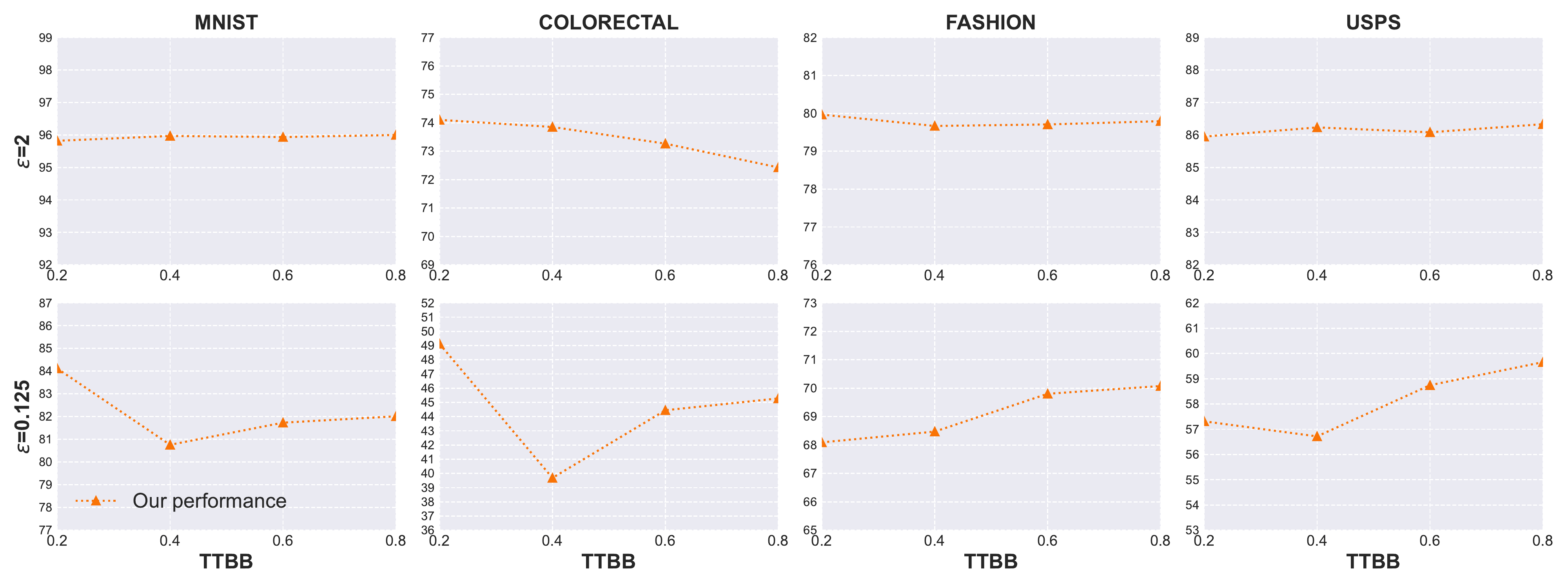}

    \caption{Experiment under Gaussian attack in i.i.d. setting with different TTBB.}
    \label{fig:smart_attacker_gaussian_iid}
\end{figure*}

\begin{figure*}[!ht] 
    \centering

     \includegraphics[width=0.9\linewidth]{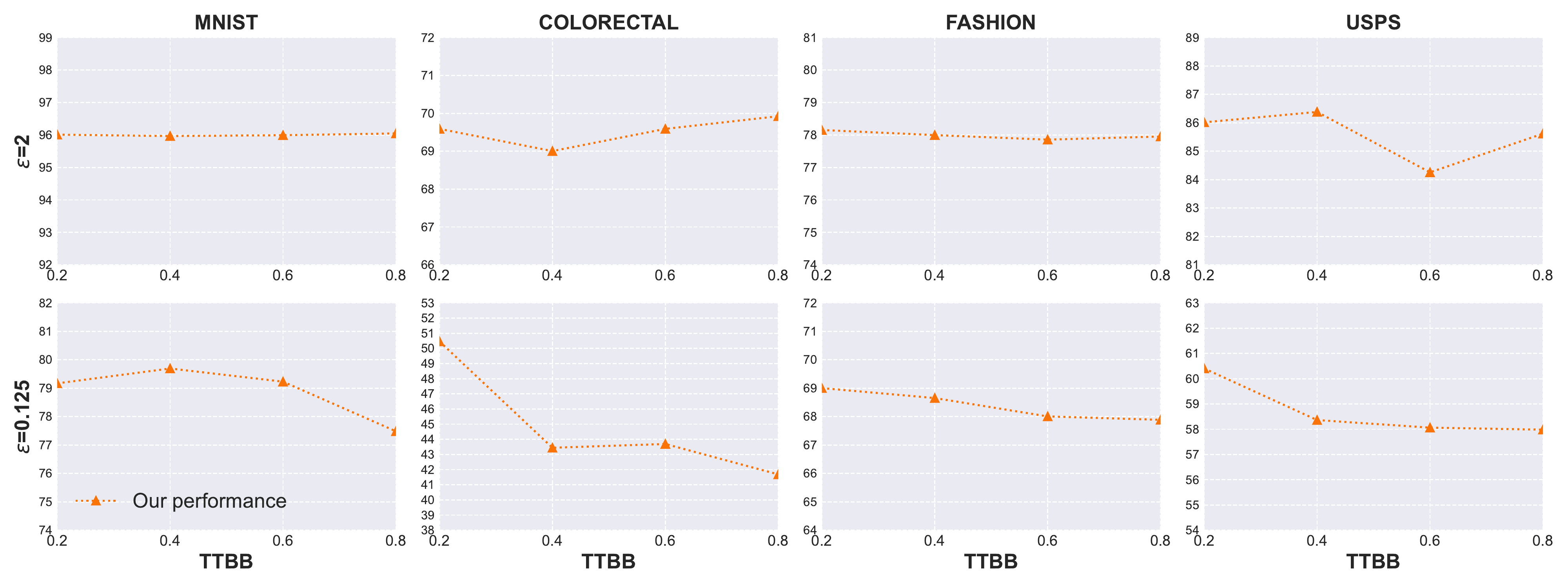}

    \caption{Experiment under Optimized Local Model Poisoning attack in i.i.d. setting with different TTBB.}
    \label{fig:smart_attacker_local_iid}
\end{figure*}

\begin{figure*}[!ht] 
    \centering

     \includegraphics[width=0.9\linewidth]{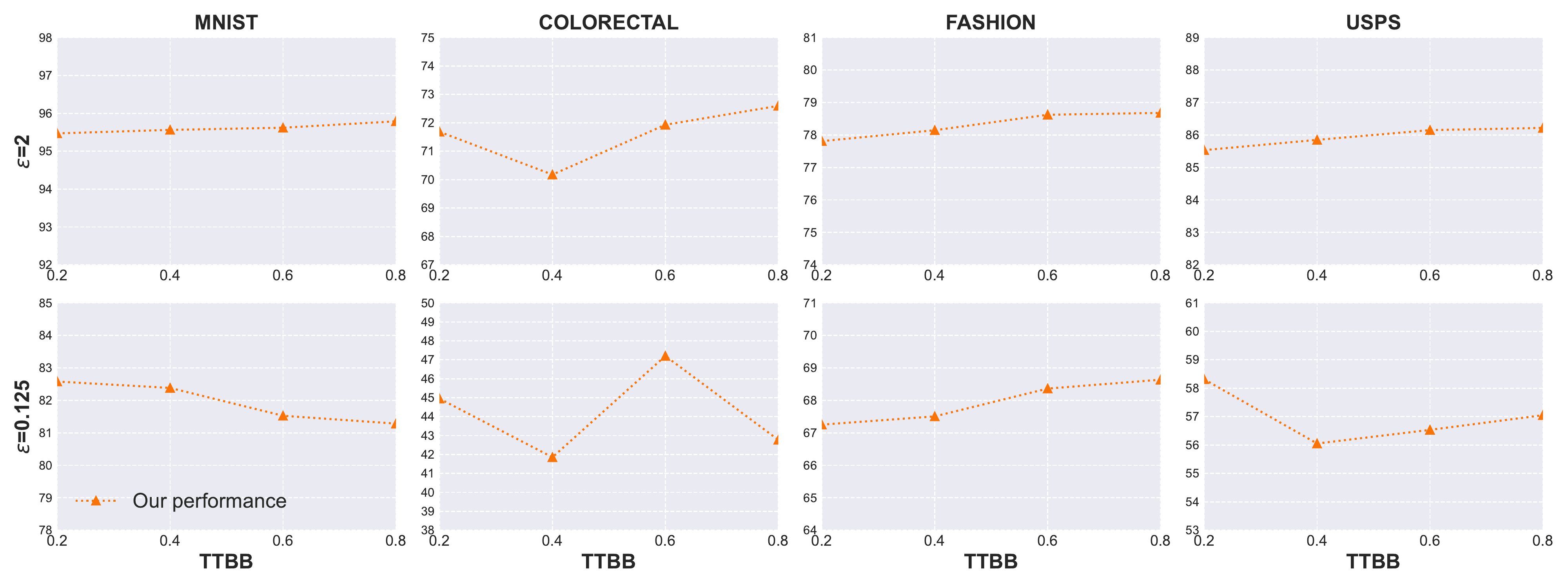}

    \caption{Experiment under Label-flipping attack in non-i.i.d. setting with different TTBB.}
    \label{fig:smart_attacker_label_noniid}
\end{figure*}

\begin{figure*}[!ht] 
    \centering

     \includegraphics[width=0.9\linewidth]{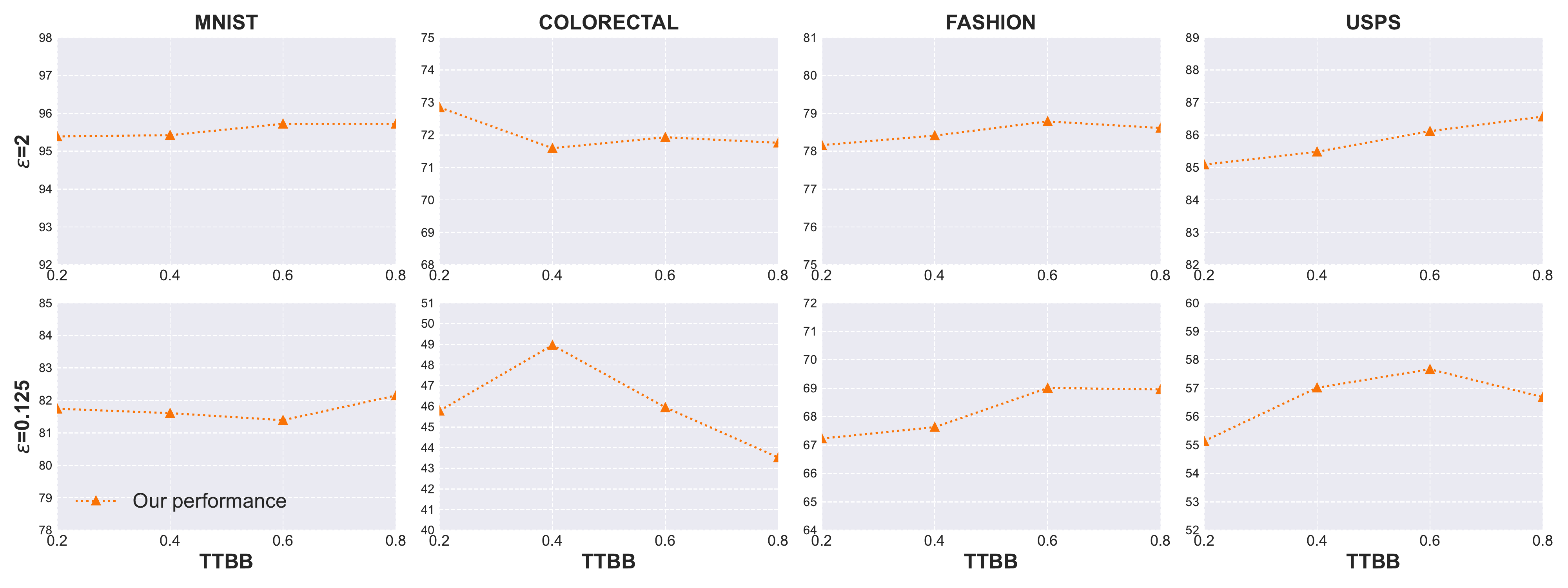}

    \caption{Experiment under Gaussian attack in non-i.i.d. setting with different TTBB.}
    \label{fig:smart_attacker_gaussian_noniid}
\end{figure*}

\begin{figure*}[!ht] 
    \centering

     \includegraphics[width=0.9\linewidth]{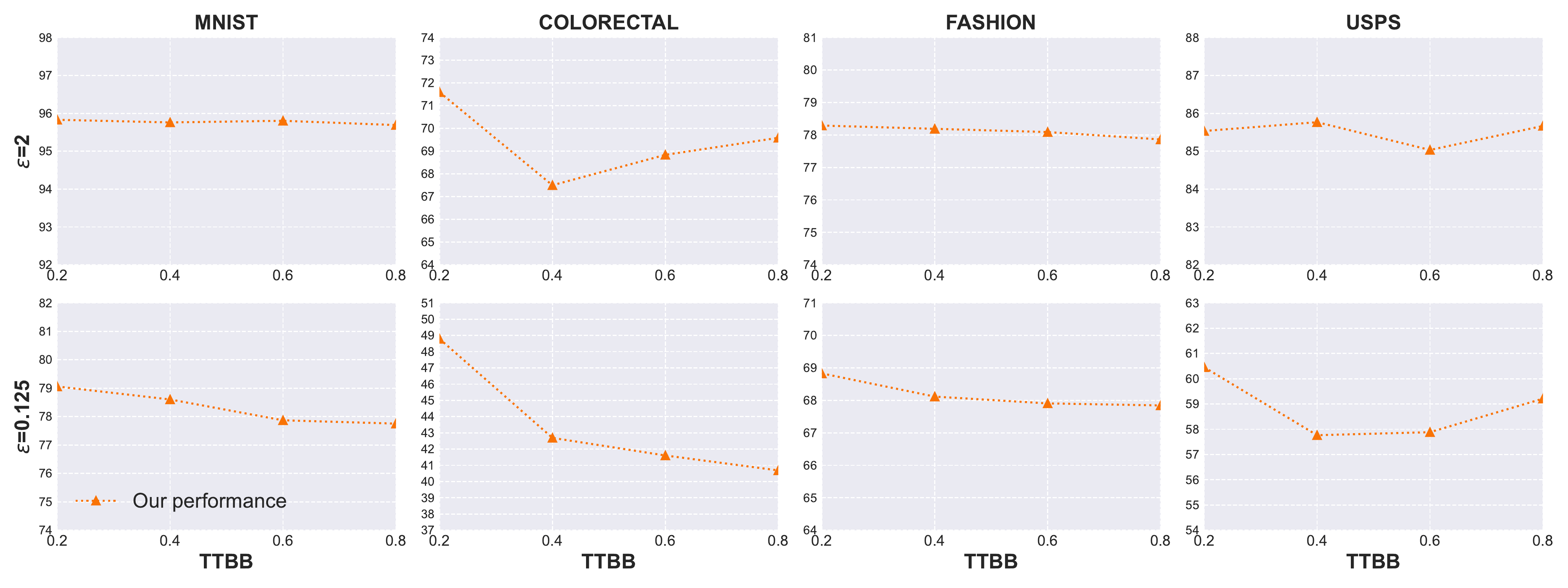}

    \caption{Experiment under Optimized Local Model Poisoning attack in non-i.i.d. setting with different TTBB.}
    \label{fig:smart_attacker_local_noniid}
\end{figure*}

\end{document}